\documentclass{article}

\usepackage{microtype}
\usepackage{graphicx}
\usepackage{subcaption}
\usepackage{booktabs} 

\usepackage{hyperref}

\usepackage[noend]{algorithmic}

\usepackage[preprint]{icml2026}

\usepackage{amsmath}
\usepackage{amssymb}
\usepackage{mathtools}
\usepackage{amsthm}

\usepackage{xcolor, colortbl}
\usepackage{enumitem}
\usepackage{multirow, hhline}
\usepackage{titletoc}

\colorlet{tabcolor}{brown!20}
\newcolumntype{g}{>{\columncolor{tabcolor}}c}

\DeclareMathOperator*{\argmin}{arg\,min}
\DeclareMathOperator*{\diag}{diag}
\DeclareMathOperator*{\tr}{tr}
\DeclareMathOperator*{\Beta}{Beta}
\newcommand{\bs}[1]{\boldsymbol{#1}}

\newcommand{\revised}[1]{\textcolor{black}{#1}}

\usepackage[capitalize,noabbrev]{cleveref}

\theoremstyle{plain}
\newtheorem{theorem}{Theorem}[section]
\newtheorem{proposition}[theorem]{Proposition}
\newtheorem{lemma}[theorem]{Lemma}
\newtheorem{corollary}[theorem]{Corollary}
\theoremstyle{definition}
\newtheorem{definition}[theorem]{Definition}
\newtheorem{assumption}[theorem]{Assumption}
\theoremstyle{remark}

\usepackage[textsize=tiny]{todonotes}

\icmltitlerunning{Sporadic Gradient Tracking over Directed Graphs}

\begin{document}

\twocolumn[
  \icmltitle{Sporadic Gradient Tracking over Directed Graphs: A Theoretical \\ Perspective on Decentralized Federated Learning}



  \icmlsetsymbol{equal}{*}

  \begin{icmlauthorlist}
    \icmlauthor{Shahryar Zehtabi}{yyy}
    \icmlauthor{Dong-Jun Han}{comp}
    \icmlauthor{Seyyedali Hosseinalipour}{sch}
    \icmlauthor{Christopher Brinton}{yyy}
  \end{icmlauthorlist}

  \icmlaffiliation{yyy}{School of Electrical and Computer Engineering, Purdue University, West Lafayette, IN, USA}
  \icmlaffiliation{comp}{Department of Computer Science and Engineering, Yonsei University, Seoul, South Korea}
  \icmlaffiliation{sch}{Department of Electrical Engineering, University at Buffalo -- SUNY, Buffalo, NY, USA}

  \icmlcorrespondingauthor{Shahryar Zehtabi}{szehtabi@purdue.edu}

  \icmlkeywords{Machine Learning, ICML}

  \vskip 0.3in
]



\printAffiliationsAndNotice{}  

\begin{abstract}
    Decentralized Federated Learning (DFL) enables clients with local data to collaborate
    in a peer-to-peer manner to train a generalized model.
    In this paper, we unify two branches of work that have separately solved important challenges in DFL: (i) gradient tracking techniques for mitigating data heterogeneity and (ii) accounting for diverse availability of resources across clients. We
    propose \textit{Sporadic Gradient Tracking} (\texttt{Spod-GT}), the first DFL algorithm
    that incorporates these factors over general directed graphs
    by allowing (i) client-specific gradient computation frequencies and (ii) heterogeneous and asymmetric communication frequencies. We conduct a rigorous convergence analysis of our methodology
    with relaxed assumptions on gradient estimation variance and gradient diversity of clients, providing consensus and optimality guarantees for GT over directed graphs despite intermittent client participation. Through numerical experiments on image classification datasets, we demonstrate the efficacy of \texttt{Spod-GT} compared to well-known GT baselines.
\end{abstract}

\section{Introduction} \label{sec:intro}

Federated learning (FL) has been well explored as a privacy-preserving distributed training method over the past decade \citep{kairouz2021advances}. While much work has studied FL's conventional server-client system architecture, there has been recent interest in decentralized federated learning (DFL) over peer-to-peer networks, i.e., for scenarios where no central server is available \citep{koloskova2020unified}. In DFL, the model aggregation step traditionally conducted at the server occurs instead through a consensus algorithm executed over the client network graph.

Similar to FL, DFL algorithms face heterogeneity challenges in multiple aspects, including
diversity in client communication and computation capabilities, and variation in local dataset statistics across
clients \citep{li2020federated}. Moreover, in DFL, the connectivity and structure (i.e., directed or undirected) of the network graph becomes a critical factor influencing training performance \cite{zehtabi2025decentralized}. This has given rise to two parallel branches of work in the DFL literature. The first has focused on resource and data heterogeneity, simplifying inter-client communications to an undirected graph \citep{liu2025decentralized}. However, in many scenarios, client communications will be asymmetric, particularly in wireless communication networks where devices may transmit at different power levels, experience different noise levels, or other factors causing unidirectional channels \citep{nedic2025ab}. Along these lines, the second branch of work has modeled inter-client communications via a directed graph, with an emphasis on improving the convergence rate of DFL via gradient tracking (GT) methods \citep{nguyen2024decentralized}. However, this branch has largely overlooked resource heterogeneity considerations covered in the first line of work.

In this paper, we aim to unify these two branches of work by answering the following research question:
\vspace{-1.5mm}
\begin{itemize}[leftmargin=*]
    \item[] \textit{How can we develop a consensus-based gradient tracking algorithm for DFL over directed graphs that is aware of heterogeneities in computations and communications of the clients while maintaining convergence guarantees?}
\end{itemize}
\vspace{-1.5mm}

\textbf{Contributions.} Our main contributions are as follows:
\vspace{-1.5mm}
\begin{itemize}[leftmargin=*]
    \item We develop \textit{Sporadic Gradient Tracking} (\texttt{Spod-GT}), the first computation- and communication-aware algorithm for consensus-based DFL over directed graphs. By allowing each client to participate in (i) gradient computation based on its processing resources and (ii) inter-client communication based on its communication resources, \texttt{Spod-GT} achieves target accuracies faster than its non-heterogeneity-aware counterparts. In \texttt{Spod-GT}, clients with lower available resources (stragglers) participate in DFL with less frequency than resource-abundant clients, balancing the training load in a client-specific way.

    \item We provide a rigorous convergence analysis of GT-enhanced DFL for non-convex loss functions when clients are connected through a directed graph. Unique to our setup, we factor in the effect of sporadicity in both (i) gradient calculations and (ii) inter-client communications, and illustrate how our analysis reverts back to conventional GT methods as special cases. We overcome novel challenges in characterizing consensus and optimality through joint analysis of sporadically occurring GT updates and sporadic communications over directed graphs.

    \item By relaxing assumptions on the gradient estimation variance and gradient diversity compared to existing literature, our theoretical results elucidate novel DFL dynamics. To the best of our knowledge, our work is the first to show a clear dependence of (i) minimum allowed gradient computation probability on non-uniform gradient diversity, (ii) minimum allowed inter-client communication probability on spectral characteristics of the underlying directed graph, (iii) minimum allowed batch size on non-uniform gradient estimation variance.

    \item We validate the efficacy of our proposed methodology by comparing it against state-of-the-art GT methods that are designed for directed graphs.
\end{itemize}

See Appendix~\ref{app:supp:not} for a summary of notation.

\section{Related Work}

Table~\ref{tab:comparison} summarizes key contributions of our work relative to closely related literature on GT algorithms for DFL. Our paper is the first to consider sporadic GT steps
and aggregations simultaneously, capturing heterogeneous resources in fully-decentralized directed networks. In the following, we discuss related work along key considerations of GT and sporadicity of communications in DFL.

\begin{table*}[t]
    \centering
    \setlength{\tabcolsep}{2pt}
    {\small
    \caption{Comparison of our work with representative related papers.}
    \begin{tabular}{c||c|c|c|c|c||c|c|c}
        \toprule
        \multirow{4}{*}{Paper} & \multicolumn{5}{g||}{Properties of Framework} & \multicolumn{3}{g}{Theoretical Assumptions}
        \\
        \hhline{~|-----|---}
        & \begin{tabular}{c} Directed \\ Graphs \end{tabular} & \begin{tabular}{c} Non-doubly \\ Stochasticable \end{tabular} & \begin{tabular}{c} Gradient \\ Tracking \end{tabular} & \begin{tabular}{c} Sporadic \\ Comp. \end{tabular} & \begin{tabular}{c} Sporadic \\ Comm. \end{tabular} & \begin{tabular}{c} General. \\ Gradient \\ Variance \end{tabular} & \begin{tabular}{c} General. \\ Gradient \\ Diversity \end{tabular} & \begin{tabular}{c} Non- \\ Convex \\ Loss \end{tabular}
        \\
        \midrule
        \citet{tang2018d}
        & & & \checkmark & & & & & \checkmark
        \\
        \hline
        \rowcolor{tabcolor} \citet{koloskova2020unified} & & & & & \checkmark & \checkmark & \checkmark & \checkmark
        \\
        \hline
        \citet{pu2020push} & \checkmark & \checkmark & \checkmark & & \checkmark & & &
        \\
        \hline
        \rowcolor{tabcolor} \citet{xin2021improved} & \checkmark & & \checkmark & & & & & \checkmark
        \\
        \hline
        \citet{liu2025decentralized} & & & \checkmark & & & & \checkmark & \checkmark
        \\
        \hline
        \rowcolor{tabcolor} \citet{zehtabi2025decentralized} & & & & \checkmark & \checkmark & & \checkmark & \checkmark
        \\
        \hline
        \citet{nedic2025ab}
        & \checkmark & \checkmark & \checkmark & & & & &
        \\
        \hline
        \rowcolor{tabcolor} \textbf{\texttt{Spod-GT} (Ours)} & \pmb{\checkmark} & \pmb{\checkmark} & \pmb{\checkmark} & \pmb{\checkmark} & \pmb{\checkmark} & \pmb{\checkmark} & \pmb{\checkmark} & \pmb{\checkmark}
        \\
        \bottomrule
    \end{tabular}
    \label{tab:comparison}}
    \vspace{-3mm}
\end{table*}

\textbf{GT over Undirected Graphs.} Early work in decentralized optimization over undirected graphs involved using stochastic (sub)gradient descent (SGD) on convex loss functions \citep{nedic2009distributed}. Later works such as D-PSGD \citep{lian2017can} extended this to general smooth non-convex loss functions.
Although SGD has a fast convergence rate for centralized/federated training, conducting consensus aggregations that are needed for decentralized optimization results in a slower rate for SGD on decentralized training \citep{nedic2018network}. To overcome this challenge,
\textit{gradient tracking} has been proposed that introduces an auxiliary set of parameters for each client that helps them track the global gradient. For example, EXTRA \citep{shi2015extra} applies GT for convex loss functions, NEXT \citep{di2016next} extends this to non-convex losses with bounded gradients, and D$^2$ \citep{tang2018d}
extends it to general non-convex smooth loss functions. However, a limitation of these works is the modeling of the network via an undirected graph, overlooking potential asymmetries.

Works such as GT-DSGD \cite{xin2021improved} work on directed graphs admitting a doubly-stochastic matrix, making their analysis similar to undirected graphs. Our focus here is general directed graphs that do not necessarily admit a doubly-stochastic matrix \cite{gharesifard2010does}.

\begin{figure}[t]
    \centering
    \includegraphics[width=\linewidth]{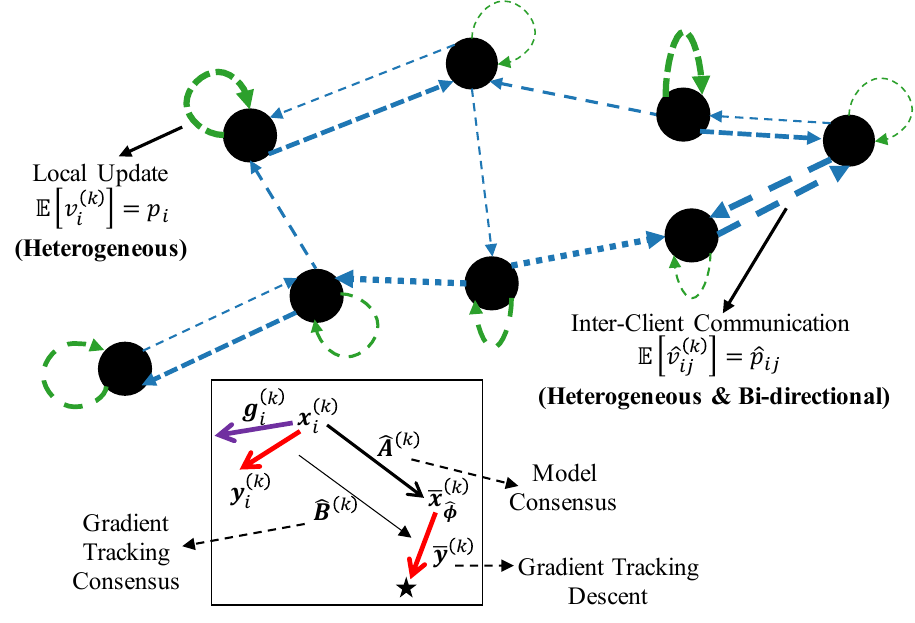}
    \caption{\small Illustration of GT-enhanced DFL over directed graphs with clients having heterogeneous computation and communication capabilities. While in conventional GT, local updates and inter-client communications occur at every iteration of training, \texttt{Spod-GT} applies sporadicity in both communications and computations (dashed lines, thickness representing relative frequency).} 
    \vspace{-0.15in}
    \label{fig:sys_diag}
\end{figure}

\textbf{GT over Directed Graphs.} Early attempts on decentralized optimization over directed graphs, such as
subgradient-push \citep{nedic2014distributed}, use the push-sum consensus strategy
in conjunction with SGD to minimize convex loss functions. However, SGD suffers from a slow convergence rate,
and the push-sum consensus strategy introduces non-linearity to iterative updates \cite{xin2020general}. Two improvements have been made since; (i) Employ GT to speed up convergence while using the push-sum strategy for consensus, as in
Push-DIGing \citep{nedic2017achieving} and FROST \cite{xin2019frost}; (ii) Instead of relying on the non-linear push-sum strategy for consensus, use a row-stochastic/column-stochastic mixing matrix pair for stable training, e.g., AB \cite{xin2018linear}, Push-Pull \cite{pu2020push}, and AB/Push-Pull
\cite{nedic2025ab}. However, despite improvements in the convergence rate of these algorithms, communication and computational resource heterogeneity of clients in DFL have not been fully studied.

\textbf{Sporadic Communications.}
Clients in DFL operate with limited computational and communication resources. An approach to reduce the number of communications is the randomized gossip algorithms \cite{boyd2006randomized}, e.g., DSGD \cite{koloskova2020unified} and DSpodFL \citep{zehtabi2025decentralized} when gossip is applied to SGD over undirected graphs, GSGT \cite{pu2021distributed} for GT methods over undirected graphs, SGP \cite{assran2019stochastic} for SGD with push-sum over directed graphs and G-Push-Pull \cite{pu2020push} for GT over directed graphs. Although DSGD and DSpodFL work with expected contraction of the mixing matrices over undirected graphs, work on directed graphs either make restrictive strong connectivity assumptions \cite{nguyen2024decentralized} or allow the clients only one edge at each iteration (G-Push-Pull). Our paper bridges this gap by proving expected contraction properties for general sporadic mixing matrices over directed graphs. Another approach to reduce the number of communications
is to take multiple local gradient steps between aggregations, as in DFedAvg \cite{sun2022decentralized} with SGD and K-GT \cite{liu2025decentralized} with GT. However, these approaches enforce a fixed number of local gradient steps for the clients, overlooking their computational heterogeneity. Our framework instead allows for an adaptive number of local steps
for each client.

\section{Methodology}

\subsection{Problem Formulation}

We consider a DFL system with $m$ clients $\mathcal{M}:=\{1,\ldots, m\}$ and a series of training iterations $k = 0, ..., K - 1$. Clients are connected through a \textit{directed graph} $\mathcal{G} = (\mathcal{M}, \mathcal{E})$, where $(i,j) \in \mathcal{E}$ if there exists a physical link between clients $i$ and $j$. Each client $i \in \mathcal{M}$ has a local dataset $\mathcal{D}_i = \{ \xi_{ij} = (x_{ij}^f, y_{ij}^c) \}_{j=1}^{D_i}$ of size $|\mathcal{D}_i| = D_i$ where $\xi_{ij}$ denotes a data sample with features $x_{ij}^f$ and class label $y_{ij}^c$. The goal of each client is to learn a classifier $h_i(x^f ; \bs{x})$ parameterized by $\bs{x} \in \mathbb{R}^n$ to correctly predict labels $y^c$, by processing its local dataset $\mathcal{D}_i$ and collaborating with its neighbors without the need to exchange data.

Let $\mathcal{L}(\bs{x} ; \xi)$ denote a loss function for the model $\bs{x}$ evaluated in the data sample $\xi$, e.g., binary cross-entropy loss for a two-label dataset $\mathcal{L}(\bs{x} ; (x^f, y^c)) = y^c \log{h(x^f ; \bs{x})} + (1 - y^c) \log{(1 - h(x^f ; \bs{x}))}$. We therefore define the local and the global loss functions as
\begin{equation} \label{eqn:local_global_loss}
    F_i(\bs{x}) = \frac1{D_i} \hspace{-1mm} \sum_{\xi_{ij} \in \mathcal{D}_i}{\hspace{-2mm} \mathcal{L}(\bs{x} ; \xi_{ij})}, \quad F(\bs{x}) = \frac1m \sum_{i=1}^m{F_i(\bs{x})},
\end{equation}
respectively, where $F_i(\bs{x}) : \mathbb{R}^n \to \mathbb{R}$ and $F(\bs{x}) : \mathbb{R}^n \to \mathbb{R}$. The goal of DFL is for clients to learn the optimal global model $\bs{x}^\star = \argmin_{x \in \mathbb{R}^n}{F{(\bs{x})}}$.

\subsection{Sporadic Gradient Tracking}

To minimize the global loss function in Eq.~\eqref{eqn:local_global_loss}, we employ \textit{gradient tracking}, which has been shown to have good convergence properties for directed graphs \citep{xin2018linear, pu2020push}. In gradient tracking methods, each client has an additional set of parameters $\bs{y}$ that tracks an estimate of the global gradient at each iteration of training. Although existing work on gradient tracking has assumed that each client computes gradients (or multiple local gradient steps) at each communication round, we focus on a relaxed and generalized extension of gradient tracking where the clients may opt out of computing gradients at some iterations. The iterative updates of our method, which we call \texttt{Spod-GT}, can be formally written in two steps as
\vspace{-1mm}
\begin{equation} \label{eqn:event_update_1}
    \begin{gathered}
        \bs{x}_i^{(k+0.5)} = \bs{x}_i^{(k)} + \sum_{j=1}^m{a_{ij} (\bs{x}_j^{(k)} - \bs{x}_i^{(k)}) \hat{v}_{ij}^{(k)}},
        \\
        \bs{y}_i^{(k+0.5)} = \bs{y}_i^{(k)} + \sum_{j=1}^m{(b_{ij} \bs{y}_j^{(k)} \hat{v}_{ij}^{(k)} - b_{ji} \bs{y}_i^{(k)} \hat{v}_{ji}^{(k)})},
    \end{gathered}
    \vspace{-1mm}
\end{equation}
which are the aggregation steps, and
\begin{equation} \label{eqn:event_update_2}
    \begin{gathered}
        \bs{x}_i^{(k+1)} = \bs{x}_i^{(k+0.5)} - \eta_i^{(k)} \bs{y}_i^{(k+0.5)},
        \\
        \bs{y}_i^{(k+1)} = \bs{y}_i^{(k+0.5)} + \bs{g}_i^{(k+1)} v_i^{(k+1)} - \bs{g}_i^{(k)} v_i^{(k)},
    \end{gathered}
\end{equation}
which are the update steps. Note that $\bs{x}_i^{(k)}$ and $\bs{y}_i^{(k)}$ are the model parameters and the global gradient estimate of the client $i \in \mathcal{M}$ in the communication round $k \geq 0$, respectively, and we initialize $\bs{y}_i^{(0)} = \bs{g}_i^{(0)} v_i^{(0)}$, where
\vspace{-1mm}
\begin{equation} \label{eqn:sgd}
    \bs{g}_i^{(k)} = \frac1{B_i} \sum_{\xi_{ij}^{(k)} \in \mathcal{B}_i^{(k)}}{\nabla{\mathcal{L}}(\bs{x}_i^{(k)} ; \xi_{ij}^{(k)})},
    \vspace{-1mm}
\end{equation}
is the local stochastic gradient evaluated using the mini-batch $\mathcal{B}_i^{(k)} \subseteq \mathcal{D}_i$ with $B_i = |\mathcal{B}_i^{(k)}|$, and $\eta_i^{(k)}$ is the learning rate of the client $i$ in the round $k$. Unique to our setup, $v_i^{(k)} \in \{ 0, 1 \}$ and $\hat{v}_{ij}^{(k)} \in \{ 0, 1 \}$ are indicator random variables that capture whether or not a client participates in the computation of gradients and inter-client communication in a given round, respectively. Finally, note that $a_{ij}$ and $b_{ij}$ are the mixing weights for the model parameters and the tracked gradients between clients $(i, j) \in \mathcal{E}$, respectively, and while $a_{ij}$ are chosen by the client receiving the model parameters (client $i$), $b_{ij}$ are determined by the client transmitting the tracked gradient estimates (client $j$).

To simplify the notation in Eqs.~\eqref{eqn:event_update_1} and \eqref{eqn:event_update_2}, let $\hat{a}_{ij}^{(k)} = a_{ij} \hat{v}_{ij}^{(k)}$ and $\hat{b}_{ij}^{(k)} = b_{ij} \hat{v}_{ij}^{(k)}$ for $i \neq j$, and $\hat{a}_{ii}^{(k)} = 1 - \sum_{\ell=1}^m{a_{i\ell} \hat{v}_{i\ell}^{(k)}}$ and $\hat{b}_{ii}^{(k)} = 1 - \sum_{\ell=1}^m{b_{\ell i} \hat{v}_{\ell i}^{(k)}}$ for $i = j$ to get
\vspace{-1mm}
\begin{equation} \label{eqn:vector_update}
    \begin{gathered}
        \bs{x}_i^{(k+1)} = \sum_{j=1}^m{\hat{a}_{ij}^{(k)} \bs{x}_j^{(k)}} - \eta_i^{(k)} \sum_{j=1}^m{\hat{b}_{ij}^{(k)} \bs{y}_j^{(k)}},
        \\
        \bs{y}_i^{(k+1)} = \sum_{j=1}^m{\hat{b}_{ij}^{(k)} \bs{y}_j^{(k)}} + \bs{g}_i^{(k+1)} v_i^{(k+1)} - \bs{g}_i^{(k)} v_i^{(k)},
    \end{gathered}
    \vspace{-1mm}
\end{equation}
where $[\hat{a}_{ij}^{(k)}]_{1 \le i,j \le m}$ and $[\hat{b}_{ij}^{(k)}]_{1 \le i,j \le m}$ represent sporadic mixing weights. Alg.~\ref{alg:spodgt} in Appendix~\ref{app:supp:alg} summarized our complete algorithm.


\revised{\textbf{Insights for Sporadicity and GT.} Resource availability is often heterogeneous and asymmetric in DFL, i.e., there are (i) variations in client computation capabilities,
and (ii) variations in communication link throughputs.
\texttt{Spod-GT} overcomes these limitations by relaxing (i) the client gradient computation frequencies to $p_i \in (0, 1]$ proportional to its processing capabilities, and (ii) the edge utilization frequency for parameter exchange to $\hat{p}_{ij} \in (0, 1]$ proportional to a link's throughput. This is in contrast to conventional DFL methods which enforce a restrictive $p_i = \hat{p}_{ij} = 1$, demanding the same work load from all clients irrespective of their heterogeneous and asymmetric capabilities. The notion of sporadicity was also explored in DSpodFL \citep{zehtabi2025decentralized} to run SGD over undirected graphs. In this paper, we focus on the more challenging paradigm decentralized optimization over non-doubly stochasticable directed graphs. We adopt GT in our work which has been shown to have good convergence properties over directed graphs, particularly when coupled with a row-stochastic/column-stochastic pair of mixing matrices to form consensus.}

\subsection{Matrix Representation} \label{ssec:mat_notation}

We collect the model parameters {\small $\bs{x}_i^{(k)}$}, the gradient tracking estimates {\small $\bs{y}_i^{(k)}$} and the local stochastic gradients {\small $\bs{g}_i^{(k)}$} of all clients to the matrices {\small $\bs{X}^{(k)} = [\bs{x}_1^{(k)} \, \cdots \, \bs{x}_m^{(k)}]^T \in \mathbb{R}^{m \times n}$}, {\small $\bs{Y}^{(k)} = [\bs{y}_1^{(k)} \, \cdots \, \bs{y}_m^{(k)}]^T \in \mathbb{R}^{m \times n}$} and {\small $\bs{G}^{(k)} = [\bs{g}_1^{(k)} \, \cdots \, \bs{g}_m^{(k)}]^T \in \mathbb{R}^{m \times n}$}, respectively. We also collect all mixing weights into matrices, i.e., {\small $\bs{A} = [ a_{ij} ]_{1 \le i,j \le m} \in \mathbb{R}^{m \times m}$}, {\small $\bs{B} = [ b_{ij} ]_{1 \le i,j \le m} \in \mathbb{R}^{m \times m}$}, {\small $\bs{\hat{A}}^{(k)} = [ \hat{a}_{ij}^{(k)} ]_{1 \le i,j \le m} \in \mathbb{R}^{m \times m}$} and {\small $\bs{\hat{B}}^{(k)} = [ \hat{b}_{ij}^{(k)} ]_{1 \le i,j \le m} \in \mathbb{R}^{m \times m}$}. For the learning rate {\small $\eta_i^{(k)}$} and the indicator variable {\small $v_i^{(k)}$} of each client, which are scalar variables, we define the diagonal matrices {\small $\bs{\Lambda}_{\bs{\eta}}^{(k)} = \diag(\bs{\eta}^{(k)}) \in \mathbb{R}^{m \times m}$} and {\small $\bs{\Lambda}_{\bs{v}}^{(k)} = \diag(\bs{v}^{(k)}) \in \mathbb{R}^{m \times m}$}, respectively. Using this matrix notation, we can collect the iterative update rules for each client in Eq.~\eqref{eqn:vector_update} as update rules for the whole decentralized network as
\begin{equation} \label{eqn:matrix_update}
	\begin{gathered}
	    \bs{X}^{(k+1)} = \bs{\hat{A}}^{(k)} \bs{X}^{(k)} - \bs{\Lambda}_{\bs{\eta}}^{(k)} \bs{\hat{B}}^{(k)} \bs{Y}^{(k)},
        \\
        \bs{Y}^{(k+1)} = \bs{\hat{B}}^{(k)} \bs{Y}^{(k)} + \bs{\Lambda}_{\bs{v}}^{(k+1)} \bs{G}^{(k+1)} - \bs{\Lambda}_{\bs{v}}^{(k)} \bs{G}^{(k)}.
	\end{gathered}
\end{equation}
We also define $\bs{F}_{\bs{X}}^{(k)} = [F_1(\bs{x}_1^{(k)}) \, \cdots \, F_m(\bs{x}_m^{(k)})]^T$ as the collection of local loss functions evaluated with individual models $\bs{x}_i^{(k)}$.


\section{Convergence Analysis}

In this section, we first lay out our Assumptions in Sec.~\ref{ssec:assumps}, and give the necessary definitions for our theoretical results in Sec.~\ref{ssec:defs}. We then discuss our key Propositions in Sec.~\ref{ssec:props} and conclude with our Theorem in Sec.~\ref{ssec:thm}. For brevity, we have relegated all supporting Lemmas to Appendices~\ref{ssec:intermed} and \ref{ssec:main_lemmas}, with all constant scalars summarized in Appendix~\ref{app:supp:table}.

\subsection{Assumptions and Immediate Implications} \label{ssec:assumps}

\begin{assumption}[Lipschitz-Continuous Gradients] \label{assump:lipschitz}
    Each local loss function $F_i : \mathbb{R}^n \to \mathbb{R}$ is continuously differentiable and its gradient $\nabla{F}_i : \mathbb{R}^n \to \mathbb{R}^n$ is $L_i$-Lipschitz continuous, i.e., $\| \nabla{F}_i(\bs{x}) - \nabla{F}_i(\bs{x}') \| \le L_i \| \bs{x} - \bs{x}' \|$, with $\{ \bs{x}, \bs{x}' \} \subseteq \mathbb{R}^n$ and $L_i > 0$ for all $i \in \mathcal{M}$.
\end{assumption}

Since the global loss function is defined as $F(\bs{x}) = (1/m) \sum_{i=1}^m{F_i(\bs{x})}$, it is also continuously differentiable as a convex combination of continuously differentiable functions. Furthermore, its gradient function is $\bar{L}$-Lipschitz with $\bar{L} = (1/m) \sum_{i=1}^m{L_i}$, and we have that
$F(\bs{x}) \le F(\bs{x}') + \left\langle \nabla{F}(\bs{x}'), \bs{x} - \bs{x}' \right\rangle + \frac12 \bar{L} \| \bs{x} - \bs{x}' \|^2$,
for all $\{ \bs{x}, \bs{x}' \} \subseteq \mathbb{R}^n$ \citep{bottou2018optimization}.

\begin{assumption}[Stochastic Gradient Estimation] \label{assump:sgd}
    The stochastic gradient $\nabla{\mathcal{L}} : \mathbb{R}^n \times \mathbb{R}^{|\xi|} \to \mathbb{R}^n$, where $\xi$ is a data sample, is an unbiased estimator, i.e., for each client $i \in \mathcal{M}$ with dataset $\mathcal{D}_i$, we have $\mathbb{E}_{\xi_{ij} \sim \mathcal{D}_i}\left[ \nabla{\mathcal{L}}(\bs{x}, \xi_{ij}) \right] = \nabla{F}_i(\bs{x})$, with $\bs{x} \in \mathbb{R}^n$. We further assume that the variance of the mini-batch gradient $\nabla{\mathcal{L}}$ is bounded as
    $\mathbb{E}_{\mathcal{B}_i \sim \mathcal{D}_i}[ \| \frac1{B_i} \sum_{\xi_{ij} \in \mathcal{B}_i}{\nabla{\mathcal{L}}(\bs{x}, \xi_{ij})} - \nabla{F}_i(\bs{x}) \|^2 ] \le \frac{1 - B_i / D_i}{B_i} ( \sigma_{0,i}^2 + \sigma_{1,i}^2 \| \nabla{F}(\bs{x}) \|^2 )$,
    with $\sigma_{0,i}, \sigma_{1,i} > 0$ where $\nabla{F}(\bs{x})$ is the global gradient.
\end{assumption}
Note that while most works in the literature consider a uniform bound on stochastic gradient estimation, i.e., $\sigma_{1,i}^2 = 0$, our Assumption~\ref{assump:sgd} relaxes it to the general case of $\sigma_{1,i}^2 > 0$ \citep{koloskova2020unified}. Furthermore, note that when using a full batch $B_i = D_i$, the gradient estimation becomes exact, and when $B_i = 1$, we get the conventional stochastic gradient estimation variance (since $1 - 1/D_i \approx 1$).


\begin{assumption}[Gradient Diversity] \label{assump:graddiv}
    Each local loss function $F_i : \mathbb{R}^n \to \mathbb{R}$ is continuously differentiable and the distance from its gradient $\nabla{F}_i : \mathbb{R}^n \to \mathbb{R}^n$ to the global loss gradient function $\nabla{F} = (1/m) \sum_{i=1}^m{\nabla{F}_i} : \mathbb{R}^n \to \mathbb{R}^n$ is bounded as $\| \nabla{F}_i(\bs{x}) \|^2 \le \delta_{0,i}^2 + \delta_{1,i}^2 \| \nabla{F}(\bs{x}) \|^2$,
    with $\bs{x} \in \mathbb{R}^n$ and $\delta_{0,i}, \delta_{1,i} > 0$.
\end{assumption}
While earlier papers worked with the restrictive bounded gradient assumption, i.e., $\delta_{1,i}^2 = 0$, Assumption~\ref{assump:graddiv} has been used more recently in the literature as a relaxed condition on gradient diversity \citep{koloskova2020unified}.


\begin{assumption}[Bernoulli Indicator Random Variables] \label{assump:bernoulli}
    The indicator random variables $v_i^{(k)}$ and $\hat{v}_{ij}^{(k)}$ that capture whether a client $i \in \mathcal{M}$ computes local gradient or participates in inter-client communication with its neighbor $j$ at iteration $k \geq 0$, respectively, are statistically independent Bernoulli random variables with non-zero probabilities. In other words, we have
    $\mathbb{E}[ v_i^{(k)} ] = p_i$ and $\mathbb{E}[ \hat{v}_{ij}^{(k)} ] = \hat{p}_{ij}$,
    where $0 < p_i, \hat{p}_{ij} \le 1$.
\end{assumption}


Let us define the expected mixing matrices as
\vspace{-1mm}
\begin{equation} \label{eqn:expected_mat}
    \bs{\hat{A}} = \mathbb{E}\left[ \bs{\hat{A}}^{(k)} \right], \qquad \bs{\hat{B}} = \mathbb{E}\left[ \bs{\hat{B}}^{(k)} \right],
    \vspace{-2mm}
\end{equation}
in which $\hat{a}_{ij} = a_{ij} p_{ij}$ and $\hat{b}_{ij} = b_{ij} p_{ij}$ if $i \neq j$, and $\hat{a}_{ii} = 1 - \sum_{j=1}^m{a_{ij} p_{ij}}$ and $\hat{b}_{ii} = 1 - \sum_{j=1}^m{b_{ij} p_{ij}}$ if $i = j$.

\begin{assumption}[Strongly Connected Directed Graph] \label{assump:graph}
    The directed graph $\mathcal{G}$ is strongly connected, i.e., for any pair of clients $\{ i, j \} \subseteq \mathcal{M}$, there exist a path from $i$ to $j$ using the edges in $\mathcal{E}$. Also, assume each client $i$ has a self-loop.
\end{assumption}

We design the mixing matrices $\bs{A}$ and $\bs{B}$ defined in Sec.~\ref{ssec:mat_notation} compatible with the graph $\mathcal{G}$, with $\bs{A}$ being row-stochastic and $\bs{B}$ being column-stochastic. In other words, $a_{ij} > 0$ if $j \in \mathcal{N}_i^{\text{out}} \cup \{ i \}$ and $a_{ij} = 0$ otherwise, and $b_{ij} > 0$ if $j \in \mathcal{N}_i^{\text{in}} \cup \{ i \}$ and $b_{ij} = 0$ otherwise. Note that if we design $\bs{A}$ as row-stochastic and $\bs{B}$ as column-stochastic, the sporadic mixing matrices {\small $\bs{\hat{A}}^{(k)}$} and {\small $\bs{\hat{B}}^{(k)}$} defined in Eq.~\eqref{eqn:matrix_update} and the expected mixing matrices $\bs{\hat{A}}$ and $\bs{\hat{B}}$ defined in Eq.~\eqref{eqn:expected_mat} will also be row-stochastic and column-stochastic, respectively. In matrix notation, we have the following.
\begin{equation} \label{eqn:stochastic}
    \begin{gathered}
        \bs{\hat{A}}^{(k)} \bs{1} = \bs{1}, \bs{\hat{A}} \bs{1} = \bs{1}, \quad
        \bs{1}^T \bs{\hat{B}}^{(k)} = \bs{1}^T, \bs{1}^T \bs{\hat{B}} = \bs{1}^T,
    \end{gathered}
\end{equation}
implying that the vector of all-ones $\bs{1}$ is the right-eigenvector of
{\small $\bs{\hat{A}}^{(k)}$} and $\bs{\hat{A}}$ and the left-eigenvector of
{\small $\bs{\hat{B}}^{(k)}$} and $\bs{\hat{B}}$ corresponding to the eigenvalue $1$. Moreover, since non-negative expected mixing matrices $\bs{\hat{A}}$ and $\bs{\hat{B}}$ characterize a strongly-connected graph, since they are scaled versions of non-negative mixing matrices $\bs{A}$ and $\bs{B}$ corresponding to a strongly-connected graph, they are irreducible and primitive as they contain self-loops. Hence, by the Perron-Frobenius Theorem, the eigenvalue $1$ is simple and is the spectral radius for all four of the matrices mentioned above \citep{pu2020push}. This also implies that there exists
a unique positive left-eigenvector $\bs{\hat{\phi}}$ for $\bs{\hat{A}}$ and a unique positive right-eigenvector $\bs{\hat{\pi}}$ for $\bs{\hat{B}}$. In other words,
\begin{equation} \label{eqn:eigenvectors}
    \bs{\hat{\phi}}^T \bs{\hat{A}} = \bs{\hat{\phi}}^T, \, \bs{\hat{B}} \bs{\hat{\pi}} = \bs{\hat{\pi}}, \qquad \bs{\hat{\phi}}^T \bs{1} = 1, \, \bs{1}^T \bs{\hat{\pi}} = 1,
\end{equation}
where we choose the eigenvectors
$\bs{\hat{\phi}}$ and $\bs{\hat{\pi}}$ to be stochastic.

\textbf{Limitations of Existing Analyses.} Note that the sporadic mixing matrices {\small $\bs{\hat{A}}^{(k)}$} and {\small $\bs{\hat{B}}^{(k)}$} do not necessarily characterize a strongly-connected graph, and we cannot use the Perron-Frobenius Theorem to deduce any positive eigenvectors for them (see \citet{touri2013product} for an in-depth discussion on dynamic random Markov chains). A key contribution of our work is obtaining a contraction property of sporadic aggregation matrices {\small $\bs{\hat{A}}^{(k)}$} and {\small $\bs{\hat{B}}^{(k)}$} in Lemma~\ref{lemma:sporadic_contraction}
by exploiting their expected behavior given in Eq.~\eqref{eqn:expected_mat}.

\subsection{Definitions} \label{ssec:defs}

First, we use the properties of the left and right eigenvectors given in Eqs.~\eqref{eqn:stochastic} and \eqref{eqn:eigenvectors} to give the following definition.
\begin{definition} \label{def:avg_vec}
    For the purpose of our theoretical analysis, we will consider the $\bs{\hat{\phi}}$-weighted average of the model parameters {\small $\bs{x}_i^{(k)}$}, the regular average of gradient tracking parameters {\small $\bs{y}_i^{(k)}$} and the weighted average of stochastic gradients {\small $\bs{g}_i^{(k)}$} given as
    $\bs{\bar{x}}_{\bs{\hat{\phi}}}^{(k)} = \sum_{i=1}^m{\hat{\phi}_i \bs{x}_i^{(k)}} = \bs{\hat{\phi}}^T \bs{X}^{(k)}$, $\bs{\bar{y}}^{(k)} = \frac1m \sum_{i=1}^m{\bs{y}_i^{(k)}} = \frac1m \bs{1}^T \bs{Y}^{(k)}$ and $\bs{\bar{g}}_{\bs{v}}^{(k)} = \frac1m \sum_{i=1}^m{\bs{g}_i^{(k)} v_i^{(k)}} = \frac1m \bs{1}^T \bs{\Lambda}_{\bs{v}}^{(k)} \bs{G}^{(k)}$.
\end{definition}
Applying the average vector notation of Definition~\ref{def:avg_vec} to the matrix form of updates given in Eq.~\eqref{eqn:matrix_update}, we get the following.
\begin{equation} \label{eqn:averages}
    \begin{gathered}
        \bs{\bar{x}}_{\bs{\hat{\phi}}}^{(k+1)} = \bs{\hat{\phi}}^T \bs{\hat{A}}^{(k)} \bs{X}^{(k)} - \bs{\hat{\phi}}^T \bs{\Lambda}_{\bs{\eta}}^{(k)} \bs{\hat{B}}^{(k)} \bs{Y}^{(k)},
        \\
        \bs{\bar{y}}^{(k+1)} = \bs{\bar{y}}^{(k)} + \bs{\bar{g}}_{\bs{v}}^{(k+1)} - \bs{\bar{g}}_{\bs{v}}^{(k)}.
    \end{gathered}
\end{equation}
Recursively expanding $\bs{\bar{y}}^{(k+1)}$ with the initialization discussed for Eq.~\eqref{eqn:vector_update} as $\bs{\bar{y}}^{(0)} = \bs{\bar{g}}_{\bs{v}}^{(0)}$, we will have
\begin{equation} \label{eqn:ybar}
    \bs{\bar{y}}^{(k+1)} = \bs{\bar{g}}_{\bs{v}}^{(k+1)},
\end{equation}
for all $k \geq 0$.
Note that in our notation, model parameters, tracked gradients and stochastic gradients {\small $\bs{x}_i^{(k)}, \bs{y}_i^{(k)}, \bs{g}_i^{(k)} \in \mathbb{R}^{n \times 1}$} are column vectors, but their corresponding (weighted) averages {\small $\bs{\bar{x}}_{\bs{\hat{\phi}}}^{(k)}, \bs{\bar{y}}^{(k)}, 
\bs{\bar{g}}_{\bs{v}}^{(k)} \in \mathbb{R}^{1 \times n}$} are row vectors.

All supporting Lemmas are given in Appendices~\ref{ssec:intermed} and \ref{ssec:main_lemmas}, with all constant scalars summarized in Appendix~\ref{app:supp:table}. Also, refer to Appendix~\ref{app:supp:prooftree} for an illustration of how our theoretical results are related to each other. We proceed with the following definition.
\begin{definition}[Consensus Error] \label{def:error_vec}
    Let us collect the expected error terms defined in Lemmas~\ref{lemma:xdispersion} and \ref{lemma:ydispersion} in a single vector $\bs{\varsigma}^{(k)} \in \mathbb{R}^{2 \times 1}$. We have
    \vspace{-1mm}
    \begin{equation}
        \bs{\varsigma}^{(k)} =
        \begin{bmatrix}
            \mathbb{E}\left[ \left\| \bs{X}^{(k)} - \bs{1} \bs{\bar{x}}_{\bs{\hat{\phi}}}^{(k)} \right\|_{\bs{\hat{\phi}}}^2 \right]
            \\
            \mathbb{E}\left[ \left\| \bs{\Lambda}_{\bs{\hat{\pi}}}^{-1} \bs{Y}^{(k)} - m \bs{1} \bs{\bar{y}}^{(k)} \right\|_{\bs{\hat{\pi}}}^2 \right]
        \end{bmatrix}.
    \end{equation}
\end{definition}
Applying the error vector notation of Definition~\ref{def:error_vec} jointly to Lemmas~\ref{lemma:xdispersion} and \ref{lemma:ydispersion}, we obtain the following
\begin{equation} \label{eqn:varsigma}
    \bs{\varsigma}^{(k+1)} \le \bs{\Psi}^{(k)} \bs{\varsigma}^{(k)} + \bs{\gamma}^{(k)} \mathbb{E}\left[ \left\| \nabla{F}(\bs{\bar{x}}_{\bs{\hat{\phi}}}^{(k)}) \right\|^2 \right] + \bs{\omega}^{(k)},
\end{equation}
where {\small $\bs{\Psi}^{(k)} = [\psi_{ij}^{(k)}]_{1 \le i,j \le m} \in \mathbb{R}^{2 \times 2}$}, {\small $\bs{\gamma}^{(k)} = [\gamma_1^{(k)}, \gamma_2^{(k)}]^T \in \mathbb{R}^{2 \times 1}$} and {\small $\bs{\omega}^{(k)} = [\omega_1^{(k)}, \omega_2^{(k)}]^T \in \mathbb{R}^{2 \times 1}$}.
\revised{A key analytical challenge detailed in Appendices~\ref{ssec:intermed} and \ref{ssec:main_lemmas} is to form a tight bound for Eq.~\eqref{eqn:varsigma} such that it would decrease over time.}

\subsection{Propositions} \label{ssec:props}

In our first proposition, we lay out the sufficient constraints on the learning rates $\eta_i^{(k)}$ to ensure that the consensus error in Eq.~\eqref{eqn:varsigma} decreases over time.
\begin{proposition}[Spectral Radius of Transition Matrix] \label{proposition:lr}
    Let Assumptions~\ref{assump:lipschitz}-\ref{assump:graph} hold, and the learning rates $\eta_i^{(k)}$ satisfy
        \vspace{-1mm}
    \begin{equation} \label{eqn:lr_constraint}
        \begin{aligned}
            & \begin{aligned}
                \eta_i^{(k)} < \frac1{2 \sqrt{\kappa_3}} \min\Bigg\lbrace & \frac1{4 \sqrt{m+1}}  \frac{1 - \tilde{\rho}_A}{\sqrt{1 + \tilde{\rho}_A}} \frac{1 - \tilde{\rho}_B}{\sqrt{1 + \tilde{\rho}_B}}
                \\
                & \frac1{\sqrt{(2 \kappa_4 + m \kappa_5 \hat{\rho}_{0,A})}},
            \end{aligned}
            \\
            & \quad \sqrt{1 - \tilde{\rho}_A} \frac1{ \sqrt{\kappa_5}}, \frac1{2 \sqrt{m (m + 1)}} \frac{1 - \tilde{\rho}_B}{\sqrt{1 + \tilde{\rho}_B}} \frac1{\sqrt{\kappa_5}} \Bigg\rbrace,
        \end{aligned}
    \end{equation}
    then the spectral radius of $\bs{\Psi}^{(k)}$ defined in Eq.~\eqref{eqn:varsigma} is strictly less than one, i.e., $\rho(\bs{\Psi}^{(k)}) < 1$. The values of $\kappa_3$, $\kappa_4$, $\kappa_5$, $\hat{\rho}_{0,A}$, $\tilde{\rho}_A$ and $\tilde{\rho}_B$ are given in Table~\ref{tab:symbols_lemmas} in Appendix~\ref{app:supp:table}.

    \begin{proof}
        See Appendix~\ref{app:proposition:lr}.
    \end{proof}
\end{proposition}
\revised{Proposition~\ref{proposition:lr} captures all parameters related to the non-convex loss and the directed graph in a constraint over the learning rate. To understand the terms in Eq.~\eqref{eqn:lr_constraint}, note that in the less complex non-sporadic centralized/FL training with uniform gradient estimation variance, we get $\hat{\rho}_{0,A} = \tilde{\rho}_A = \tilde{\rho}_B = 0$, $\kappa_3 = 1 / m^2$, $\kappa_4 = 5m\max(L_i)^2$ and $\kappa_5 = 5\max(L_i)^2$. Thus, the constraint in Eq.~\eqref{eqn:lr_constraint} simplifies to $\eta_i^{(k)} < \frac1{16 \sqrt{5} \max(L_i)}$, i.e., an inverse dependence of the learning rate on smoothness, which is a standard result for gradient-based iterative methods \cite{bottou2018optimization}.}

We provide two corollaries to Proposition~\ref{proposition:lr} in Appendix~\ref{app:corollaries}. Having established a decreasing consensus error in Proposition~\ref{proposition:lr},
we next analyze the global loss decrement.
\begin{proposition}[Loss Descent Constraints] \label{proposition:loss}
    Let Assumptions~\ref{assump:lipschitz}-\ref{assump:graph} hold, and each client $i \in \mathcal{M}$ use a constant learning rate $\eta_i^{(k)} = \eta_i^{(0)}$ that satisfies the conditions outlined in Corollary~\ref{corollary:nonincreasing}, to have $\rho(\bs{\Psi}^{(k)}) = \rho_\Psi < 1$ for all $k \geq 0$ according to Corollary~\ref{corollary:constant}. Let $\Gamma_1 = 1 + \frac{8 \Gamma_3}m \frac{\max(\hat{\phi}_i \hat{\pi}_i)}{\min(\hat{\phi}_i \hat{\pi}_i)} \frac{\Gamma_2 \kappa_7 \tilde{\rho}_B}{1 - \rho_\Psi}$ for arbitrary $\Gamma_2, \Gamma_3 > 1$. Let the gradient calculation probabilities, communication probabilities, batch sizes, and learning rate heterogeneities satisfy
    \begin{equation}
        \begin{aligned}
            & p_i \geq 1 - \frac{m \hat{\pi}_i}{9 \Gamma_1 \delta_{1,i}^2}, \qquad \hat{p}_{ij} \in [\max(\hat{r}_A, \hat{r}_B, \hat{r}'_B), 1],
            \\
            & B_i \geq \frac{D_i}{1 + \frac{m \hat{\pi}_i}{6 \Gamma_1 \sigma_{1,i}^2} D_i}, \qquad \frac{\bar{\eta}^{(0)}}{\eta_i^{(0)}} \geq \frac1{\Gamma_3},
        \end{aligned}
    \end{equation}
    for each client $i \in \mathcal{M}$, respectively. Note that $\hat{r}_A, \hat{r}_B < 1$ are defined in Lemma~\ref{lemma:sporadic_contraction}, and {\small $\hat{r}'_B = \frac12 (1 + \sqrt{\max(0, 1 - \hat{\tau}'_B)})$} with {\small $\hat{\tau}'_B = \frac{2m}{9 (m - 1) \kappa_2 \Gamma_1 b_{ij}^2} \frac{\min(\hat{\pi}_i)^2}{\max(\hat{\pi}_i)}$}. Then, If the learning rates further satisfies
    \begin{equation}
        \eta_i^{(0)} < \sqrt{\frac{1 - 1 / \Gamma_3}{\Gamma_1} \frac{\kappa_{10}}{\kappa_9} (1 - \rho_\Psi)},
    \end{equation}
    where {\small $\kappa_9 = ( \frac{m (1 + 3 \hat{\rho}_{0,B})}{3 \kappa_7 \tilde{\rho}_B} + 1 ) \kappa_4 \kappa_6 + \frac{\kappa_5}{m}$} and {\small $\kappa_{10} = \frac8{3 \kappa_2 \kappa_3} \frac{\max(\hat{\phi}_i \hat{\pi}_i)}{\max(\hat{\phi}_i)}$},
    the loss decrement (coefficients of {\small $\mathbb{E}[ \| \nabla{F}(\bs{\bar{x}}_{\bs{\hat{\phi}}}^{(r)}) \|^2 ]$} in Eq.~\eqref{eqn:loss_expanded} of Appendix~\ref{ssec:main_lemmas}) is lower bounded as {\small $q \max(\eta_i^{(0)})$}, where {\small $q = \frac{m^2}2 (1 - \frac1{\Gamma_1}) (1 - \frac1{\Gamma_2}) \frac1{\Gamma_3} \min(\hat{\phi}_i \hat{\pi}_i)$}.

    \begin{proof}
        See Appendix~\ref{app:proposition:loss}
    \end{proof}
\end{proposition}
Note how a small gradient diversity, i.e., $\delta_{1,i}^2 < \frac{m \hat{\pi}_i}{9 \Gamma_1}$, does not impose any constraint on the allowed computation probability, i.e., $p_i \in (0, 1]$. Also, having a uniform bound on the gradient estimation variance, i.e., $\sigma_{1,i}^2 = 0$, allows us to choose any desired batch size, i.e., $B_i \in [1, D_i]$. However, in extreme cases for gradient diversity ($\delta_{1,i}^2 \to \infty$) and non-uniform gradient estimation variance ($\sigma_{1,i}^2 \to \infty$), we must have $p_i = 1$ and $B_i = D_i$. Finally, using coordinated learning rates $\eta_i^{(0)} = \bar{\eta}^{(0)}$ for all clients automatically satisfies condition $\bar{\eta}^{(0)} / \eta_i^{(0)} \geq 1 / \Gamma_3$ since $\Gamma_3 > 1$.

\subsection{Main Theorem} \label{ssec:thm}

We establish the convergence rate of \texttt{Spod-GT} to stationary points and its optimality gap in the following theorem.
\begin{theorem}[Convergence for Non-Convex Loss] \label{theorem:constant}
    Let Assumptions~\ref{assump:lipschitz}-\ref{assump:graph} hold, and each client $i \in \mathcal{M}$ use a constant learning rate $\eta_i^{(k)} = \eta_i^{(0)}$ that satisfies the constraints in Propositions~\ref{proposition:lr} and \ref{proposition:loss}. If the gradient calculation probabilities $p_i$, batch sizes $B_i$ and learning rate heterogeneities $\bar{\eta}^{(0)} / \eta_i^{(0)}$ also satisfy the constraints of Proposition~\ref{proposition:loss}, the stationary point of our approach can be characterized as
    \begin{equation} \label{eqn:conv_rate}
        \begin{aligned}
            & \frac1{k + 1} \sum_{r=0}^k{\mathbb{E}\left[ \left\| \nabla{F}(\bs{\bar{x}}_{\bs{\hat{\phi}}}^{(r)}) \right\|^2 \right]} \le \frac{\overbrace{F(\bs{\bar{x}}_{\bs{\hat{\phi}}}^{(0)}) - F^\star}^{\revised{\text{Optimality error}}}}{\mathcal{O}\left( \max(\eta_i^{(0)}) (k + 1) \right)}
            \\
            & + \frac{\overbrace{\mathbb{E}\left[ \left\| \bs{X}^{(0)} \hspace{-0.5mm} - \hspace{-0.5mm} \bs{1} \bs{\bar{x}}_{\bs{\hat{\phi}}}^{(0)} \right\|_{\bs{\hat{\phi}}}^2 \right]}^{\revised{\text{Model Consensus Error}}}}{\mathcal{O}(k + 1)} + \frac{\overbrace{\mathbb{E}\left[ \left\| \bs{\Lambda}_{\bs{\hat{\pi}}}^{-1} \bs{Y}^{(0)} \hspace{-0.5mm} - \hspace{-0.5mm} m \bs{1} \bs{\bar{y}}^{(0)} \right\|_{\bs{\hat{\pi}}}^2 \right]}^{\revised{\text{GT Consensus Error}}}}{\mathcal{O}(k + 1)}
            \\
            & + \mathcal{O}\Big( \underbrace{\overline{\sigma_0^2 p B^{-1} (1 - B/D)}}_{\revised{\text{Gradient Estimation}}} + \underbrace{\overline{3 \overline{(1 - p) \delta_0^2}}}_{\revised{\text{Sporadicity}}} \Big) \max(\eta_i^{(0)})^2
            \\[-0.7ex]
            & + \mathcal{O}\Big( \overbrace{\overline{\sigma_0^2 p B^{-1} (1 - B/D)}}^{\revised{\text{Variance Error}}} + \overbrace{5 \overline{(1 - p) \delta_0^2}}^{\revised{\text{Error}}} \Big),
        \end{aligned}
    \end{equation}
    with a convergence rate of $\mathcal{O}(1 / k)$. By Letting $k \to \infty$, we obtain the asymptotic stationarity gap as follows:
    \begin{equation} \label{eqn:gap}
        \begin{aligned}
            & \lim_{k \to \infty}\frac1{k+1} \sum_{r=0}^k{\mathbb{E}\left[ \left\| \nabla{F}(\bs{\bar{x}}_{\bs{\hat{\phi}}}^{(r)}) \right\|^2 \right]} \le 
            \\
            & \mathcal{O}\left( \overline{\sigma_0^2 p B^{-1} (1 - B/D)} + 3 \overline{(1 - p) \delta_0^2} \right) \max(\eta_i^{(0)})^2
            \\
            & + \mathcal{O}\left( \overline{\sigma_0^2 p B^{-1} (1 - B/D)} + 5 \overline{(1 - p) \delta_0^2} \right).
        \end{aligned}
    \end{equation}
    \begin{proof}
        See Appendix~\ref{app:theorem:constant}.
    \end{proof}
\end{theorem}
\revised{Theorem~\ref{theorem:constant} states that \texttt{Spod-GT} achieves a stationary point of the non-convex loss, up to an asymptotic neighborhood controlled by design parameters of our method. Eq.~\eqref{eqn:conv_rate} shows how the errors for optimality, model consensus and GT consensus decrease over time.}
We see that if clients perform GT in every iteration and use a full batch, i.e., $p_i = 1$ and $B_i = D_i$, we get a zero gap per Eq.~\eqref{eqn:gap}. Next, we obtain a general result which shows that under a careful design, \texttt{Spod-GT} achieves a zero stationarity gap.

\begin{corollary}[Arbitrarily Small Bound]
    Let $K \geq 0$ denote the total number of training iterations. If constant learning rates, gradient  calculation probabilities and batch sizes are chosen as $\eta_i^{(0)} = \mathcal{O}(1 / \sqrt{K + 1})$, $p_i = 1 - \mathcal{O}(1 / \sqrt{K + 1})$ and $B_i = \frac{D_i}{1 + D_i / \mathcal{O}(\sqrt{K + 1})}$ by each client $i \in \mathcal{M}$, then Theorem~\ref{theorem:constant} implies that
        \vspace{-2mm}
    \begin{equation}
        \begin{aligned}
            & \frac1{K + 1} \sum_{r=0}^K{\mathbb{E}\left[ \left\| \nabla{F}(\bs{\bar{x}}_{\bs{\hat{\phi}}}^{(r)}) \right\|^2 \right]} \le \frac{F(\bs{\bar{x}}_{\bs{\hat{\phi}}}^{(0)}) - F^\star}{\mathcal{O}(\sqrt{K + 1})}
            \\
            & + \frac{\mathbb{E}\left[ \left\| \bs{X}^{(0)} \hspace{-0.5mm} - \hspace{-0.5mm} \bs{1} \bs{\bar{x}}_{\bs{\hat{\phi}}}^{(0)} \right\|_{\bs{\hat{\phi}}}^2 \right]}{\mathcal{O}(K + 1)} + \frac{\mathbb{E}\left[ \left\| \bs{\Lambda}_{\bs{\hat{\pi}}}^{-1} \bs{Y}^{(0)} \hspace{-0.5mm} - \hspace{-0.5mm} m \bs{1} \bs{\bar{y}}^{(0)} \right\|_{\bs{\hat{\pi}}}^2 \right]}{\mathcal{O}(K + 1)}
            \\
            & + \frac{\overline{\sigma_0^2 p} + 5 \overline{\delta_0^2}}{\mathcal{O}(\sqrt{K + 1})} + \frac{\overline{\sigma_0^2 p} + 5 \overline{\delta_0^2}}{\mathcal{O}((K + 1)^{3/2})}.
        \end{aligned}
    \end{equation}
    Therefore, by increasing $K$, we can decrease the stationarity gap to an arbitrarily small value, i.e.,
    \begin{equation}
        \lim_{K \to \infty}{\frac1{K+1} \sum_{r=0}^K{\mathbb{E}\left[ \left\| \nabla{F}(\bs{\bar{x}}_{\bs{\hat{\phi}}}^{(r)}) \right\|^2 \right]}} = 0.
    \end{equation}
    Note that asymptotically this means clients compute gradients at each iteration ($p_i \to 1$), use full-batch ($B_i \to D_i$) and use an arbitrarily small learning rate ($\eta_i^{(0)} \to 0$).
\end{corollary}
    \vspace{-1mm}

\section{Experiments}
    \vspace{-1mm}

\begin{figure*}[t]
    \centering
    \begin{subfigure}[t]{\textwidth}
        \centering
        \includegraphics[width=0.7\textwidth]{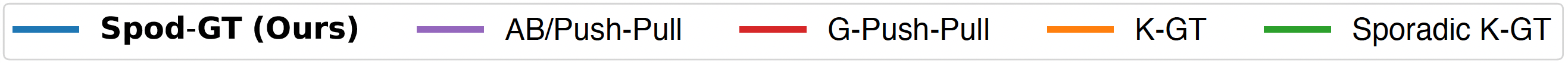}
    \end{subfigure}
    
    \begin{subfigure}[t]{0.24\textwidth}
        \centering
        \includegraphics[width=\textwidth]{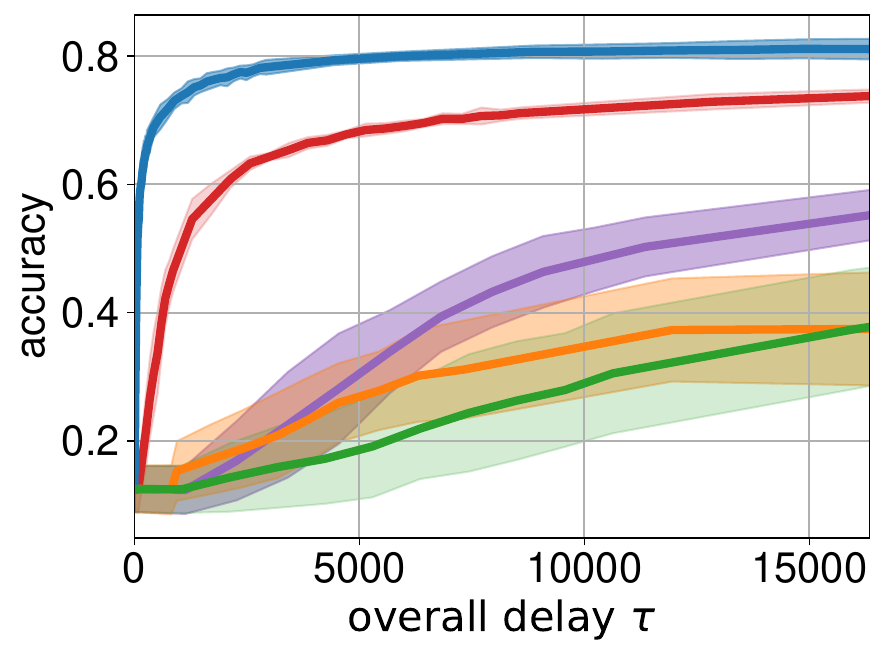}
        \caption{FMNIST, IID.}
        \label{fig:svm_eff:iid}
    \end{subfigure}
    \hfill
    \begin{subfigure}[t]{0.24\textwidth}
        \centering
        \includegraphics[width=\textwidth]{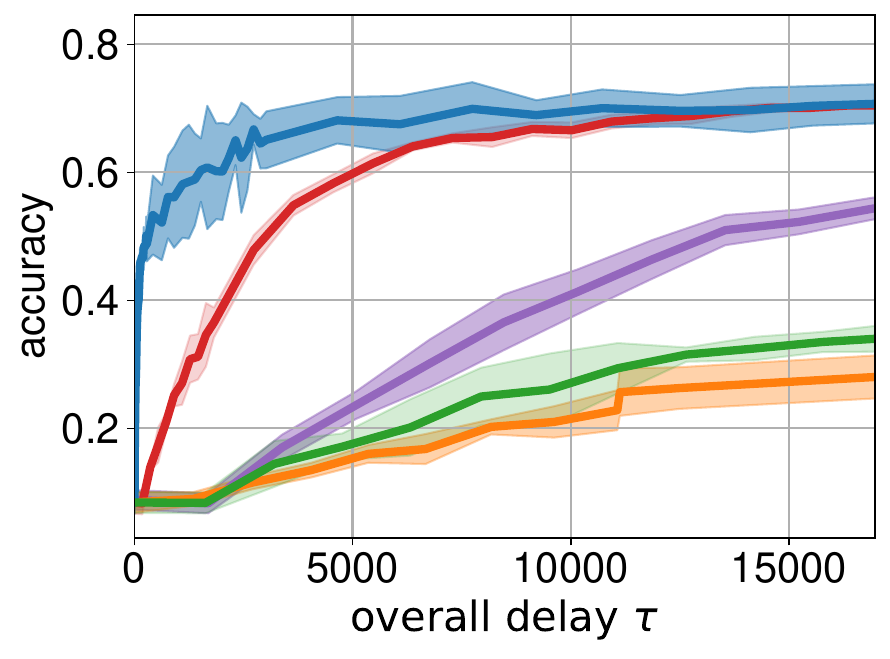}
        \caption{FMNIST, Non-IID.}
        \label{fig:svm_eff:noniid}
    \end{subfigure}
    \hfill
    \begin{subfigure}[t]{0.24\textwidth}
        \centering
        \includegraphics[width=\textwidth]{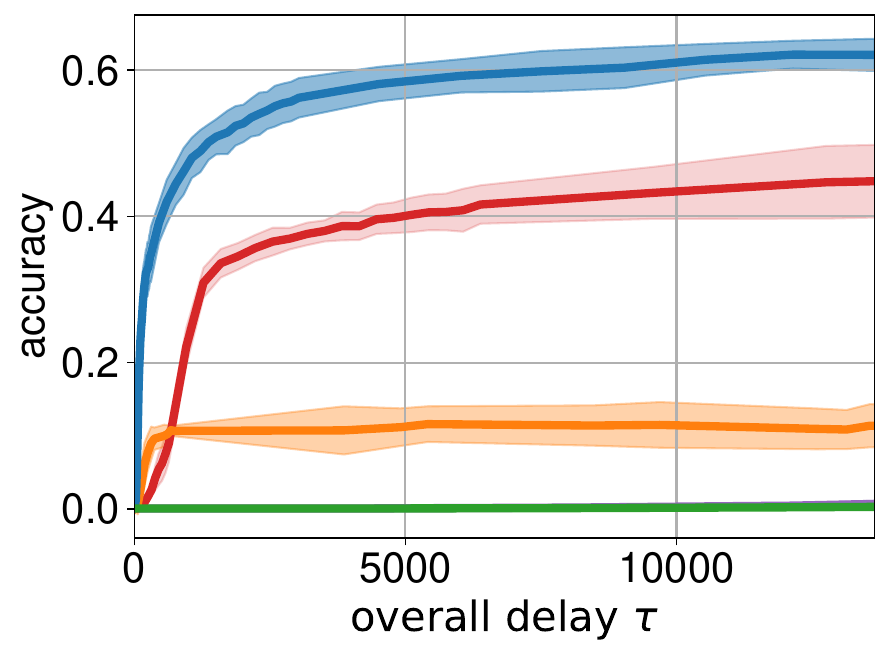}
        \caption{CIFAR-10, IID.}
        \label{fig:vgg_eff:iid}
    \end{subfigure}
    \hfill
    \begin{subfigure}[t]{0.24\textwidth}
        \centering
        \includegraphics[width=\textwidth]{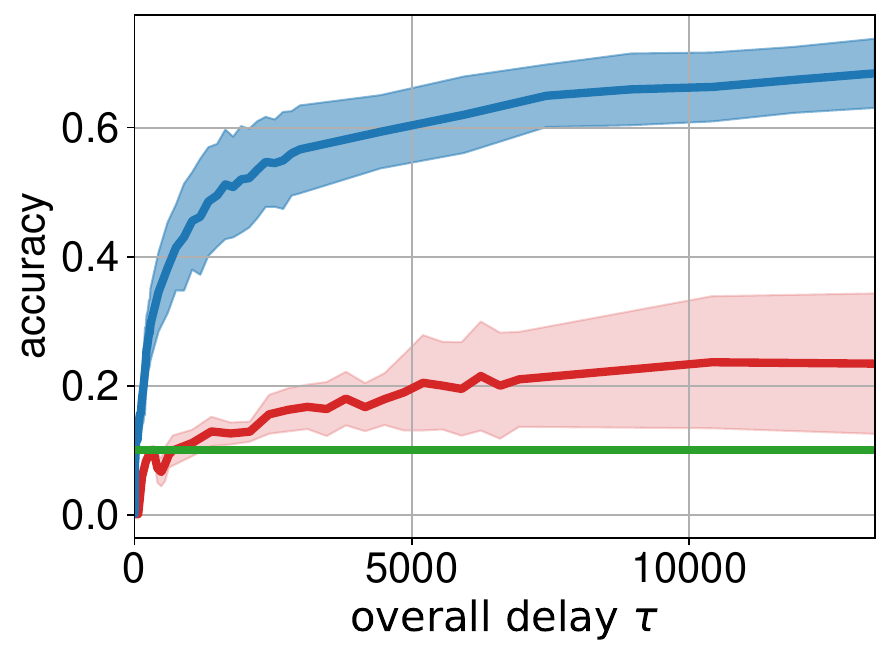}
        \caption{CIFAR-10, Non-IID.}
        \label{fig:vgg_eff:noniid}
    \end{subfigure}
    \caption{\small Accuracy vs. latency plots. \texttt{Spod-GT} achieves the target accuracy much faster with less  delay, emphasizing the benefit of sporadicity in DFL over directed graphs for GT iterations and aggregations simultaneously.}
    \label{fig:eff}
\end{figure*}

\begin{figure*}[t]
    \centering
    \begin{subfigure}[t]{\textwidth}
        \centering
        \includegraphics[width=0.75\textwidth]{images/legend.png}
    \end{subfigure}
    \begin{subfigure}[t]{0.24\textwidth}
        \centering
        \includegraphics[width=\textwidth]{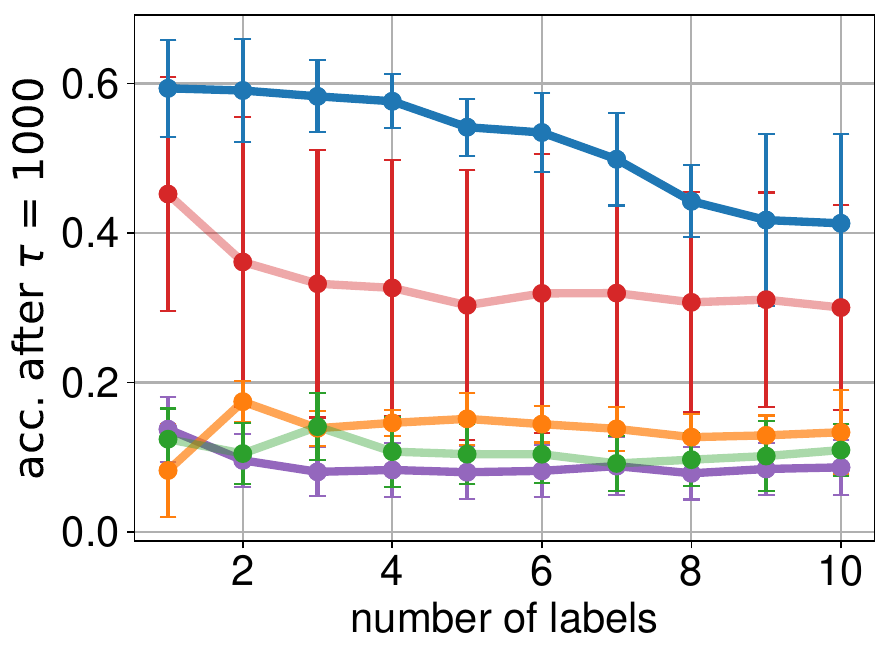}
        \caption{Varying number of labels per client.}
        \label{fig:svm_data_dist}
    \end{subfigure}
    \hfill
    \begin{subfigure}[t]{0.24\textwidth}
        \centering
        \includegraphics[width=\textwidth]{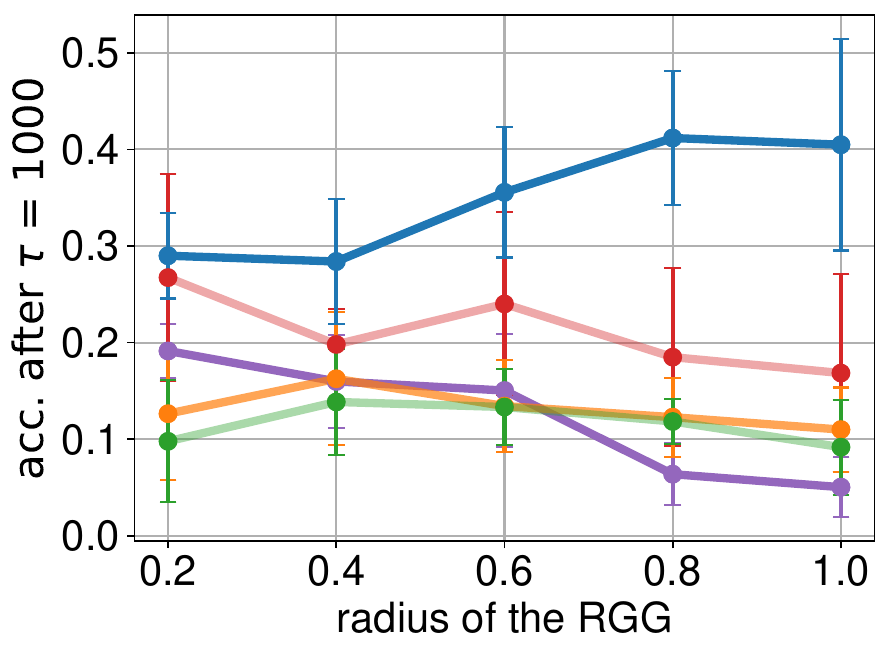}
        \caption{Varying the radius of random geometric  graph.}
        \label{fig:svm_graph_conn}
    \end{subfigure}
    \hfill
    \begin{subfigure}[t]{0.24\textwidth}
        \centering
        \includegraphics[width=\textwidth]{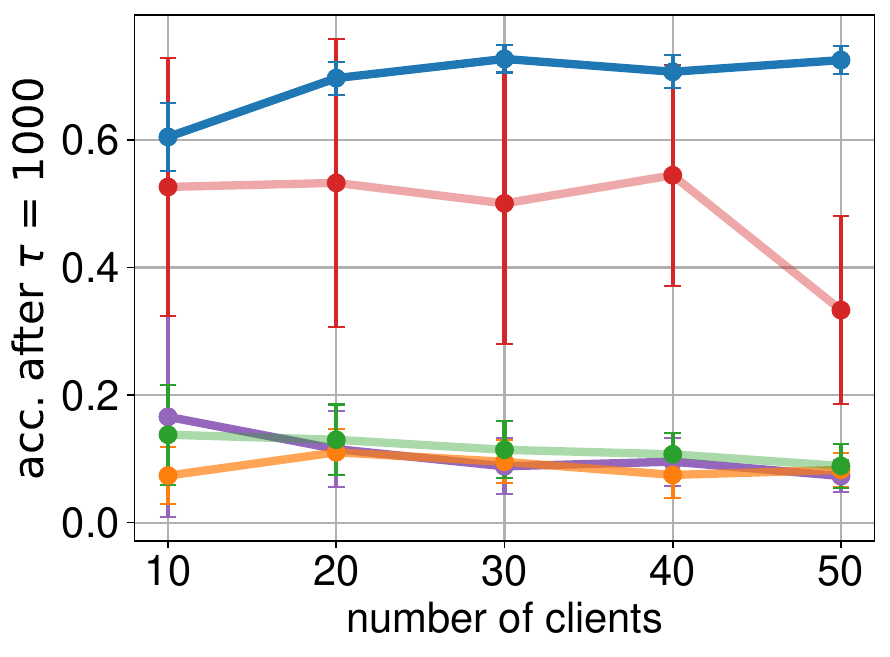}
        \caption{Varying the number of clients $m$ in the network.}
        \label{fig:svm_num_clients}
    \end{subfigure}
    \hfill
    \begin{subfigure}[t]{0.24\textwidth}
        \centering
        \includegraphics[width=\textwidth]{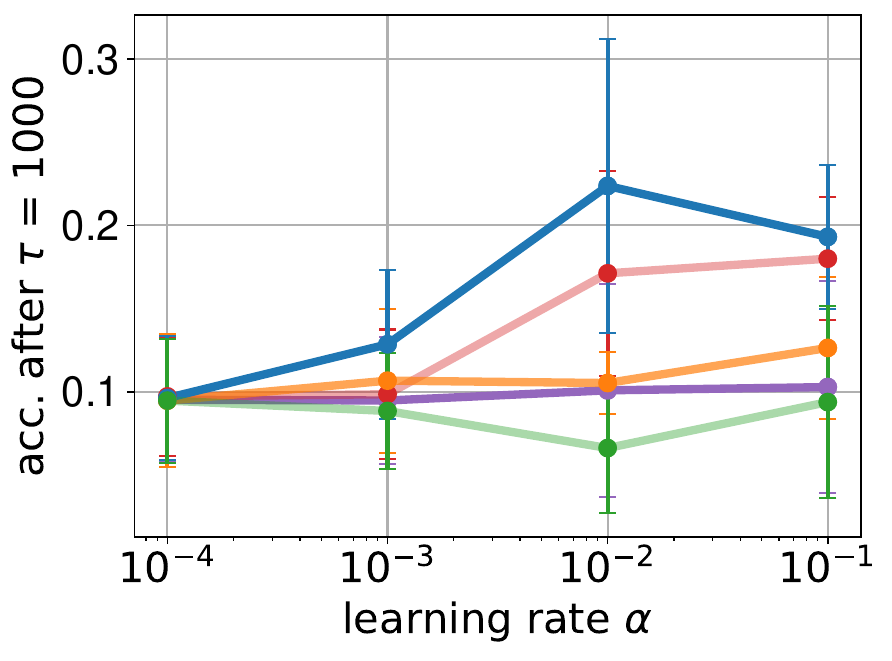}
        \caption{Varying the learning rate $\eta$ for all clients.}
        \label{fig:svm_learning_rate}
    \end{subfigure}
    \caption{\small Effects of system parameters on FMNIST. The overall results confirm the advantage of \texttt{Spod-GT}.}
    \label{fig:fmnist_ablations}
        \vspace{-3mm}
\end{figure*}

\subsection{Setup}
\textbf{Models and Datasets.} To evaluate our methodology, we consider image classification tasks using the Fashion-MNIST (FMNIST) \citep{xiao2017fashion} and CIFAR-10 \citep{krizhevsky2009learning} datasets. We train a Support Vector Machine (SVM)
for FMNIST, while adopting ResNet-18 \citep{he2016deep} to train for CIFAR-10.

\textbf{Settings.} By default, we consider $m = 10$ clients connected via a random geometric graph (RGG) \citep{penrose2003random}. We design our directed network graph $\mathcal{G}$ by first generating an undirected RGG, then making all edges bidirectional with heterogeneous weights for each link. We adopt a constant learning rate $\gamma = 0.01$ and a batch size of $B=16$ for all clients. The probability of gradient computation $p_i$ for each client and the probability of transmission $\hat{p}_{ij}$ through a link $(i,j) \in \mathcal{E}$ on the directed graph are randomly chosen according to the Beta distribution $\Beta(0.5, 0.5)$, yielding an inverted bell-shaped distribution, corresponding to significant resource heterogeneity. We consider two different data distribution scenarios: (i) IID, where each client receives samples from all $10$ classes in the dataset, and (ii) non-IID, where each client receives samples from just $1$ class. We measure the test accuracy of each scheme achieved over the average total delay incurred up to iteration $k$. Specifically, {\small $\bar{\tau}_{\text{total}}^{(k)} = \bar{\tau}_{\text{trans, in}}^{(k)} + \bar{\tau}_{\text{proc}}^{(k)} + \bar{\tau}_{\text{trans, out}}^{(k)}$}, in which {\small $\bar{\tau}_{\text{trans, in}}^{(k)} = \frac1m \sum_{i=1}^m{(1 / | \mathcal{N}_i^{\text{in}} |) \sum_{j=1}^m{(\hat{v}_{ij}^{(k)} / \hat{p}_{ij})}}$}, {\small $\bar{\tau}_{\text{trans, out}}^{(k)} = \frac1m \sum_{i=1}^m{(1 / | \mathcal{N}_i^{\text{out}} |) \sum_{j=1}^m{(\hat{v}_{ij}^{(k)} / \hat{p}_{ij})}}$} and {\small $\bar{\tau}_{\text{proc}}^{(k)} = \frac1m \sum_{i=1}^m{(v_i^{(k)} / p_i)}$} are the transmission delays per-client incurred across the links of in-neighbors and out-neighbors, and the processing delays incurred across clients in iteration $k$, respectively. For a fair comparison, we determine the aggregation interval $K$  for the K-GT algorithm based on these $p_i$, i.e., {\small $K = \lceil \frac1m \sum_{i=1}^m{(1/p_i)} \rceil$}.
We run the experiments five times in each setup and present the mean and one-sigma standard deviation error bars.

\textbf{Baselines.} We compare \texttt{Spod-GT} with four baselines that use GT in directed graphs: (a) AB/Push-Pull \citep{nedic2025ab},
where GT and local directed aggregations occur at every iteration; (b) the G-Push-Pull algorithm \citep{pu2020push}, which is a special case of our \texttt{Spod-GT} framework with sporadic aggregations but conducting GT at every iteration; (c) K-GT \citep{liu2025decentralized} with K local GT steps between aggregations; and (d) Sporadic K-GT, which is another special case of our \texttt{Spod-GT} framework with aggregations at every iteration but performing sporadic GT.

\subsection{Results and Discussion}

\textbf{Accuracy vs. Total Delay.} Fig.~\ref{fig:eff} illustrates the accuracy achieved per total delay for our \texttt{Spod-GT} algorithm and other GT baselines. We observe that when the clients in a DFL setup adopt our \texttt{Spod-GT} training algorithm, both for the Fashion-MNIST and the CIFAR-10 datasets, they are able to consistently achieve higher accuracies faster compared to their non-sporadic counterparts. For example, if we compare all algorithms after $\bar{\tau}_{\text{total}}^{(k)} = 5000$ in Figs.~\ref{fig:svm_eff:iid}-\ref{fig:vgg_eff:noniid}, we could see a $7$ - $40\%$ accuracy improvement for our \texttt{Spod-GT} algorithm compared to the second best baseline.

\textbf{Ablation Studies.} In Fig.~\ref{fig:fmnist_ablations}, we conduct ablation studies on key system parameters of our methodology for the FMNIST dataset. See Appendix~\ref{app:additional_exps} for a similar ablation study on the CIFAR-10 dataset. In Fig.~\ref{fig:svm_data_dist}, we vary the number of labels each client holds, with $1$ corresponding to the extreme non-IID case, and $10$ corresponding to the completely IID case. We see that our \texttt{Spod-GT} algorithm is capable of consistently outperforming all baselines at different levels of data heterogeneity, confirming the benefit of integrating the notion of sporadicity in both communications and computations. In Fig.~\ref{fig:svm_graph_conn}, we vary the radius of the RGG, with higher radius indicating a denser connectivity in the network, and observe that
\texttt{Spod-GT} performs the best for all choices of radii. In Fig.~\ref{fig:svm_num_clients}, we vary the number of clients $m$ from $10$ to $50$ to validate that our methodology maintains its superiority for larger network sizes. Finally, in Fig.~\ref{fig:svm_learning_rate}, we do a hyperparameter analysis of the learning rate $\eta$ from $10^{-4}$ to $10^{-1}$, which shows that the choice of $\eta = 0.01$ is the optimal choice for both datasets. In addition, we see that \texttt{Spod-GT} achieves a higher accuracy compared to all baselines for any choice of learning rate $\eta$.

\section{Conclusion}

We developed \textit{Sporadic Gradient Tracking} (\texttt{Spod-GT}), a resource-efficient DFL algorithm based on GT that designed to work on general non-doubly stochasticable directed graphs. By adjusting each client's participation in both (i) gradient computations and (ii) peer-to-peer communications according to its processing and transmission resource availabilities, respectively, our methodology achieves significant improvements over non-heterogeneity-aware optimization algorithms. We carried out a thorough analysis of our methodology, proving consensus and optimality guarantees despite intermittent client participation in DFL. We validate the efficacy of our approach through numerical experiments.








\section*{Impact Statement}


This paper presents work whose goal is to advance the field of Machine
Learning. There are many potential societal consequences of our work, none
which we feel must be specifically highlighted here.



\bibliography{example_paper}
\bibliographystyle{icml2026}

\newpage
\appendix
\onecolumn


\section{Supplementary Technical Details}

\subsection{Notation} \label{app:supp:not}

We use boldface lowercase $\bs{z} \in \mathbb{R}^{n \times 1}$ for vectors and boldface uppercase $\bs{Z} \in \mathbb{R}^{m \times n}$ for matrices. Subscripts and superscripts are used to denote the client and the iteration, respectively, i.e., $\bs{z}_i^{(k)}$ belongs to client $i$ at iteration $k$. We use the Frobenius norm and inner product for all matrices, i.e., $\| \bs{Z} \|^2 = \| \bs{Z} \|_F^2 = \sum_{i=1}^m{\| \bs{Z}_i \|^2}$ and $\langle \bs{Z}, \bs{W} \rangle = \langle \bs{Z}, \bs{W} \rangle_F = \sum_{i=1}^m{\langle \bs{z}_i, \bs{w}_i \rangle}$, where $\bs{Z} = [\bs{z}_1, ..., \bs{z}_m]^T$ and $\bs{W} = [\bs{w}_1, ..., \bs{w}_m]^T$, and we drop the subscript $F$ to avoid repetition. Given a non-negative vector $\bs{\phi} \in \mathbb{R}^{n \times 1}$, we define the row-based matrix norm and the inner product imposed by a vector as $\| \bs{Z} \|_{\bs{\phi}}^2 = \sum_{i=1}^m{\phi_i \| \bs{z}_i \|^2}$ and $\langle \bs{Z}, \bs{W} \rangle_{\bs{\phi}} = \sum_{i=1}^m{\phi_i \langle \bs{z}_i, \bs{w}_i \rangle}$, respectively, which is a generalization of the Frobenius norm and inner products if $\bs{\phi} = \bs{1}_m$. For a matrix $\bs{Z} = [\bs{z}_1, ..., \bs{z}_m]^T$, we define its row-based average as $\bs{\bar{z}} = \frac1m \sum_{i=1}^m{\bs{z}_i}$, and its weighted average imposed by a vector $\bs{\phi}$ as $\bs{\bar{z}}_{\bs{\phi}} = \sum_{i=1}^m{\phi_i \bs{z}_i}$. Also, given $N$ vectors $\bs{z}_1, ..., \bs{z}_N$, we define their element-wise average as $\overline{z_1...z_N} = \frac1m \sum_{i=1}^m{z_{1,i}...z_{N,i}}$. We use the notation $\bs{\Lambda}_{\bs{z}} = \diag(\boldsymbol{z}) \in \mathbb{R}^{n \times n}$ to indicate a diagonal matrix whose diagonal entries are the elements of the vector $\bs{z}$.

Let $\mathcal{F}_\sigma^{(k)}$ the $\sigma$-algebra generated by the set of random variables {\small $\{ \mathcal{B}_i^{(r)} \}_{\substack{0 \le r \le k-1 \\ i \in \mathcal{M}}} \cup \{ v_i^{(r)} \}_{\substack{0 \le r \le k-1 \\ i \in \mathcal{M}}} \cup \{ \hat{v}_{ij}^{(r)} \}_{\substack{0 \le r \le k-1 \\ i, j \in \mathcal{M}}}$}. If $z^{(k)}$ is an event in the filtered probability space with filtration $\mathcal{F}_\sigma^{(k)}$, we use the notation $\mathbb{E}[z^{(k)}]$ for brevity to denote the conditional expectation $\mathbb{E}[z^{(k)} \,|\, \mathcal{F}_\sigma^{(k)}]$. Let $D(\mathcal{G})$ and $K(\mathcal{G})$ denote the graph diameter and maximal edge utility in the digraph $\mathcal{G}$, respectively (see \cite{nedic2025ab} for more details).

\subsection{Summary of Scalars} \label{app:supp:table}

We define several constant scalars throughout our theoretical analysis that depend on the parameters of our Assumptions given in Sec.~\ref{ssec:assumps}. Table~\ref{tab:symbols_lemmas} outlines the auxiliary constant scalars used in our Intermediary Lemmas and Propositions that are presented in Appendix~\ref{ssec:intermed} and Sec.~\ref{ssec:props}, respectively. Moreover, Table~\ref{tab:symbols_props} includes the constant scalars used in our Main Lemmas that are discussed in Appendix.~\ref{ssec:main_lemmas}, which are in turn are central to the proofs of our Propositions.

\begin{table}[h]
    \centering
    \caption{Auxiliary constant scalars defined throughout our intermediary Lemmas (Sec.~\ref{ssec:intermed}) and Propositions (Sec.~\ref{ssec:props}).}
    \begin{tabular}{c|c || c|c}
        \toprule
        Symbol & Value & Symbol & Value
        \\
        \midrule
        $\kappa_1$ & $\max( \frac{p_i}{\hat{\phi}_i} (2 \frac{(1 - B_i / D_i) \sigma_{1,i}^2}{B_i} \bar{L}^2 + 3 L_i^2))$ & $\tilde{\rho}_A$ & $\hat{\rho}_A + 2 \hat{\rho}_{0,A}$
        \\
        $\kappa_2$ & $2 \overline{\sigma_1^2 p B^{-1} (1 - B/D)} + 3 \overline{(1 - p) \delta_1^2} + 3$ & $\tilde{\rho}_B$ & $\hat{\rho}_B + \hat{\rho}_{0,B}$
        \\
        $\kappa_3$ & $\max(\hat{\phi}_i) ( \frac{\max(\hat{\pi}_i)^2}{\min(\hat{\pi}_i)} + 2m \max(\hat{\pi}_j \overline{b_{:j}^2 \hat{p}_{:j} (1 - \hat{p}_{:j})}) )$ & $\kappa_6$ & $\frac{1 + \hat{\rho}_A}{1 - \tilde{\rho}_A}$
        \\
        $\kappa_4$ & $\max\left( \frac{p_i}{\hat{\phi}_i} \left( 2 \frac{(1 - B_i / D_i) \sigma_{1,i}^2}{B_i} \bar{L}^2 + 5 L_i^2 \right) \right)$ & $\kappa_7$ & $2 (m + 1) \frac{1 + \tilde{\rho}_B}{1 - \tilde{\rho}_B}$
        \\
        $\kappa_5$ & $2 \overline{p \left( 2 \sigma_1^2 B^{-1} (1 - B/D) + 5 (1 - p) \delta_1^2 \right)} \bar{L}^2 + 5 \overline{p^2 L^2}$ & $\kappa_8$ & $\max( p_i ( 2 \frac{(1 - B_i / D_i) \sigma_{1,i}^2}{B_i} \bar{L}^2 + 3 L_i^2 ))$
        \\
        $\hat{\tau}_A$ & $16 m^2 (m - 1) D(\mathcal{G}) K(\mathcal{G}) \frac{\max(a_{ij}^+)^2}{\min(a_{ij}^+)^{2(m+1)}}$ & $\kappa_9$ & $( \frac{m (1 + 3 \hat{\rho}_{0,B})}{3 \kappa_7 \tilde{\rho}_B} + 1 ) \kappa_4 \kappa_6 + \frac{\kappa_5}{m}$
        \\
        $\hat{\tau}_B$ & $4 m^3 (m - 1) D(\mathcal{G}) K(\mathcal{G}) \frac{\max(b_{ij}^+)^2}{\min(b_{ij}^+)^{3m+2}}$ & $\kappa_{10}$ & $\frac8{3 \kappa_2 \kappa_3} \frac{\max(\hat{\phi}_i \hat{\pi}_i)}{\max(\hat{\phi}_i)}$
        \\
        $\hat{\rho}_{0,A}$ & $4 (m - 1) \frac{\max(\hat{\phi}_i)}{\min(\hat{\phi}_i)} \max( a_{ij}^2 \hat{p}_{ij} (1 - \hat{p}_{ij}))$ & $\hat{\tau}'_B$ & $\frac{2m}{9 (m - 1) \kappa_2 \Gamma_1 b_{ij}^2} \frac{\min(\hat{\pi}_i)^2}{\max(\hat{\pi}_i)}$
        \\
        $\hat{\rho}_{0,B}$ & $2 (m - 1) \frac{\max(\hat{\pi}_i)}{\min(\hat{\pi}_i)} \max( b_{ij}^2 \hat{p}_{ij} (1 - \hat{p}_{ij}) )$ & $\hat{r}'_B$ & $\frac12 (1 + \sqrt{\max(0, 1 - \hat{\tau}'_B)})$
        \\
        \bottomrule
    \end{tabular}
    \label{tab:symbols_lemmas}
\end{table}

\begin{table}[t]
    \centering
    \caption{Core constant scalars central to our main Lemmas given in Appendix~\ref{ssec:main_lemmas}.}
    \begin{tabular}{c|c}
        \toprule
        Symbol & Value
        \\
        \midrule
        $\psi_{11}^{(k)}$ & $\frac{1 + \tilde{\rho}_A}2 + m \kappa_1 \kappa_6 \kappa_3 \max(\eta_i^{(k)})^2$
        \\
        $\psi_{12}^{(k)}$ & $\kappa_6 \kappa_3 \max(\eta_i^{(k)})^2$
        \\
        $\gamma_1^{(k)}$ & $m^2 \kappa_2 \kappa_6 \kappa_3 \max(\eta_i^{(k)})^2$
        \\
        $\omega_1^{(k)}$ & $m^2 (\overline{\sigma_0^2 p B^{-1} (1 - B/D)} + 3 \overline{(1 - p) \delta_0^2}) \kappa_6 \kappa_3 \max(\eta_i^{(k)})^2$
        \\
        $\psi_{21}^{(k)}$ & $\kappa_7 \kappa_4 (1 + \psi_{11}^{(k)}) + \kappa_7 \kappa_5 \kappa_3 \kappa_1 \max(\eta_i^{(k)})^2 + m \kappa_7 \kappa_5 \hat{\rho}_{0,A}$
        \\
        $\psi_{22}^{(k)}$ & $\frac{1 + \tilde{\rho}_B}2 + \kappa_7 \kappa_4 \psi_{12}^{(k)} + m \kappa_7 \kappa_5 \kappa_3 \max(\eta_i^{(k)})^2$
        \\
        $\gamma_2^{(k)}$ & $\kappa_7 \kappa_4 \gamma_1^{(k)} + 3m \kappa_7 \overline{p (2 \sigma_1^2 B^{-1} (1 - B/D) + 5 (1 - p) \delta_1^2)} + m \kappa_7 \kappa_5 \kappa_3 \kappa_2 \max(\eta_i^{(k)})^2$
        \\
        $\omega_2^{(k)}$ & $\begin{aligned}
            \kappa_7 \kappa_4 \omega_1^{(k)} & + 2m \kappa_7 \overline{p ( \sigma_0^2 B^{-1} (1 - B/D) + 5 (1 - p) \delta_0^2)}
            \\
            & + m (\overline{\sigma_0^2 p B^{-1} (1 - B/D)} + 3 \overline{(1 - p) \delta_0^2}) \kappa_7 \kappa_5 \kappa_3 \max(\eta_i^{(k)})^2
        \end{aligned}$
        \\
        $\psi_{01}^{(k)}$ & $\frac12 (1 + 3 \hat{\rho}_{0,B}) \kappa_8 \max(\eta_i^{(k)})$
        \\
        $\psi_{02}^{(k)}$ & $\frac{3 \max(\hat{\phi}_i) \tilde{\rho}_B}{2m} \max(\eta_i^{(k)})$
        \\
        $\gamma_0^{(k)}$ & $\frac12 m \overline{\hat{\phi} \eta^{(k)} \left( m \hat{\pi} - 2 \sigma_1^2 p B^{-1} (1 - B/D) - 3 (1 - p) \delta_1^2 - 3 \hat{\rho}_{0,B} \kappa_2 \right)}$
        \\
        $\omega_0^{(k)}$ & $\frac12 m (1 + 3 \hat{\rho}_{0,B}) \overline{\hat{\phi} \eta^{(k)} \left( \sigma_0^2 p B^{-1} (1 - B/D) + 3 (1 - p) \delta_0^2 \right)}$
        \\
        \bottomrule
    \end{tabular}
    \label{tab:symbols_props}
\end{table}

\subsection{Algorithm} \label{app:supp:alg}

We present the complete steps the \texttt{Spod-GT} algorithm in Alg.~\ref{alg:spodgt}, explicitly demonstrating what each client does when it decides to participate or not to participate in DFL training at a given iteration.

\begin{algorithm}[t]
    \caption{Sporadic Gradient Tracking (\texttt{Spod-GT})}
    \label{alg:spodgt}
    \begin{algorithmic}[1]
        \STATE {\bfseries Input:} $K$, $\mathcal{G} = (\mathcal{M}, \mathcal{E})$, ${\lbrace \mathcal{D}_i \rbrace}_{i \in \mathcal{M}}$, ${\lbrace B_i \rbrace}_{i \in \mathcal{M}}$, ${\lbrace p_i \rbrace}_{i \in \mathcal{M}}$, ${\lbrace \hat{p}_{ij} \rbrace}_{(i,j) \in \mathcal{E}}$, ${\lbrace \eta_i^{(k)} \rbrace}_{\substack{0 \le k \le K - 1 \\ i \in \mathcal{M}}}$
        
        \STATE {\bfseries Output:} ${\lbrace \bs{x}_i^{(K)} \rbrace}_{i \in \mathcal{M}}$
        
        \STATE {\bfseries Initialize:} $k \gets 0$, ${\lbrace \bs{x}_i^{(0)} \in \mathbb{R}^n \rbrace}_{i \in \mathcal{M}}$, ${\lbrace \bs{y}_i^{(0)} \gets \bs{g}_i^{(0)} = \bs{0}_n \rbrace}_{i \in \mathcal{M}}$, $\{ v_i^{(0)} \in \{ 0, 1 \} \}_{i \in \mathcal{M}}$, $\{ \hat{v}_{ij}^{(0)} \in \{ 0, 1 \} \}_{(i, j) \in \mathcal{E}}$

        \FORALL{$i \in \mathcal{M}$}
            \STATE \textbf{if} $v_i^{(0)} = 1$ \textbf{then} $\bs{y}_i^{(0)} \gets \bs{g}_i^{(0)} = \frac1{B_i} \sum_{\xi_{ij}^{(0)} \in \mathcal{B}_i^{(0)}}{\nabla{\mathcal{L}}(\bs{x}_i^{(0)} ; \xi_{ij}^{(0)})}$ where $\mathcal{B}_i^{(0)} \sim \mathcal{D}_i$

            \textbf{for all} $j \in \mathcal{N}_i^{\text{out}}$ \textbf{do}: $a_{ij} \gets 1 / (1 + |\mathcal{N}_i^{\text{in}}|)$, $b_{ji} \gets 1 / (1 + |\mathcal{N}_i^{\text{out}}|)$
        \ENDFOR
        
        \WHILE{$k \le K - 1$}
            \FORALL{$i \in \mathcal{M}$}
                \STATE $\bs{x}_i^{(k+0.5)} \gets \bs{x}_i^{(k)}$, $\bs{y}_i^{(k+0.5)} \gets \bs{y}_i^{(k)}$
                
                \STATE \textbf{for all} $j \in \mathcal{N}_i^{\text{in}}$ \textbf{do}: \textbf{if} $\hat{v}_{ij}^{(k)} = 1$ \textbf{then} $\bs{x}_i^{(k+0.5)} \gets \bs{x}_i^{(k+0.5)} + a_{ij} \left( \bs{x}_j^{(k)} - \bs{x}_i^{(k)} \right)$, $\bs{y}_i^{(k+0.5)} \gets \bs{y}_i^{(k+0.5)} + b_{ij} \bs{y}_j^{(k)}$

                \STATE \textbf{for all} $j \in \mathcal{N}_i^{\text{out}}$ \textbf{do}: \textbf{if} $\hat{v}_{ij}^{(k)} = 1$ \textbf{then} $\bs{y}_i^{(k+0.5)} \gets \bs{y}_i^{(k+0.5)} - b_{ji} \bs{y}_i^{(k)}$
            \ENDFOR
        
            \STATE \textbf{for all} $i \in \mathcal{M}$ \textbf{do}: $\bs{x}_i^{(k+1)} \gets \bs{x}_i^{(k+0.5)} - \eta_i^{(k)} \bs{y}_i^{(k+0.5)}$, $\bs{y}_i^{(k+1)} \gets \bs{y}_i^{(k+0.5)} - \bs{g}_i^{(k)}$

            \FORALL{$i \in \mathcal{M}$}
                \IF{$U_i^{(k+1)} \le p_i$ where $U_i \sim \mathrm{Uniform}(0,1]$}
                    \STATE $v_i^{(k+1)} = 1$

                    \STATE $\bs{g}_i^{(k+1)} = \frac1{B_i} \sum_{\xi_{ij}^{(k+1)} \in \mathcal{B}_i^{(k+1)}}{\nabla{\mathcal{L}}(\bs{x}_i^{(k+1)} ; \xi_{ij}^{(k+1)})}$ where $\mathcal{B}_i^{(k+1)} \sim \mathcal{D}_i$

                    \STATE $\bs{y}_i^{(k+1)} \gets \bs{y}_i^{(k+1)} + \bs{g}_i^{(k+1)}$
                \ENDIF

                \STATE \textbf{for all} $j \in \mathcal{N}_i^{\text{out}}$ \textbf{do}: \textbf{if} $\hat{U}_{ij}^{(k+1)} \le \hat{p}_{ij}$ where $\hat{U}_{ij}^{(k+1)} \sim \mathrm{Uniform}(0,1]$ \textbf{then} $\hat{v}_{ij}^{(k+1)} \gets 1$
            \ENDFOR
            
            \STATE $k \gets k + 1$
       \ENDWHILE
    \end{algorithmic}
\end{algorithm}

\subsection{Proof Tree} \label{app:supp:prooftree}

In Fig.~\ref{fig:prooftree}, we present a bird’s-eye schematic of our theoretical results, illustrating how each individual lemma contributes to the overall argument culminating in our final theorem.

\begin{figure}[t]
    \centering
    \includegraphics[width=0.6\linewidth]{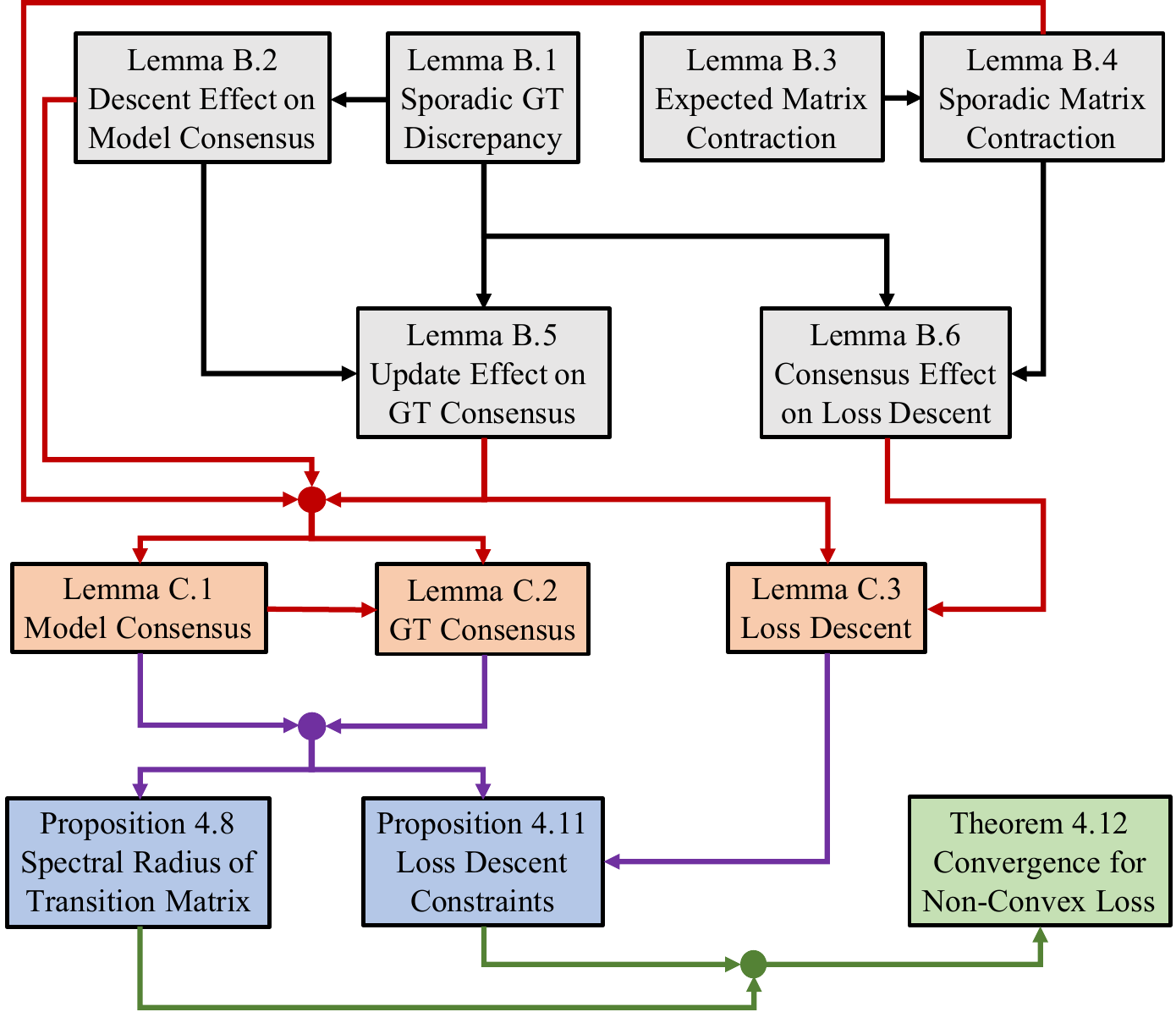}
    \caption{A schematic of the sequence of our theoretical results and their dependence on each other.}
    \label{fig:prooftree}
\end{figure}

\section{Intermediary Lemmas} \label{ssec:intermed}

In Definition~\eqref{def:avg_vec} we see that while we took the regular average of $\bs{Y}^{(k)}$, we needed to take the $\bs{\hat{\phi}}$-weighted average of $\bs{X}^{(k)}$ to arrive at the average iterative updates in Eq.~\eqref{eqn:averages}. To analyze the convergence of our methodology under these expressions, we need to obtain an upper bound on $\| \sum_{j=1}^m{b_{ij} \bs{y}_j^{(k)}} - \nabla{F}(\bar{\bs{x}}_{\bs{\hat{\phi}}}^{(k)}) \|^2$, which is the difference between the aggregated gradient tracking at client $i \in \mathcal{M}$ and the global gradient in any iteration $k \geq 0$. In order to obtain this bound using Assumptions~\ref{assump:lipschitz}-\ref{assump:graph}, we need to break it down to multiple steps, i.e.,
\begin{equation} \label{eqn:grad_steps}
    \sum_{j=1}^m{b_{ij} \bs{y}_j^{(k)}} \overset{(a)}{\to} \bs{\bar{y}}^{(k)} \overset{(b)}{\to} \bs{\bar{g}}_{\bs{v}}^{(k)} \overset{(c)}{\to} \overline{\nabla{\bs{F}}(\bs{X}^{(k)})}_{\bs{v}^{(k)}} \overset{(d)}{\to} \overline{\nabla{\bs{F}}(\bs{1} \bs{\bar{x}}_{\bs{\hat{\phi}}}^{(k)})}_{\bs{v}^{(k)}} \overset{(e)}{\to} \nabla{F}(\bs{\bar{x}}_{\bs{\hat{\phi}}}^{(k)}).
\end{equation}
We will postpone discussing $(a)$ to Lemma~\ref{lemma:gradconsensus}, and the equality for $(b)$ was obtained in Eq.~\eqref{eqn:ybar}.
\begin{lemma}[Sporadic Gradient Tracking Discrepancy] \label{lemma:graderror}
    Let $\bs{\hat{\phi}}$ be a positive stochastic vector. The following bounds hold for the steps $(c)$-$(e)$ of Eq.~\eqref{eqn:grad_steps}:
    \begin{enumerate}[label=(\roman*)]
        \item \label{lemma:graderror:sgd} Let Assumptions~\ref{assump:lipschitz}, \ref{assump:sgd} and \ref{assump:bernoulli} hold. We have
        \begin{equation}
            \begin{aligned}
                \mathbb{E}\left[ \left\| \bar{\bs{g}}_v^{(k)} - \overline{\nabla{\bs{F}}(\bs{X}^{(k)})}_{\bs{v}^{(k)}} \right\|^2 \right] & \le \overline{\sigma_0^2 p B^{-1} (1 - B_i / D_i)}
                \\
                & + \frac{2\bar{L}^2}m \max\left( \frac{\sigma_{1,i}^2 p_i (1 - B_i / D_i)}{B_i \hat{\phi}_i} \right) \mathbb{E}\left[ \left\| \bs{X}^{(k)} - \bs{1} \bs{\bar{x}}_{\bs{\hat{\phi}}}^{(k)} \right\|_{\bs{\hat{\phi}}}^2 \right]
                \\
                & + 2 \overline{\sigma_1^2 p B^{-1} (1 - B/D)} \mathbb{E}\left[ \left\| \nabla{F}(\bs{\bar{x}}_{\bs{\hat{\phi}}}^{(k)}) \right\|^2 \right].
            \end{aligned}
        \end{equation}
        
        \item \label{lemma:graderror:lipschitz} Let Assumptions~\ref{assump:lipschitz}, \ref{assump:graddiv} and \ref{assump:bernoulli} hold. We have
        \begin{equation}
            \mathbb{E}\left[ \left\| \overline{\nabla{\bs{F}}(\bs{X}^{(k)})}_{\bs{v}^{(k)}} - \overline{\nabla{\bs{F}}(\bs{1} \bs{\bar{x}}_{\bs{\hat{\phi}}}^{(k)})}_{\bs{v}^{(k)}} \right\|^2 \right] \le \frac1m \max\left( \frac{p_i}{\hat{\phi}_i} L_i^2 \right) \mathbb{E}\left[ \left\| \bs{X}^{(k)} - \bs{1} \bs{\bar{x}}_{\bs{\hat{\phi}}}^{(k)} \right\|_{\bs{\hat{\phi}}}^2 \right].
        \end{equation}

        \item \label{lemma:graderror:bernoulli} Let Assumptions~\ref{assump:graddiv} and \ref{assump:bernoulli} hold. We have
        \begin{equation}
            \mathbb{E}\left[ \left\| \overline{\nabla{\bs{F}}(\bs{1} \bs{\bar{x}}_{\bs{\hat{\phi}}}^{(k)})}_{\bs{v}^{(k)}} - \nabla{F}(\bar{\bs{x}}_{\bs{\hat{\phi}}}^{(k)}) \right\|^2 \right] \le \overline{(1 - p) \delta_0^2} + \overline{(1 - p) \delta_1^2} \mathbb{E}\left[ \left\| \nabla{F}(\bar{\bs{x}}_{\bs{\hat{\phi}}}^{(k)}) \right\|^2 \right].
        \end{equation}
    \end{enumerate}

    \begin{proof}
        See Appendix~\ref{app:lemma:graderror}
    \end{proof}
\end{lemma}
We observe that when clients conduct GT in every training iteration, i.e., $p_i = 1$ for all $i \in \mathcal{M}$, the upper bound in Lemma~\ref{lemma:graderror}-\ref{lemma:graderror:bernoulli} becomes zero. On the other hand, if the gradient estimation variance is uniformly bounded, i.e., $\sigma_{1,i}^2 = 0$, we get a uniform bound for Lemma~\ref{lemma:graderror}-\ref{lemma:graderror:sgd}. Also, note that when clients use a full batch, i.e., $B_i = D_i$, the bound in Lemma~\ref{lemma:graderror}-\ref{lemma:graderror:sgd} becomes zero.

Next, we obtain several key upper bounds involving the gradient tracking parameters $\bs{y}_i^{(k)}$.
\begin{lemma}[Descent Effect on Model Consensus] \label{lemma:descenteffect}
    Let $\bs{\hat{\phi}}$ and $\bs{\hat{\pi}}$ be positive stochastic vectors. Then, the following bounds hold:
    
    \begin{enumerate}[label=(\roman*)]
        \item \label{lemma:descenteffect:ybar} Let Assumptions~\ref{assump:lipschitz}-\ref{assump:bernoulli} hold. We will then have according to Lemma~\ref{lemma:graderror} that
        \begin{equation}
            \mathbb{E}\left[ \| \bs{\bar{y}}^{(k)} \|^2 \right] \le \overline{\sigma_0^2 p B^{-1} (1 - B/D)} + 3 \overline{(1 - p) \delta_0^2} + \frac{\kappa_1}m \mathbb{E}\left[ \left\| \bs{X}^{(k)} - \bs{1} \bs{\bar{x}}_{\bs{\hat{\phi}}}^{(k)} \right\|_{\bs{\hat{\phi}}}^2 \right] + \kappa_2 \mathbb{E}\left[ \left\| \nabla{F}(\bs{\bar{x}}_{\bs{\hat{\phi}}}^{(k)}) \right\|^2 \right],
        \end{equation}
        where $\kappa_1 = \max( \frac{p_i}{\hat{\phi}_i} (2 \frac{(1 - B_i / D_i) \sigma_{1,i}^2}{B_i} \bar{L}^2 + 3 L_i^2))$ and $\kappa_2 = 2 \overline{\sigma_1^2 p B^{-1} (1 - B/D)} + 3 \overline{(1 - p) \delta_1^2} + 3$.

        \item \label{lemma:descenteffect:ypiinv} The $\bs{\hat{\pi}}^{-1}$-weighted norm of gradient tracking parameters $\bs{y}_i^{(k)}$ can be upper bounded as 
        \begin{equation}
            \left\| \bs{Y}^{(k)} \right\|_{\bs{\hat{\pi}}^{-1}}^2 \le \left\| \bs{\Lambda}_{\bs{\hat{\pi}}}^{-1} \bs{Y}^{(k)} - m \bs{1} \bs{\bar{y}}^{(k)} \right\|_{\bs{\hat{\pi}}}^2 + m^2 \left\| \bs{\bar{y}}^{(k)} \right\|^2.
        \end{equation}

        \item \label{lemma:descenteffect:yconsphi} Let Assumption~\ref{assump:bernoulli} hold. The descent step deviation from its $\bs{\hat{\phi}}$-weighted average is bounded as
        \begin{equation}
            \mathbb{E}\left[ \left\| \bs{\Lambda}_{\bs{\eta}}^{(k)} \bs{\hat{B}}^{(k)} \bs{Y}^{(k)} - \bs{1} \bs{\hat{\phi}}^T \bs{\Lambda}_{\bs{\eta}}^{(k)} \bs{\hat{B}}^{(k)} \bs{Y}^{(k)} \right\|_{\bs{\hat{\phi}}}^2 \right] \le \kappa_3 \max(\eta_i^{(k)})^2 \mathbb{E}\left[ \left\| \bs{Y}^{(k)} \right\|_{\bs{\hat{\pi}}^{-1}}^2 \right],
        \end{equation}
        where $\kappa_3 = \max(\hat{\phi}_i) ( \frac{\max(\hat{\pi}_i)^2}{\min(\hat{\pi}_i)} + 2m \max(\hat{\pi}_j \overline{b_{:j}^2 \hat{p}_{:j} (1 - \hat{p}_{:j})}) )$.
    \end{enumerate}

    \begin{proof}
        See Appendix~\ref{app:lemma:descenteffect}
    \end{proof}
\end{lemma}


    

In Eq.~\eqref{eqn:matrix_update} where we obtained the matrix update equations for \texttt{Spod-GT}, we see that the aggregations are governed by the row-stochastic matrix $\bs{\hat{A}}^{(k)}$ and column-stochastic matrix $\bs{\hat{B}}^{(k)}$. Furthermore, in Eq.~\eqref{eqn:averages}, we obtained the weighted average of model parameters $\bs{\bar{x}}_{\bs{\hat{\phi}}}^{(k)}$ and the average of GT parameters $\bs{\hat{y}}^{(k)}$. In the following Lemma, we first characterize error upper bounds on the expected mixing matrices $\bs{\hat{A}}$ and $\bs{\hat{B}}$, which were defined in Eq.~\eqref{eqn:expected_mat}.
\begin{lemma}[Expected Matrix Contraction] \label{lemma:expected_contraction}
    Let Assumption~\ref{assump:graph} hold. Let the row-stochastic $\bs{\hat{A}}$ and column-stochastic $\bs{\hat{B}}$ expected mixing matrices (Eq.~\eqref{eqn:expected_mat}) be designed to be irreducible and primitive such that $\bs{\hat{\phi}}$ is the left-eigenvector of $\bs{\hat{A}}$ and $\bs{\hat{\pi}}$ is the right-eigenvector of $\bs{\hat{B}}$ (Eq.~\eqref{eqn:eigenvectors}). Then, the following bounds will hold for the vector-induced matrix norms:
    \begin{enumerate}[label=(\roman*)]
        \item \label{lemma:expected_contraction:A} $\left\| \bs{\hat{A}} \bs{X}^{(k)} - \bs{1} \bs{\bar{x}}_{\bs{\hat{\phi}}}^{(k)} \right\|_{\bs{\hat{\phi}}}^2 \le \hat{\rho}_A \left\| \bs{X}^{(k)} - \bs{1} \bs{\bar{x}}_{\bs{\hat{\phi}}}^{(k)} \right\|_{\bs{\hat{\phi}}}^2$,

        \item \label{lemma:expected_contraction:B} $\left\| \bs{\Lambda}_{\bs{\hat{\pi}}}^{-1} \bs{\hat{B}} \bs{Y}^{(k)} - m \bs{1} \bs{\bar{y}}^{(k)} \right\|_{\bs{\hat{\pi}}}^2 \le \hat{\rho}_B \left\| \bs{\Lambda}_{\bs{\hat{\pi}}}^{-1} \bs{Y}^{(k)} - m \bs{1} \bs{\bar{y}}^{(k)} \right\|_{\bs{\hat{\pi}}}^2$.
    \end{enumerate}
    where $\bs{X}^{(k)}, \bs{Y}^{(k)} \in \mathbb{R}^{m \times n}$ are a collection of $m$ arbitrary vectors and $0 \le \hat{\rho}_A, \hat{\rho}_B < 1$. In particular, we have $\hat{\rho}_A = 1 - \frac{\min(\hat{\phi}_i) \min(a_{ij}^+ \hat{p}_{ij})^2}{\max(\hat{\phi}_i)^2 D(\mathcal{G}) K(\mathcal{G})}$ and $\hat{\rho}_B = 1 - \frac{\min(\hat{\pi}_i)^2 \min(b_{ij}^+ \hat{p}_{ij})^2}{\max(\hat{\pi}_i)^3 D(\mathcal{G}) K(\mathcal{G})}$, in which $D(\mathcal{G})$ is the diameter of graph $\mathcal{G}$ and $K(\mathcal{G}$ is its maximal edge utility
    \begin{proof}
        See Lemmas~4.2 and 4.5 in \cite{nedic2025ab} applied to expected mixing matrices in Eq.~\eqref{eqn:expected_mat}, noting that $\hat{a}_{ij} = a_{ij} \hat{p}_{ij}$ and $\hat{b}_{ij} = b_{ij} \hat{p}_{ij}$.
    \end{proof}
\end{lemma}
Note that in the simpler case where the directed graph $\mathcal{G}$ admits a doubly-stochastic mixing matrix \cite{gharesifard2010does, xin2021improved}, we would have $\bs{\hat{\phi}} = \bs{\hat{\pi}} = \frac1m \bs{1}$, and $\hat{\rho}_A$ and $\hat{\rho}_B$ in Lemma~\ref{lemma:expected_contraction} would correspond to the spectral radius of the matrices $\bs{\hat{A}} - \frac1m \bs{1} \bs{1}^T$ and $\bs{\hat{B}} - \frac1m \bs{1} \bs{1}^T$, respectively. Furthermore, note that if $\mathcal{G}$ is a fully-connected graph, which results in $D(\mathcal{G}) = K(\mathcal{G}) = 1$, and we assign the mixing weights to $a_{ij} = b_{ij} = 1 / m$ for all $(i,j) \in \mathcal{E}$, the spectral radii simplify to $\hat{\rho}_A = \hat{\rho}_B = 1 - \min(\hat{p}_{ij})^2$. Notice how assigning $\hat{p}_{ij} = 1$ for all the edges further gives us $\hat{\rho}_A = \hat{\rho}_B = 0$.

Having established a contraction property for the expected mixing matrices in Lemma~\ref{lemma:expected_contraction}, we now extend it to sporadic mixing matrices $\bs{\hat{A}}^{(k)}$ and $\bs{\hat{B}}^{(k)}$.
\begin{lemma}[Sporadic Matrix Contraction] \label{lemma:sporadic_contraction}
    Let Assumptions~\ref{assump:bernoulli} and \ref{assump:graph} hold. Let the row-stochastic $\bs{\hat{A}}^{(k)}$ and column-stochastic $\bs{\hat{B}}^{(k)}$ sporadic mixing matrices (Eq.~\eqref{eqn:matrix_update}) be designed such that their respective expected values $\bs{\hat{A}}$ and $\bs{\hat{B}}$ (Eq.~\eqref{eqn:expected_mat}) are irreducible and primitive. Then, Lemmas~\ref{lemma:descenteffect} and \ref{lemma:expected_contraction} imply the following bounds for the vector-induced sporadic matrix norms:
    \begin{enumerate}[label=(\roman*)]
        \item \label{lemma:sporadic_contraction:A} If the communication probabilities satisfy $\min(\hat{p}_{ij}) \in [\hat{r}_A, 1]$ where $\hat{r}_A \in (0, 1)$ is the solution of the equation $f_A(x) = x^{2(m+1)} + \hat{\tau}_A x - \hat{\tau}_A$ with $\hat{\tau}_A = 16 m^2 (m - 1) D(\mathcal{G}) K(\mathcal{G}) \frac{\max(a_{ij}^+)^2}{\min(a_{ij}^+)^{2(m+1)}}$, then
        \begin{equation}
            \begin{aligned}
                \mathbb{E}\left[ \left\| \bs{\hat{A}}^{(k)} \bs{X}^{(k)} - \bs{1} \bs{\bar{x}}_{\bs{\hat{\phi}}}^{(k)} \right\|_{\bs{\hat{\phi}}}^2 \right] \le (\hat{\rho}_A + \hat{\rho}_{0,A}) \mathbb{E}\left[ \left\| \bs{X}^{(k)} - \bs{1} \bs{\bar{x}}_{\bs{\hat{\phi}}}^{(k)} \right\|_{\bs{\hat{\phi}}}^2 \right],
            \end{aligned}
        \end{equation}
        where $\hat{\rho}_{0,A} = 4 (m - 1) \frac{\max(\hat{\phi}_i)}{\min(\hat{\phi}_i)} \max( a_{ij}^2 \hat{p}_{ij} (1 - \hat{p}_{ij})) \le \frac{1 - \hat{\rho}_A}4$.

        \item \label{lemma:sporadic_contraction:B} If the communication probabilities satisfy $\min(\hat{p}_{ij}) \in [\hat{r}_B, 1]$ where $\hat{r}_B \in (0, 1)$ is the solution of the equation $f_B(x) = x^{3m+2} + \hat{\tau}_B x - \hat{\tau}_B$ with $\hat{\tau}_B = 4 m^3 (m - 1) D(\mathcal{G}) K(\mathcal{G}) \frac{\max(b_{ij}^+)^2}{\min(b_{ij}^+)^{3m+2}}$, then
        \begin{equation}
            \mathbb{E}\left[ \left\| \bs{\Lambda}_{\bs{\hat{\pi}}}^{-1} \bs{\hat{B}}^{(k)} \bs{Y}^{(k)} - m \bs{1} \bs{\bar{y}}^{(k)} \right\|_{\bs{\hat{\pi}}}^2 \right] \le (\hat{\rho}_B + \hat{\rho}_{0,B}) \mathbb{E}\left[ \left\| \bs{\Lambda}_{\bs{\hat{\pi}}}^{-1} \bs{Y}^{(k)} - m \bs{1} \bs{\bar{y}}^{(k)} \right\|_{\bs{\hat{\pi}}}^2 \right] + m^2 \hat{\rho}_{0,B} \mathbb{E}\left[ \| \bs{\bar{y}}^{(k)} \|^2 \right],
        \end{equation}
        where $\hat{\rho}_{0,B} = 2 (m - 1) \frac{\max(\hat{\pi}_i)}{\min(\hat{\pi}_i)} \max( b_{ij}^2 \hat{p}_{ij} (1 - \hat{p}_{ij}) ) \le \frac{1 - \hat{\rho}_B}2$.
    \end{enumerate}
    Note that $\bs{X}^{(k)}, \bs{Y}^{(k)} \in \mathbb{R}^{m \times n}$ are a collection of $m$ arbitrary $n$-dimensional vectors.
    \begin{proof}
        See Appendix~\ref{app:lemma:sporadic_contraction}.
    \end{proof}
\end{lemma}
Lemma~\ref{lemma:sporadic_contraction} is one of the main contributions of our paper, as we show that the row-stochastic $\bs{\hat{A}}^{(k)}$ and column-stochastic $\bs{\hat{B}}^{(k)}$ transition matrices have a contraction property, despite them being sporadic. In other words, $\bs{\hat{A}}^{(k)}$ and $\bs{\hat{B}}^{(k)}$ contract in expectation despite not corresponding to a strongly-connected digraph. Moreover, we observe that when clients participate in peer-to-peer communications in every training iteration, i.e., $\hat{p}_{ij} = 1$ for all $(i, j) \in \mathcal{E}$, we will have both $\hat{\rho}_{0,A} = \hat{\rho}_{0,B} = 0$, and Lemma~\ref{lemma:sporadic_contraction} reduces back to Lemma~\ref{lemma:expected_contraction}.

We see in Eqs.~\eqref{eqn:vector_update} and \eqref{eqn:matrix_update} that each client $i \in \mathcal{M}$ updates its gradient tracking parameters $\bs{y}_i^{(k+1)}$ by subtracting the old local stochastic gradient from it and adding the new local stochastic gradient, i.e., $\bs{g}_i^{(k+1)} v_i^{(k+1)} - \bs{g}_i^{(k)} v_i^{(k)}$. Therefore, we need to characterize how much this update affects the gradient tracking consensus.

\begin{lemma}[Update Effect on Gradient Tracking Consensus] \label{lemma:trackingeffect}
    Let $\bs{\hat{\phi}}$ and $\bs{\hat{\pi}}$ be positive stochastic vectors. Then, the following bounds hold:
    \begin{enumerate}[label=(\roman*)]
        \item \label{lemma:trackingeffect:sporadic_modeldiv} Let Assumption~\ref{assump:bernoulli} hold. Then, an upper bound on the sporadic consensus variance can be obtained as
        \begin{equation}
            \mathbb{E}\left[ \left\| \left( \bs{\hat{A}}^{(k)} - \bs{\hat{A}} \right) \bs{X}^{(k)} \right\|_{\bs{\hat{\phi}}}^2 \right] \le \hat{\rho}_{0,A} \mathbb{E}\left[ \left\| \bs{X}^{(k)} - \bs{1} \bs{\bar{x}}_{\bs{\hat{\phi}}}^{(k)} \right\|_{\bs{\hat{\phi}}}^2 \right].
        \end{equation}
    
        \item \label{lemma:trackingeffect:modelonestep} Let Assumption~\ref{assump:bernoulli} hold. Then, an upper bound on the one-step average model update can be obtained as
        \begin{equation}
            \mathbb{E}\left[ \left\| \bs{\bar{x}}_{\bs{\hat{\phi}}}^{(k+1)} - \bs{\bar{x}}_{\bs{\hat{\phi}}}^{(k)} \right\|^2 \right] \le \hat{\rho}_{0,A} \mathbb{E}\left[ \left\| \bs{X}^{(k)} - \bs{1} \bs{\bar{x}}_{\bs{\hat{\phi}}}^{(k)} \right\|_{\bs{\hat{\phi}}}^2 \right] + \kappa_3 \max(\eta_i^{(k)})^2 \mathbb{E}\left[ \left\| \bs{Y}^{(k)} \right\|_{\bs{\hat{\pi}}^{-1}}^2 \right].
        \end{equation}
        
        \item \label{lemma:trackingeffect:sporadic_sgddiv} Let Assumptions~\ref{assump:lipschitz}-\ref{assump:bernoulli} hold. Then, an upper bound on the sporadic stochastic gradient diversity can be obtained as
        \begin{equation}
            \begin{aligned}
                \mathbb{E} & \left[ \left\| \bs{\Lambda}_{\bs{v}}^{(k+1)} \bs{G}^{(k+1)} - \bs{\Lambda}_{\bs{v}}^{(k)} \bs{G}^{(k)} \right\|^2 \right] \le 2 m \overline{p \left( \sigma_0^2 B^{-1} (1 - B/D) + 5 (1 - p) \delta_0^2 \right)}
                \\
                & + \kappa_4 \left( \mathbb{E}\left[ \left\| \bs{X}^{(k+1)} - \bs{1} \bs{\bar{x}}_{\bs{\hat{\phi}}}^{(k+1)} \right\|_{\bs{\hat{\phi}}}^2 \right] + \mathbb{E}\left[ \left\| \bs{X}^{(k)} - \bs{1} \bs{\bar{x}}_{\bs{\hat{\phi}}}^{(k)} \right\|_{\bs{\hat{\phi}}}^2 \right] \right)
                \\
                & + 3m \overline{p \left( 2 \sigma_1^2 B^{-1} (1 - B/D) + 5 (1 - p) \delta_1^2 \right)} \mathbb{E}\left[ \left\| \nabla{F}(\bs{\bar{x}}_{\bs{\hat{\phi}}}^{(k)}) \right\|^2 \right] + m \kappa_5 \mathbb{E}\left[ \left\| \bs{\bar{x}}_{\bs{\hat{\phi}}}^{(k+1)} - \bs{\bar{x}}_{\bs{\hat{\phi}}}^{(k)} \right\|^2 \right],
        	\end{aligned}
        \end{equation}
        where $\kappa_4 = \max\left( \frac{p_i}{\hat{\phi}_i} \left( 2 \frac{(1 - B_i / D_i) \sigma_{1,i}^2}{B_i} \bar{L}^2 + 5 L_i^2 \right) \right)$ and $\kappa_5 = 2 \overline{p \left( 2 \sigma_1^2 B^{-1} (1 - B/D) + 5 (1 - p) \delta_1^2 \right)} \bar{L}^2 + 5 \overline{p^2 L^2}$.
    \end{enumerate}
    
    \begin{proof}
        See Appendix~\ref{app:lemma:trackingeffect}.
    \end{proof}
\end{lemma}
We observe that when clients participate in peer-to-peer communications in every training iteration, i.e., $\hat{p}_{ij} = 1$ for all $(i,j) \in \mathcal{E}$ resulting in $\hat{\rho}_{0,A} = 0$, the upper bound of Lemma~\ref{lemma:trackingeffect}-\ref{lemma:trackingeffect:sporadic_modeldiv} becomes zero, and the upper bound of Lemma~\ref{lemma:trackingeffect}-\ref{lemma:trackingeffect:modelonestep} only involves the term $\kappa_3 \max(\eta_i^{(k)})^2 \mathbb{E}[ \| \bs{Y}^{(k)} \|_{\bs{\hat{\pi}}^{-1}}^2 ]$. On the other hand, if clients conduct GT in every training iteration, i.e., $p_i = 1$ for all $i \in \mathcal{M}$, and also the gradient estimation variance is uniformly bounded, i.e., $\sigma_{1,i}^2 = 0$, the bound of Lemma~\ref{lemma:trackingeffect}-\ref{lemma:trackingeffect:sporadic_sgddiv} will no longer depend on $\mathbb{E}[ \| \nabla{F}(\bs{\bar{x}}_{\bs{\hat{\phi}}}^{(k)}) \|^2 ]$.

Equipped with the previous Lemmas, we can finally obtain the desired upper bound for Eq.~\eqref{eqn:grad_steps}.
\begin{lemma}[Consensus Effect on Loss Descent] \label{lemma:gradconsensus}
    Let Assumptions~\ref{assump:lipschitz}-\ref{assump:graph} hold, and $\bs{\hat{\phi}}$ and $\bs{\hat{\pi}}$ be positive stochastic vectors. Then, the following bounds hold:
    \begin{equation}
        \begin{aligned}
            \mathbb{E} & \left[ \left\| \bs{\Lambda}_{\bs{\hat{\pi}}}^{-1} \bs{\hat{B}}^{(k)} \bs{Y}^{(k)} - \bs{1} \nabla{F}(\bs{\bar{x}}_{\bs{\hat{\phi}}}^{(k)}) \right\|_{\bs{\hat{\pi}}}^2 \right] \le 3 (\hat{\rho}_B + \hat{\rho}_{0,B}) \mathbb{E}\left[ \left\| \bs{\Lambda}_{\bs{\hat{\pi}}}^{-1} \bs{Y}^{(k)} - m \bs{1} \bs{\bar{y}}^{(k)} \right\|_{\bs{\hat{\pi}}}^2 \right]
            \\
            & + m^2 (1 + 3 \hat{\rho}_{0,B}) \left( \overline{\sigma_0^2 p B^{-1} (1 - B/D)} + 3 \overline{(1 - p) \delta_0^2} \right) + m (1 + 3 \hat{\rho}_{0,B}) \kappa_1 \mathbb{E}\left[ \left\| \bs{X}^{(k)} - \bs{1} \bs{\bar{x}}_{\bs{\hat{\phi}}}^{(k)} \right\|_{\bs{\hat{\phi}}}^2 \right]
            \\
            & + m^2 \left( 2 \overline{\sigma_1^2 p B^{-1} (1 - B/D)} + 3 \overline{(1 - p) \delta_1^2} + 3 \hat{\rho}_{0,B} \kappa_2 \right) \mathbb{E}\left[ \left\| \nabla{F}(\bs{\bar{x}}_{\bs{\hat{\phi}}}^{(k)}) \right\|^2 \right].
        \end{aligned}
    \end{equation}
    
    \begin{proof}
        See Appendix~\ref{app:lemma:gradconsensus}.
    \end{proof}
\end{lemma}
We observe that when clients conduct GT in every training iteration, i.e., $p_i = 1$ for all $i \in \mathcal{M}$, clients participate in peer-to-peer communications in every training iteration, i.e., $\hat{p}_{ij} = 1$ for all $(i,j) \in \mathcal{E}$ resulting in $\hat{\rho}_{0,B} = 0$, and also the gradient estimation variance is uniformly bounded, i.e., $\sigma_{1,i}^2 = 0$, the bound of Lemma~\ref{lemma:gradconsensus} will no longer depend on $\mathbb{E}[ \| \nabla{F}(\bs{\bar{x}}_{\bs{\hat{\phi}}}^{(k)}) \|^2 ]$.

\section{Main Lemmas} \label{ssec:main_lemmas}

Our first main lemma is a consensus upper bound of model parameters $\bs{x}_i^{(k)}$ for each client compared to the weighted average of all model parameters in the directed network graph $\bs{\bar{x}}_{\bs{\phi}}^{(k)}$.
\begin{lemma}[Model Consensus] \label{lemma:xdispersion}
    Let Assumptions~\ref{assump:lipschitz}-\ref{assump:graph} hold. Then, Lemmas~\ref{lemma:expected_contraction}-\ref{lemma:expected_contraction:A} and \ref{lemma:descenteffect} imply that
    \begin{equation}
        \begin{aligned}
            \mathbb{E}\left[ \left\| \bs{X}^{(k+1)} - \bs{1} \bs{\bar{x}}_{\bs{\hat{\phi}}}^{(k+1)} \right\|_{\bs{\hat{\phi}}}^2 \right] \le \psi_{11}^{(k)} \mathbb{E}\left[ \left\| \bs{X}^{(k)} - \bs{1} \bs{\bar{x}}_{\bs{\hat{\phi}}}^{(k)} \right\|_{\bs{\hat{\phi}}}^2 \right] & + \psi_{12}^{(k)} \mathbb{E}\left[ \left\| \bs{\Lambda}_{\bs{\hat{\pi}}}^{-1} \bs{Y}^{(k)} - m \bs{1} \bs{\bar{y}}^{(k)} \right\|_{\bs{\hat{\pi}}}^2 \right]
            \\
            & + \gamma_1^{(k)} \mathbb{E}\left[ \left\| \nabla{F}(\bs{\bar{x}}_{\bs{\hat{\phi}}}^{(k)}) \right\|^2 \right] + \omega_1^{(k)},
        \end{aligned}
    \end{equation}
    where $\psi_{11}^{(k)} = \frac{1 + \tilde{\rho}_A}2 + m \kappa_1 \kappa_6 \kappa_3 \max(\eta_i^{(k)})^2$, $\psi_{12}^{(k)} = \kappa_6 \kappa_3 \max(\eta_i^{(k)})^2$, $\gamma_1^{(k)} = m^2 \kappa_2 \kappa_6 \kappa_3 \max(\eta_i^{(k)})^2$ and $\omega_1^{(k)} = m^2 (\overline{\sigma_0^2 p B^{-1} (1 - B/D)} + 3 \overline{(1 - p) \delta_0^2}) \kappa_6 \kappa_3 \max(\eta_i^{(k)})^2$, in which $\kappa_6 = \frac{1 + \hat{\rho}_A}{1 - \tilde{\rho}_A}$ and $\tilde{\rho}_A = \hat{\rho}_A + 2 \hat{\rho}_{0,A}$.
    
    \begin{proof}
        See Appendix~\ref{app:lemma:xdispersion}.
    \end{proof}
\end{lemma}
We observe that when clients conduct GT in every training iteration, i.e., $p_i = 1$ for all $i \in \mathcal{M}$, and use a full batch, i.e., $B_i = D_i$, we will have $\omega_1^{(k)} = 0$.

Our next main lemma is a consensus upper bound of scaled GT parameters $\frac1{\hat{\pi}_i} \bs{y}_i^{(k)}$ for each client compared to the sum of all GT parameters in the directed network graph $m \bs{\bar{y}}^{(k)}$.
\begin{lemma}[Gradient Tracking Consensus] \label{lemma:ydispersion}
    Let Assumptions~\ref{assump:lipschitz}-\ref{assump:graph} hold. Then, Lemmas~\ref{lemma:expected_contraction}-\ref{lemma:expected_contraction:B} and \ref{lemma:trackingeffect}
    imply that
    \begin{equation}
        \begin{aligned}
            \mathbb{E}\left[ \left\| \bs{\Lambda}_{\bs{\hat{\pi}}}^{-1} \bs{Y}^{(k+1)} - m \bs{1} \bs{\bar{y}}^{(k+1)} \right\|_{\bs{\hat{\pi}}}^2 \right] \le \psi_{21}^{(k)} \mathbb{E}\left[ \left\| \bs{X}^{(k)} - \bs{1} \bs{\bar{x}}_{\bs{\hat{\phi}}}^{(k)} \right\|_{\bs{\hat{\phi}}}^2 \right] & + \psi_{22}^{(k)} \mathbb{E}\left[ \left\| \bs{\Lambda}_{\bs{\hat{\pi}}}^{-1} \bs{Y}^{(k)} - m \bs{1} \bs{\bar{y}}^{(k)} \right\|_{\bs{\hat{\pi}}}^2 \right]
            \\
            & + \gamma_2^{(k)} \mathbb{E}\left[ \left\| \nabla{F}(\bs{\bar{x}}_{\bs{\hat{\phi}}}^{(k)}) \right\|^2 \right] + \omega_2^{(k)},
        \end{aligned}
    \end{equation}
    where $\psi_{21}^{(k)} = \kappa_7 \kappa_4 (1 + \psi_{11}^{(k)}) + \kappa_7 \kappa_5 \kappa_3 \kappa_1 \max(\eta_i^{(k)})^2 + m \kappa_7 \kappa_5 \hat{\rho}_{0,A}$, $\psi_{22}^{(k)} = \frac{1 + \tilde{\rho}_B}2 + \kappa_7 \kappa_4 \psi_{12}^{(k)} + m \kappa_7 \kappa_5 \kappa_3 \max(\eta_i^{(k)})^2$, $\gamma_2^{(k)} = \kappa_7 \kappa_4 \gamma_1^{(k)} + 3m \kappa_7 \overline{p (2 \sigma_1^2 B^{-1} (1 - B/D) + 5 (1 - p) \delta_1^2)} + m \kappa_7 \kappa_5 \kappa_3 \kappa_2 \max(\eta_i^{(k)})^2$ and $\omega_2^{(k)} = \kappa_7 \kappa_4 \omega_1^{(k)} + 2m \kappa_7 \overline{p ( \sigma_0^2 B^{-1} (1 - B/D) + 5 (1 - p) \delta_0^2)} + m (\overline{\sigma_0^2 p B^{-1} (1 - B/D)} + 3 \overline{(1 - p) \delta_0^2}) \kappa_7 \kappa_5 \kappa_3 \max(\eta_i^{(k)})^2$, in which $\kappa_7 = 2 (m + 1) \frac{1 + \tilde{\rho}_B}{1 - \tilde{\rho}_B}$ and $\tilde{\rho}_B = \hat{\rho}_B + \hat{\rho}_{0,B}$.

    \begin{proof}
        See Appendix~\ref{app:lemma:ydispersion}.
    \end{proof}
\end{lemma}
We observe that if clients perform GT in every training iteration, i.e., $p_i = 1$ for all $i \in \mathcal{M}$, and the gradient estimation variance is uniformly bounded, i.e., $\sigma_{1,i}^2 = 0$, then {\small $\gamma_2^{(k)} \propto \max(\eta_i^{(k)})^2$}. On the other hand, when clients perform GT in every training iteration, i.e., $p_i = 1$ for all $i \in \mathcal{M}$, and clients use a full batch, i.e., $B_i = D_i$, we will have {\small $\omega_2^{(k)} = 0$}.

Recursively expanding Eq.~\eqref{eqn:varsigma}, we obtain the following explicit inequality for the error vector $\bs{\varsigma}^{(k)}$
\begin{equation} \label{eqn:varsigma_explicit}
    \bs{\varsigma}^{(k+1)} \le \bs{\Psi}^{(k:0)} \bs{\varsigma}^{(0)} + \sum_{r=0}^k{\bs{\Psi}^{(k-1:r)} \bs{\gamma}^{(r)} \mathbb{E}\left[ \left\| \nabla{F}(\bs{\bar{x}}_{\bs{\hat{\phi}}}^{(r)}) \right\|^2 \right]} + \sum_{r=0}^k{\bs{\Psi}^{(k-1:r)} \bs{\omega}^{(r)}},
\end{equation}
in which $\bs{\Psi}^{(k:r)} = \bs{\Psi}^{(k)} \bs{\Psi}^{(k-1)} ... \bs{\Psi}^{(r)}$ for $k > r$, $\bs{\Psi}^{(k:k)} = \bs{\Psi}^{(k)}$ for $k = r$ and $\bs{\Psi}^{(k:r)} = \bs{I}$ for $k < r$. 

Finally, we provide an upper bound on the global loss value at each training iteration.
\begin{lemma}[Loss Descent] \label{lemma:loss}
    Let Assumptions~\ref{assump:lipschitz}-\ref{assump:graph} hold. Then, Lemmas~\ref{lemma:trackingeffect}-\ref{lemma:trackingeffect:modelonestep} and \ref{lemma:gradconsensus} imply that
    \begin{equation}
        \begin{aligned}
            \mathbb{E}\left[ F(\bs{\bar{x}}_{\bs{\hat{\phi}}}^{(k+1)}) \right] \le \mathbb{E}\left[ F(\bs{\bar{x}}_{\bs{\hat{\phi}}}^{(k)}) \right] - \gamma_0^{(k)} \mathbb{E}\left[ \left\| \nabla{F}(\bs{\bar{x}}_{\bs{\hat{\phi}}}^{(k)}) \right\|^2 \right] & + \psi_{01}^{(k)} \mathbb{E} \left[ \left\| \bs{X}^{(k)} - \bs{1} \bs{\bar{x}}_{\bs{\hat{\phi}}}^{(k)} \right\|_{\bs{\hat{\phi}}}^2 \right]
            \\
            & + \psi_{02}^{(k)} \mathbb{E}\left[ \left\| \bs{\Lambda}_{\bs{\hat{\pi}}}^{-1} \bs{Y}^{(k)} - m \bs{1} \bs{\bar{y}}^{(k)} \right\|_{\bs{\hat{\pi}}}^2 \right] + \omega_0^{(k)},
        \end{aligned}
    \end{equation}
    in which we have $\gamma_0^{(k)} = \frac12 m \overline{\hat{\phi} \eta^{(k)} \left( m \hat{\pi} - 2 \sigma_1^2 p B^{-1} (1 - B/D) - 3 (1 - p) \delta_1^2 - 3 \hat{\rho}_{0,B} \kappa_2 \right)}$, $\psi_{01}^{(k)} = \frac12 (1 + 3 \hat{\rho}_{0,B}) \kappa_8 \max(\eta_i^{(k)})$, $\psi_{02}^{(k)} = \frac{3 \max(\hat{\phi}_i) \tilde{\rho}_B}{2m} \max(\eta_i^{(k)})$ and $\omega_0^{(k)} = \frac12 m (1 + 3 \hat{\rho}_{0,B}) \overline{\hat{\phi} \eta^{(k)} \left( \sigma_0^2 p B^{-1} (1 - B/D) + 3 (1 - p) \delta_0^2 \right)}$, in which $\kappa_8 = \max( p_i ( 2 \frac{(1 - B_i / D_i) \sigma_{1,i}^2}{B_i} \bar{L}^2 + 3 L_i^2 ))$.

    \begin{proof}
        See Appendix~\ref{app:lemma:loss}.
    \end{proof}
\end{lemma}
We observe that when clients conduct GT in every training iteration, i.e., $p_i = 1$ for all $i \in \mathcal{M}$, and they participate in peer-to-peer communications in every training iteration, i.e., $\hat{p}_{ij} = 1$ for all $(i,j) \in \mathcal{E}$ resulting in $\hat{\rho}_{0,B} = 0$, and also the gradient estimation variance is uniformly bounded, i.e., $\sigma_{1,i}^2 = 0$, then $\gamma_0^{(k)} = \frac12 m^2 \overline{\hat{\phi} \hat{\pi} \eta^{(k)}}$. Furthermore, if the underlying graph admits doubly stochastic matrices $\bs{\hat{A}}$ and $\bs{\hat{B}}$ to get $\bs{\hat{\phi}} = \bs{\hat{\pi}} = \frac1m \bs{1}$, this simplifies to $\gamma_0^{(k)} = \frac12 \overline{\eta^{(k)}}$. On the other hand, if clients conduct GT in every training iteration, i.e., $p_i = 1$ for all $i \in \mathcal{M}$, and they use a full batch, i.e., $B_i = D_i$, we will have $\omega_0^{(k)} = 0$.

Applying Definition~\ref{def:error_vec} further to Lemma~\ref{lemma:loss}, we get
\begin{equation} \label{eqn:loss_redefined}
    \mathbb{E}\left[ F(\bs{\bar{x}}_{\bs{\hat{\phi}}}^{(k+1)}) \right] \le \mathbb{E} \left[ F(\bs{\bar{x}}_{\bs{\hat{\phi}}}^{(k)}) \right] - \gamma_0^{(k)} \mathbb{E}\left[ \left\| \nabla{F}(\bs{\bar{x}}_{\bs{\hat{\phi}}}^{(k)}) \right\|^2 \right] + (\bs{\psi}_0^{(k)})^T \bs{\varsigma}^{(k)} + \omega_0^{(k)},
\end{equation}
and recursively expanding Eq.~\eqref{eqn:loss_redefined}, we obtain the following explicit inequality for the expected loss
\begin{equation} \label{eqn:loss_explicit}
    \mathbb{E}\left[ F(\bs{\bar{x}}_{\bs{\hat{\phi}}}^{(k+1)}) \right] \le F(\bs{\bar{x}}_{\bs{\hat{\phi}}}^{(0)}) - \sum_{r=0}^k{\gamma_0^{(r)} \mathbb{E}\left[ \left\| \nabla{F}(\bs{\bar{x}}_{\bs{\hat{\phi}}}^{(r)}) \right\|^2 \right]} + \sum_{r=0}^k{(\bs{\psi}_0^{(r)})^T \bs{\varsigma}^{(r)}} + \sum_{r=0}^k{\omega_0^{(r)}}.
\end{equation}
Finally, substituting Eq.~\eqref{eqn:varsigma_explicit} into Eq.~\eqref{eqn:loss_explicit} gives us
\begin{equation} \label{eqn:loss_expanded}
    \boxed{
    \begin{aligned}
        \mathbb{E}\left[ F(\bs{\bar{x}}_{\bs{\hat{\phi}}}^{(k+1)}) \right] \le F(\bs{\bar{x}}_{\bs{\hat{\phi}}}^{(0)}) & - \gamma_0^{(k)} \mathbb{E}\left[ \left\| \nabla{F}(\bs{\bar{x}}_{\bs{\hat{\phi}}}^{(k)}) \right\|^2 \right]
        \\
        & - \sum_{r=0}^{k-1}{\left( \gamma_0^{(r)} - \left( \sum_{s=r+1}^k{(\bs{\psi}_0^{(s)})^T \bs{\Psi}^{(s-2:r)}} \right) \bs{\gamma}^{(r)} \right) \mathbb{E}\left[ \left\| \nabla{F}(\bs{\bar{x}}_{\bs{\hat{\phi}}}^{(r)}) \right\|^2 \right]}
        \\
        & + \left( \sum_{r=0}^k{(\bs{\psi}_0^{(r)})^T \bs{\Psi}^{(r-1:0)}} \right) \bs{\varsigma}^{(0)} + \sum_{r=0}^{k-1}{\left( \sum_{s=r+1}^k{(\bs{\psi}_0^{(s)})^T \bs{\Psi}^{(s-2:r)}} \right) \bs{\omega}^{(r)}} + \sum_{r=0}^k{\omega_0^{(r)}}.
    \end{aligned}
    }
\end{equation}
Eq.~\eqref{eqn:loss_expanded} is a key equation that is used in both Propositions in Sec.~\ref{ssec:props}.

\section{Corollaries to Proposition~\ref{proposition:lr}} \label{app:corollaries}

We proceed with two corollaries to Proposition~\ref{proposition:lr} when further structure is placed on the learning rates.
\begin{corollary}[Non-Increasing Learning Rates] \label{corollary:nonincreasing}
    Let Assumptions~\ref{assump:lipschitz}-\ref{assump:graph} hold. Let each client $i \in \mathcal{M}$ use a non-increasing learning rate $\eta_i^{(k+1)} \le \eta_i^{(k)}$. Then, it is sufficient for the constraint of Proposition~\ref{proposition:lr} to hold only for the initial values $\eta_i^{(0)}$ to ensure $\rho(\bs{\Psi}^{(k)}) < 1$ for all $k \geq 0$.

\end{corollary}

\begin{corollary}[Constant Learning Rates] \label{corollary:constant}
    Let Assumptions~\ref{assump:lipschitz}-\ref{assump:graph} hold. Let each client $i \in \mathcal{M}$ use a constant learning rate, i.e., $\eta_i^{(k)} = \eta_i^{(0)}$. Then, the values defined in Lemmas~\ref{lemma:xdispersion}-\ref{lemma:loss} will all become constants, i.e., $\bs{\Psi}^{(k)} = \bs{\Psi}^{(0)}$, $\bs{\gamma}^{(k)} = \bs{\gamma}^{(0)}$, $\bs{\omega}^{(k)} = \bs{\omega}^{(0)}$, $\bs{\psi}_0^{(k)} = \bs{\psi}_0^{(0)}$, $\gamma_0^{(k)} = \gamma_0^{(0)}$ and $\omega_0^{(k)} = \omega_0^{(0)}$, for all $k \geq 0$. Furthermore, if constant learning rates satisfy the conditions outlined in Corollary~\ref{corollary:nonincreasing}, then there exists a scalar $0 < \rho_\Psi < 1$ such that $\rho(\bs{\Psi}^{(k)}) = \rho_\Psi$ for all $k \geq 0$.
\end{corollary}

\section{Proofs of Intermediary Lemmas}

\subsection{Proof of Lemma~\ref{lemma:graderror}} \label{app:lemma:graderror}

\ref{lemma:graderror:sgd} According to Jensen's inequality, we have
\begin{equation}
    \left\| \bar{\bs{g}}_v^{(k)} - \overline{\nabla{\bs{F}}(\bs{X}^{(k)})}_{\bs{v}^{(k)}} \right\|^2 = \left\| \frac1m \sum_{i=1}^m{(\bs{g}_i^{(k)} - \nabla{F}_i(\bs{x}_i^{(k)})) v_i^{(k)}} \right\|^2 \le \frac1m \sum_{i=1}^m{\left\| \bs{g}_i^{(k)} - \nabla{F}_i(\bs{x}_i^{(k)}) \right\|^2 v_i^{(k)}}.
\end{equation}
Taking the expected value of this expression and using Assumptions~\ref{assump:sgd} and \ref{assump:bernoulli}, we obtain the following.
\begin{equation} \label{eqn:graderror_crossref}
    \begin{aligned}
        \mathbb{E}\left[ \frac1m \sum_{i=1}^m{\left\| \bs{g}_i^{(k)} - \nabla{F}_i(\bs{x}_i^{(k)}) \right\|^2 v_i^{(k)}} \right] & = \frac1m \sum_{i=1}^m{\mathbb{E}\left[ \left\| \bs{g}_i^{(k)} - \nabla{F}_i(\bs{x}_i^{(k)}) \right\|^2 \right] \mathbb{E}[v_i^{(k)}]}
        \\
        & \le \frac1m \sum_{i=1}^m{\frac{p_i (1 - B_i / D_i)}{B_i} \left( \sigma_{0,i}^2 + \sigma_{1,i}^2 \mathbb{E}\left[ \left\| \nabla{F}(\bs{x}_i^{(k)}) \right\|^2 \right] \right)}
        \\
        & = \overline{\sigma_0^2 p B^{-1} (1 - B/D)} + \frac1m \sum_{i=1}^m{\frac{\sigma_{1,i}^2 p_i (1 - B_i / D_i)}{B_i} \mathbb{E}\left[ \left\| \nabla{F}(\bs{x}_i^{(k)}) \right\|^2 \right]},
    \end{aligned}
\end{equation}
where we used Assumption~\ref{assump:sgd}. We finally use the Lipschitz continuity of the gradients of the global loss function $F$ which is implied by Assumption~\ref{assump:lipschitz} to write
\begin{equation} \label{eqn:graderror_crossref:norm}
    \begin{aligned}
        \frac1m & \sum_{i=1}^m{\frac{\sigma_{1,i}^2 p_i (1 - B_i / D_i)}{B_i} \left\| \nabla{F}(\bs{x}_i^{(k)}) \right\|^2} = \frac1m \sum_{i=1}^m{\frac{\sigma_{1,i}^2 p_i (1 - B_i / D_i)}{B_i} \left\| \nabla{F}(\bs{x}_i^{(k)}) \mp \nabla{F}(\bs{\bar{x}}_{\bs{\hat{\phi}}}^{(k)}) \right\|^2}
        \\
        & \le \frac2m \sum_{i=1}^m \frac{\sigma_{1,i}^2 p_i (1 - B_i / D_i)}{B_i} \left( \left\| \nabla{F}(\bs{x}_i^{(k)}) - \nabla{F}(\bs{\bar{x}}_{\bs{\hat{\phi}}}^{(k)}) \right\|^2 + \left\| \nabla{F}(\bs{\bar{x}}_{\bs{\hat{\phi}}}^{(k)}) \right\|^2 \right)
        \\
        & \le \frac2m \sum_{i=1}^m \frac{\sigma_{1,i}^2 p_i (1 - B_i / D_i)}{B_i} \bar{L}^2 \left\| \bs{x}_i^{(k)} - \bs{\bar{x}}_{\bs{\hat{\phi}}}^{(k)} \right\|^2 + 2 \overline{\sigma_1^2 p B^{-1} (1 - B/D)} \left\| \nabla{F}(\bs{\bar{x}}_{\bs{\hat{\phi}}}^{(k)}) \right\|^2
        \\
        & = \frac{2\bar{L}^2}m \sum_{i=1}^m \frac{\sigma_{1,i}^2 p_i (1 - B_i / D_i)}{B_i} \frac{\hat{\phi}_i}{\hat{\phi}_i} \left\| \bs{x}_i^{(k)} - \bs{\bar{x}}_{\bs{\hat{\phi}}}^{(k)} \right\|^2 + 2 \overline{\sigma_1^2 p B^{-1} (1 - B/D)} \left\| \nabla{F}(\bs{\bar{x}}_{\bs{\hat{\phi}}}^{(k)}) \right\|^2
        \\
        & \le \frac{2\bar{L}^2}m \max\left( \frac{\sigma_{1,i}^2 p_i (1 - B_i / D_i)}{B_i \hat{\phi}_i} \right) \left\| \bs{X}^{(k)} - \bs{1} \bs{\bar{x}}_{\bs{\hat{\phi}}}^{(k)} \right\|_{\bs{\hat{\phi}}}^2 + 2 \overline{\sigma_1^2 p B^{-1} (1 - B/D)} \left\| \nabla{F}(\bs{\bar{x}}_{\bs{\hat{\phi}}}^{(k)}) \right\|^2.
    \end{aligned}
\end{equation}

\ref{lemma:graderror:lipschitz} According to Jensen's inequality, we have
\begin{equation}
    \begin{aligned}
        \left\| \overline{\nabla{\bs{F}}(\bs{X}^{(k)})}_{\bs{v}^{(k)}} - \overline{\nabla{\bs{F}}(\bs{1} \bs{\bar{x}}_{\bs{\hat{\phi}}}^{(k)})}_{\bs{v}^{(k)}} \right\|^2 & = \left\| \frac1m \sum_{i=1}^m{\left( \nabla{F}_i(\bs{x}_i^{(k)}) - \nabla{F}_i(\bs{\bar{x}}_{\bs{\hat{\phi}}}^{(k)}) \right) v_i^{(k)}} \right\|^2
        \\
        & \le \frac1m \sum_{i=1}^m{\left\| \nabla{F}_i(\bs{x}_i^{(k)}) - \nabla{F}_i(\bs{\bar{x}}_{\bs{\hat{\phi}}}^{(k)}) \right\|^2 v_i^{(k)}} \le \frac1m \sum_{i=1}^m{L_i^2 \left\| \bs{x}_i^{(k)} - \bs{\bar{x}}_{\bs{\hat{\phi}}}^{(k)} \right\|^2 v_i^{(k)}},
    \end{aligned}
\end{equation}
where we used Assumption~\ref{assump:lipschitz} in the end. Taking the expected value of this expression and using Assumption~\ref{assump:bernoulli}, we get
\begin{equation}
    \begin{aligned}
        \mathbb{E} & \left[ \frac1m \sum_{i=1}^m{L_i^2 \left\| \bs{x}_i^{(k)} - \bs{\bar{x}}_{\bs{\hat{\phi}}}^{(k)} \right\|^2 v_i^{(k)}} \right] = \frac1m \sum_{i=1}^m{L_i^2 \mathbb{E}\left[ \left\| \bs{x}_i^{(k)} - \bs{\bar{x}}_{\bs{\hat{\phi}}}^{(k)} \right\|^2 \right] \mathbb{E}[ v_i^{(k)} ]} = \frac1m \sum_{i=1}^m{L_i^2 p_i \mathbb{E}\left[ \left\| \bs{x}_i^{(k)} - \bs{\bar{x}}_{\bs{\hat{\phi}}}^{(k)} \right\|^2 \right]}
        \\
        & = \frac1m \sum_{i=1}^m{L_i^2 p_i \frac{\hat{\phi}_i}{\hat{\phi}_i} \mathbb{E}\left[ \left\| \bs{x}_i^{(k)} - \bs{\bar{x}}_{\bs{\hat{\phi}}}^{(k)} \right\|^2 \right]} \le \frac{\max(L_i^2 p_i \hat{\phi}_i^{-1})}m \sum_{i=1}^m{\hat{\phi}_i \mathbb{E}\left[ \left\| \bs{x}_i^{(k)} - \bs{\bar{x}}_{\bs{\hat{\phi}}}^{(k)} \right\|^2 \right]}
        \\
        & = \frac{\max(L_i^2 p_i \hat{\phi}_i^{-1})}m \mathbb{E}\left[ \left\| \bs{X}^{(k)} - \bs{1} \bs{\bar{x}}_{\bs{\hat{\phi}}}^{(k)} \right\|_{\bs{\hat{\phi}}}^2 \right].
    \end{aligned}
\end{equation}

\ref{lemma:graderror:bernoulli} According to Jensen's inequality, we have
\begin{equation}
    \begin{aligned}
        \left\| \overline{\nabla{\bs{F}}(\bs{1} \bs{\bar{x}}_{\bs{\hat{\phi}}}^{(k)})}_{\bs{v}^{(k)}} - \nabla{F}(\bar{\bs{x}}_{\bs{\hat{\phi}}}^{(k)}) \right\|^2 & = \left\| \frac1m \sum_{i=1}^m{\left( \nabla{F}_i(\bar{\bs{x}}_{\bs{\hat{\phi}}}^{(k)}) v_i^{(k)} - \nabla{F}_i(\bar{\bs{x}}_{\bs{\hat{\phi}}}^{(k)}) \right)} \right\|^2 = \left\| \frac1m \sum_{i=1}^m{\nabla{F}_i(\bar{\bs{x}}_{\bs{\hat{\phi}}}^{(k)}) (v_i^{(k)} - 1)} \right\|^2
        \\
        & \le \frac1m \sum_{i=1}^m{\left\| \nabla{F}_i(\bar{\bs{x}}_{\bs{\hat{\phi}}}^{(k)}) \right\|^2 (1 - v_i^{(k)})}
        \\
        & \le \frac1m \sum_{i=1}^m{\left( \delta_{0,i}^2 + \delta_{1,i}^2 \left\| \nabla{F}(\bar{\bs{x}}_{\bs{\hat{\phi}}}^{(k)}) \right\|^2 \right) (1 - v_i^{(k)})}.
    \end{aligned}
\end{equation}
Taking the expected value of this expression and using Assumptions~\ref{assump:graddiv} and \ref{assump:bernoulli}, we obtain the following.
\begin{equation}
    \begin{aligned}
        \mathbb{E} & \left[ \frac1m \sum_{i=1}^m{\left( \delta_{0,i}^2 + \delta_{1,i}^2 \left\| \nabla{F}(\bar{\bs{x}}_{\bs{\hat{\phi}}}^{(k)}) \right\|^2 \right) (1 - v_i^{(k)})} \right] = \frac1m \sum_{i=1}^m{\left( \delta_{0,i}^2 + \delta_{1,i}^2 \mathbb{E}\left[ \left\| \nabla{F}(\bar{\bs{x}}_{\bs{\hat{\phi}}}^{(k)}) \right\|^2 \right] \right) (1 - \mathbb{E}[ v_i^{(k)} ])}
        \\
        & = \frac1m \sum_{i=1}^m{(1 - p_i) \left( \delta_{0,i}^2 + \delta_{1,i}^2 \mathbb{E}\left[ \left\| \nabla{F}(\bar{\bs{x}}_{\bs{\hat{\phi}}}^{(k)}) \right\|^2 \right] \right)} = \overline{(1 - p) \delta_0^2} + \overline{(1 - p) \delta_1^2} \left\| \nabla{F}(\bar{\bs{x}}_{\bs{\hat{\phi}}}^{(k)}) \right\|^2.
    \end{aligned}
\end{equation}

\subsection{Proof of Lemma~\ref{lemma:descenteffect}} \label{app:lemma:descenteffect}

\ref{lemma:descenteffect:ybar} We derive this upper bound using steps $(b)$-$(e)$ of Eq.~\eqref{eqn:grad_steps}. We have
\begin{equation} \label{eqn:ynorm}
    \begin{aligned}
        & \left\| \bs{\bar{y}}^{(k)} \right\|^2 = \left\| \bar{\bs{g}}_{\bs{v}}^{(k)} \right\|^2 = \left\| \bar{\bs{g}}_{\bs{v}}^{(k)} \mp \overline{\nabla{\bs{F}}(\bs{X}^{(k)})}_{\bs{v}^{(k)}} \mp \overline{\nabla{\bs{F}}(\bs{1} \bs{\bar{x}}_{\bs{\hat{\phi}}}^{(k)})}_{\bs{v}^{(k)}} \mp \nabla{F}(\bar{\bs{x}}_{\bs{\hat{\phi}}}^{(k)}) \right\|^2
        \\
        & \begin{aligned}
            \,\, = \left\| \bar{\bs{g}}_{\bs{v}}^{(k)} - \overline{\nabla{\bs{F}}(\bs{X}^{(k)})}_{\bs{v}^{(k)}} \right\|^2 & + \left\| \overline{\nabla{\bs{F}}(\bs{X}^{(k)})}_{\bs{v}^{(k)}} \mp \overline{\nabla{\bs{F}}(\bs{1} \bs{\bar{x}}_{\bs{\hat{\phi}}}^{(k)})}_{\bs{v}^{(k)}} \mp \nabla{F}(\bar{\bs{x}}_{\bs{\hat{\phi}}}^{(k)}) \right\|^2
            \\
            & + 2 \bigg\langle \bar{\bs{g}}_{\bs{v}}^{(k)} - \overline{\nabla{\bs{F}}(\bs{X}^{(k)})}_{\bs{v}^{(k)}}, \overline{\nabla{\bs{F}}(\bs{X}^{(k)})}_{\bs{v}^{(k)}}  \mp \overline{\nabla{\bs{F}}(\bs{1} \bs{\bar{x}}_{\bs{\hat{\phi}}}^{(k)})}_{\bs{v}^{(k)}} \mp \nabla{F}(\bar{\bs{x}}_{\bs{\hat{\phi}}}^{(k)}) \bigg\rangle,
        \end{aligned}
    \end{aligned}
\end{equation}
where we separated the inner product term because it will be equal to $0$ when taking its expected value, since $\bs{g}_i^{(k)}$ is an unbiased estimate of $\nabla{F}_i(\bs{x}_i^{(k)})$ according to the Assumption~\ref{assump:sgd}. We further expand the bound as follows.
\begin{equation} \label{eqn:ynorm_rest}
    \begin{aligned}
        & \left\| \overline{\nabla{\bs{F}}(\bs{X}^{(k)})}_{\bs{v}^{(k)}}  \mp \overline{\nabla{\bs{F}}(\bs{1} \bs{\bar{x}}_{\bs{\hat{\phi}}}^{(k)})}_{\bs{v}^{(k)}} \mp \nabla{F}(\bar{\bs{x}}_{\bs{\hat{\phi}}}^{(k)}) \right\|^2
        \\
        & \le 3 \left( \left\| \overline{\nabla{\bs{F}}(\bs{X}^{(k)})}_{\bs{v}^{(k)}} - \overline{\nabla{\bs{F}}(\bs{1} \bs{\bar{x}}_{\bs{\hat{\phi}}}^{(k)})}_{\bs{v}^{(k)}} \right\|^2 + \left\| \overline{\nabla{\bs{F}}(\bs{1} \bs{\bar{x}}_{\bs{\hat{\phi}}}^{(k)})}_{\bs{v}^{(k)}} - \nabla{F}(\bar{\bs{x}}_{\bs{\hat{\phi}}}^{(k)}) \right\|^2 + \left\| \nabla{F}(\bar{\bs{x}}_{\bs{\hat{\phi}}}^{(k)}) \right\|^2 \right),
    \end{aligned}
\end{equation}
in which we used Young's inequality. We add Eq.~\eqref{eqn:ynorm_rest} back to Eq.~\eqref{eqn:ynorm}, which results in an upper bound involving the three expressions given in Lemma~\ref{lemma:graderror}. Taking the expected value of the result concludes the proof. Note that while the upper bounds in Lemmas~\ref{lemma:graderror}-\ref{lemma:graderror:sgd} and \ref{lemma:graderror}-\ref{lemma:graderror:lipschitz} each involve a $\max(\cdot)$ term, their proofs in Appendix~\ref{app:lemma:graderror} show that we can first combine their coefficients and then take their maximum.

\ref{lemma:descenteffect:ypiinv} See Lemma~4.6 in \citet{nedic2025ab}.

\ref{lemma:descenteffect:yconsphi} According to Corollary 1 of Lemma~1(b) in \citet{nguyen2025distributed}, we have $\sum_{i=1}^m{\gamma_i \| u_i - \sum_{\ell=1}^m{\gamma_\ell u_\ell} \|^2} \le \sum_{i=1}^m{\gamma_i \| u_i \|^2}$, which holds for any collection of vectors $\{ u_i \}_{i=1}^m$ and scalars $\{ \gamma_i \}_{i=1}^m$ such that $\sum_{i=1}^m{\gamma_i} = 1$. Thus, we have
\begin{equation} \label{eqn:yconsphi_crossref}
    \begin{aligned}
        & \left\| \bs{\Lambda}_{\bs{\eta}}^{(k)} \bs{\hat{B}}^{(k)} \bs{Y}^{(k)} - \bs{1} \bs{\hat{\phi}}^T \bs{\Lambda}_{\bs{\eta}}^{(k)} \bs{\hat{B}}^{(k)} \bs{Y}^{(k)} \right\|_{\bs{\hat{\phi}}}^2 \le \left\| \bs{\Lambda}_{\bs{\eta}}^{(k)} \bs{\hat{B}}^{(k)} \bs{Y}^{(k)} \right\|_{\bs{\hat{\phi}}}^2 = \sum_{i=1}^m{\hat{\phi}_i (\eta_i^{(k)})^2 \left\| \sum_{j=1}^m{\hat{b}_{ij}^{(k)} \bs{y}_j^{(k)}} \right\|^2}
        \\
        & = \sum_{i=1}^m{\hat{\phi}_i (\eta_i^{(k)})^2 \left\| \left( \sum_{\ell=1}^m{\hat{b}_{i\ell}^{(k)}} \right) \frac{\sum_{j=1}^m{\hat{b}_{ij}^{(k)} \bs{y}_j^{(k)}}}{\sum_{\ell=1}^m{b_{i\ell}^{(k)}}} \right\|^2} \le \sum_{i=1}^m{\hat{\phi}_i (\eta_i^{(k)})^2 \left( \sum_{\ell=1}^m{\hat{b}_{i\ell}^{(k)}} \right)^2 \frac{\sum_{j=1}^m{\hat{b}_{ij}^{(k)} \left\| \bs{y}_j^{(k)} \right\|^2}}{\sum_{\ell=1}^m{\hat{b}_{i\ell}^{(k)}}}}
        \\
        & = \sum_{i=1}^m{\hat{\phi}_i (\eta_i^{(k)})^2 \left( \sum_{\ell=1}^m{\hat{b}_{i\ell}^{(k)}} \right) \sum_{j=1}^m{\hat{b}_{ij}^{(k)} \left\| \bs{y}_j^{(k)} \right\|^2}}
        \\
        & \le \sum_{i=1}^m{\hat{\phi}_i (\eta_i^{(k)})^2 \Bigg( \bigg( \sum_{\substack{\ell=1 \\ \ell \neq i}}^m{\hat{b}_{ii}^{(k)} \hat{b}_{i\ell}^{(k)}} + (\hat{b}_{ii}^{(k)})^2 \bigg) \left\| \bs{y}_i^{(k)} \right\|^2 + \sum_{\substack{j=1 \\ j \neq i}}^m{\bigg( \sum_{\substack{\ell=1 \\ \ell \neq i \\ \ell \neq j}}^m{\hat{b}_{ij}^{(k)} \hat{b}_{i\ell}^{(k)}} + \hat{b}_{ij}^{(k)} \hat{b}_{ii}^{(k)} + (\hat{b}_{ij}^{(k)})^2 \bigg)} \left\| \bs{y}_j^{(k)} \right\|^2 \Bigg)},
    \end{aligned}
\end{equation}
where we used Jensen's inequality noting that $\bs{\hat{B}}^{(k)}$ is a column-stochastic matrix.
Taking the expected value of this expression and employing Assumption~\ref{assump:bernoulli}, we get
\begin{equation}
    \begin{aligned}
        & \mathbb{E}\left[ \left\| \bs{\Lambda}_{\bs{\eta}}^{(k)} \bs{\hat{B}}^{(k)} \bs{Y}^{(k)} \right\|_{\bs{\hat{\phi}}}^2 \right]
        \\
        & \begin{aligned}
            \le \sum_{i=1}^m & \hat{\phi}_i (\eta_i^{(k)})^2 \Bigg[ \Bigg( \sum_{\substack{\ell=1 \\ \ell \neq i}}^m{\bigg( 1 - \sum_{\substack{\ell=1 \\ \ell \neq i}}{b_{\ell i} \hat{p}_{\ell i}} \bigg) b_{i\ell} \hat{p}_{i\ell}} + \bigg( 1 - 2 \sum_{\substack{\ell=1 \\ \ell \neq i}}^m{b_{\ell i} \hat{p}_{\ell i}} + \sum_{\substack{\ell=1 \\ \ell \neq i}}^m{\sum_{\substack{\ell'=1 \\ \ell' \neq \ell}}^m{b_{\ell i} b_{\ell' i} \hat{p}_{\ell i} \hat{p}_{\ell' i}}} + \sum_{\substack{\ell=1 \\ \ell \neq i}}^m{b_{\ell i}^2 \hat{p}_{\ell i}} \bigg) \Bigg)
            \\
            & \mathbb{E}\left[ \left\| \bs{y}_i^{(k)} \right\|^2 \right] + \sum_{\substack{j=1 \\ j \neq i}}^m{\Bigg( \sum_{\substack{\ell=1 \\ \ell \neq i \\ \ell \neq j}}^m{b_{ij} b_{i\ell} \hat{p}_{ij} \hat{p}_{i\ell}} + b_{ij} \hat{p}_{ij} \bigg( 1 - \sum_{\substack{\ell=1 \\ \ell \neq i}}^m{b_{\ell i} \hat{p}_{\ell i}} \bigg) + b_{ij}^2 \hat{p}_{ij} \Bigg)} \mathbb{E}\left[ \left\| \bs{y}_j^{(k)} \right\|^2 \right] \Bigg]
        \end{aligned}
    \end{aligned}
\end{equation}
\begin{equation*}
    \begin{aligned}
        & \begin{aligned}
            = \sum_{i=1}^m & \hat{\phi}_i (\eta_i^{(k)})^2 \left( \sum_{\ell=1}^m{\hat{b}_{i\ell}} \right) \sum_{j=1}^m{\hat{b}_{ij} \mathbb{E}\left[ \left\| \bs{y}_j^{(k)} \right\|^2 \right]}
            \\
            & + \sum_{i=1}^m{\hat{\phi}_i (\eta_i^{(k)})^2 \Bigg[ \bigg( \sum_{\substack{\ell=1 \\ \ell \neq i}}^m{b_{\ell i}^2 \hat{p}_{\ell i} (1 - \hat{p}_{\ell i})} \bigg) \mathbb{E}\left[ \left\| \bs{y}_i^{(k)} \right\|^2 \right] + \sum_{\substack{j=1 \\ j \neq i}}^m{b_{ij}^2 \hat{p}_{ij} (1 - \hat{p}_{ij}) \mathbb{E}\left[ \left\| \bs{y}_j^{(k)} \right\|^2 \right]} \Bigg]}
        \end{aligned}
        \\
        & \le \sum_{i=1}^m{\hat{\phi}_i (\eta_i^{(k)})^2 \left( \sum_{\ell=1}^m{\frac{\hat{\pi}_\ell}{\hat{\pi}_\ell} \hat{b}_{i\ell}} \right) \sum_{j=1}^m{\hat{b}_{ij} \mathbb{E}\left[ \left\| \bs{y}_j^{(k)} \right\|^2 \right]}} + 2 \max(\hat{\phi}_i (\eta_i^{(k)})^2) \sum_{j=1}^m{\bigg( \sum_{\substack{\ell=1 \\ \ell \neq j}}^m{b_{\ell j}^2 \hat{p}_{\ell j} (1 - \hat{p}_{\ell j})} \bigg) \mathbb{E}\left[ \left\| \bs{y}_j^{(k)} \right\|^2 \right]}
        \\
        & \le \frac1{\min(\hat{\pi}_i)} \sum_{i=1}^m{\hat{\phi}_i (\eta_i^{(k)})^2 \hat{\pi}_i \sum_{j=1}^m{\hat{b}_{ij} \mathbb{E}\left[ \left\| \bs{y}_j^{(k)} \right\|^2 \right]}} + 2m \max(\hat{\phi}_i (\eta_i^{(k)})^2) \sum_{j=1}^m{\frac{\hat{\pi}_j}{\hat{\pi}_j} \overline{b_{:j}^2 \hat{p}_{:j} (1 - \hat{p}_{:j})} \mathbb{E}\left[ \left\| \bs{y}_j^{(k)} \right\|^2 \right]}
        \\
        & \le \frac{\max(\hat{\phi}_i (\eta_i^{(k)})^2 \hat{\pi}_i)}{\min(\hat{\pi}_i)} \sum_{j=1}^m{\mathbb{E}\left[ \left\| \bs{y}_j^{(k)} \right\|^2 \right] \sum_{i=1}^m{\hat{b}_{ij}}} + 2m \max(\hat{\phi}_i (\eta_i^{(k)})^2) \max(\hat{\pi}_j \overline{b_{:j}^2 \hat{p}_{:j} (1 - \hat{p}_{:j})}) \sum_{j=1}^m{\frac1{\hat{\pi}_j} \mathbb{E}\left[ \left\| \bs{y}_j^{(k)} \right\|^2 \right]}
        \\
        & \le \frac{\max(\hat{\phi}_i (\eta_i^{(k)})^2 \hat{\pi}_i)}{\min(\hat{\pi}_i)} \sum_{j=1}^m{\mathbb{E}\left[ \left\| \bs{y}_j^{(k)} \right\|^2 \right]} + 2m \max(\hat{\phi}_i (\eta_i^{(k)})^2) \max(\hat{\pi}_j \overline{b_{:j}^2 \hat{p}_{:j} (1 - \hat{p}_{:j})}) \mathbb{E}\left[ \left\| \bs{Y}^{(k)} \right\|_{\bs{\hat{\pi}}^{-1}}^2 \right]
        \\
        & \le \max(\hat{\phi}_i (\eta_i^{(k)})^2) \left( \frac{\max(\hat{\pi}_i)^2}{\min(\hat{\pi}_i)} + 2m \max(\hat{\pi}_j \overline{b_{:j}^2 \hat{p}_{:j} (1 - \hat{p}_{:j})}) \right) \mathbb{E}\left[ \left\| \bs{Y}^{(k)} \right\|_{\bs{\hat{\pi}}^{-1}}^2 \right],
    \end{aligned}
\end{equation*}
where we used the fact that $\bs{\hat{B}}^{(k)} \bs{\hat{\pi}} = \bs{\hat{\pi}}$ and $\bs{1}^T \bs{\hat{B}}^{(k)} = \bs{1}^T$ to simplify the first term.

\subsection{Proof of Lemma~\ref{lemma:sporadic_contraction}} \label{app:lemma:sporadic_contraction}

\ref{lemma:sporadic_contraction:A} We start by expanding the norm as
\begin{equation} \label{eqn:spod_contract_A_halfway}
	\begin{aligned}
		\left\| \bs{\hat{A}}^{(k)} \bs{X}^{(k)} - \bs{1} \bs{\hat{\phi}}^T \bs{X}^{(k)} \right\|_{\bs{\hat{\phi}}}^2 & = \sum_{i=1}^m{\hat{\phi}_i \left\| (\bs{\hat{a}}_i^{(k)})^T \bs{X}^{(k)} - \bs{\hat{\phi}}^T \bs{X}^{(k)} \right\|^2} = \sum_{i=1}^m{\left\| \sqrt{\hat{\phi}_i} \left( (\bs{\hat{a}}_i^{(k)})^T \bs{X}^{(k)} - \bs{\hat{\phi}}^T \bs{X}^{(k)} \right) \right\|^2}
		\\
		& = \left\| \bs{\Lambda}_{\bs{\hat{\phi}}}^{1/2} \left( \bs{\hat{A}}^{(k)} \bs{X}^{(k)} - \bs{1} \bs{\hat{\phi}}^T \bs{X}^{(k)} \right) \right\|^2 = \left\| \bs{\Lambda}_{\bs{\hat{\phi}}}^{1/2} \left( \bs{\hat{A}}^{(k)} \bs{X}^{(k)} - \bs{\hat{A}}^{(k)} \bs{1} \bs{\hat{\phi}}^T \bs{X}^{(k)} \right) \right\|^2
		\\
		& = \left\| \bs{\Lambda}_{\bs{\hat{\phi}}}^{1/2} \bs{\hat{A}}^{(k)} \left( \bs{I} - \bs{1} \bs{\hat{\phi}}^T \right) \bs{X}^{(k)} \right\|^2 = \sum_{d=1}^n{\left\| \bs{\Lambda}_{\bs{\hat{\phi}}}^{1/2} \bs{\hat{A}}^{(k)} \left( \bs{I} - \bs{1} \bs{\hat{\phi}}^T \right) [\bs{X}^{(k)}]_{[:,d]} \right\|^2}
		\\
		& = \sum_{d=1}^n{[\bs{X}^{(k)}]_{[:,d]}^T \left( \bs{I} - \bs{\hat{\phi}} \bs{1}^T \right) (\bs{\hat{A}}^{(k)})^T \bs{\Lambda}_{\bs{\hat{\phi}}} \bs{\hat{A}}^{(k)} \left( \bs{I} - \bs{1} \bs{\hat{\phi}}^T \right) [\bs{X}^{(k)}]_{[:,d]}}.
	\end{aligned}
\end{equation}
Noting that when taking the expected value of the above expression we can push the expectation operation inside to $\mathbb{E}[(\bs{\hat{A}}^{(k)})^T \bs{\Lambda}_{\bs{\hat{\phi}}} \bs{\hat{A}}^{(k)}]$, we first determine the values of this product. We have
\begin{equation}
	\begin{aligned}
		& (\bs{\hat{A}}^{(k)})^T \bs{\Lambda}_{\bs{\hat{\phi}}} \bs{\hat{A}}^{(k)} = \left[ \sum_{\ell=1}^m{\hat{\phi}_\ell \hat{a}_{\ell i}^{(k)} \hat{a}_{\ell j}^{(k)}} \right]_{1 \le i,j \le m}
		\\
		& = \begin{cases}
			\sum_{\substack{\ell=1 \\ \ell \neq i}}^m{\hat{\phi}_\ell a_{\ell i}^2 \hat{v}_{\ell i}^{(k)}} + \hat{\phi}_i \left( 1 - \sum_{\substack{\ell=1 \\ \ell \neq i}}^m{a_{i\ell} \hat{v}_{i\ell}^{(k)}} \right)^2 & i = j,
			\\
			\sum_{\substack{\ell=1 \\ \ell \neq i \\ \ell \neq j}}^m{\hat{\phi}_\ell a_{\ell i} a_{\ell j} \hat{v}_{\ell i}^{(k)} \hat{v}_{\ell j}^{(k)}} + \hat{\phi}_i \left( 1 - \sum_{\substack{\ell=1 \\ \ell \neq i}}^m{a_{i\ell} \hat{v}_{i\ell}^{(k)}} \right) a_{ij} \hat{v}_{ij}^{(k)} + \hat{\phi}_j a_{ji} \hat{v}_{ji}^{(k)} \left( 1 - \sum_{\substack{\ell=1 \\ \ell \neq j}}^m{a_{j\ell} \hat{v}_{j\ell}^{(k)}} \right) & i \neq j
		\end{cases}
	\end{aligned}
\end{equation}
\begin{equation*}
    \begin{aligned}
        & = \begin{cases}
            \sum_{\substack{\ell=1 \\ \ell \neq i}}^m{\hat{\phi}_\ell a_{\ell i}^2 \hat{v}_{\ell i}^{(k)}} + \hat{\phi}_i \Bigg( 1 - 2 \sum_{\substack{\ell=1 \\ \ell \neq i}}^m{a_{i\ell} \hat{v}_{i\ell}^{(k)}} + \sum_{\substack{\ell=1 \\ \ell \neq i}}^m{\sum_{\substack{\ell'=1 \\ \ell' \neq i \\ \ell' \neq \ell}}^m{a_{i\ell} a_{i\ell'} \hat{v}_{i\ell}^{(k)} \hat{v}_{i \ell'}^{(k)}}} + \sum_{\substack{\ell=1 \\ \ell \neq i}}^m{a_{i\ell}^2 \hat{v}_{i\ell}^{(k)}} \Bigg) & i = j,
            \\
            \begin{aligned}
                \sum_{\substack{\ell=1 \\ \ell \neq i \\ \ell \neq j}}^m{\hat{\phi}_\ell a_{\ell i} a_{\ell j} \hat{v}_{\ell i}^{(k)} \hat{v}_{\ell j}^{(k)}} + \hat{\phi}_i \Bigg( 1 - \sum_{\substack{\ell=1 \\ \ell \neq i \\ \ell \neq j}}^m{a_{i\ell} \hat{v}_{i\ell}^{(k)}} \Bigg) a_{ij} \hat{v}_{ij}^{(k)} - \hat{\phi}_i a_{ij}^2 \hat{v}_{ij}^{(k)} & + \hat{\phi}_j a_{ji} \hat{v}_{ji}^{(k)} \Bigg( 1 - \sum_{\substack{\ell=1 \\ \ell \neq j \\ \ell \neq i}}^m{a_{j\ell} \hat{v}_{j\ell}^{(k)}} \Bigg)
                \\
                & - \hat{\phi}_j a_{ji}^2 \hat{v}_{ji}^{(k)}
            \end{aligned} & i \neq j,
        \end{cases}
    \end{aligned}
\end{equation*}
where we considered the terms that are statistically dependent on each other together. Now, taking the expected value of this product and using Assumption~\ref{assump:bernoulli}, we get
\begin{equation} \label{eqn:spod_contract_A0}
	\begin{aligned}
		& \mathbb{E}\left[ (\bs{\hat{A}}^{(k)})^T \bs{\Lambda}_{\bs{\hat{\phi}}} \bs{\hat{A}}^{(k)} \right]
		\\
		& = \begin{cases}
			\sum_{\substack{\ell=1 \\ \ell \neq i}}^m{\hat{\phi}_\ell a_{\ell i}^2 \hat{p}_{\ell i}} + \hat{\phi}_i \Bigg( 1 - 2 \sum_{\substack{\ell=1 \\ \ell \neq i}}^m{a_{i\ell} \hat{p}_{i\ell}} + \sum_{\substack{\ell=1 \\ \ell \neq i}}^m{\sum_{\substack{\ell'=1 \\ \ell' \neq i \\ \ell' \neq \ell}}^m{a_{i\ell} a_{i\ell'} \hat{p}_{i\ell} \hat{p}_{i \ell'}}} + \sum_{\substack{\ell=1 \\ \ell \neq i}}^m{a_{i\ell}^2 \hat{p}_{i\ell}} \Bigg) & i = j,
			\\
			\begin{aligned}
			    \sum_{\substack{\ell=1 \\ \ell \neq i \\ \ell \neq j}}^m{\hat{\phi}_\ell a_{\ell i} a_{\ell j} \hat{p}_{\ell i} \hat{p}_{\ell j}} + \hat{\phi}_i \Bigg( 1 - \sum_{\substack{\ell=1 \\ \ell \neq i \\ \ell \neq j}}^m{a_{i\ell} \hat{p}_{i\ell}} \Bigg) a_{ij} \hat{p}_{ij} - \hat{\phi}_i a_{ij}^2 \hat{p}_{ij} & + \hat{\phi}_j a_{ji} \hat{p}_{ji} \Bigg( 1 - \sum_{\substack{\ell=1 \\ \ell \neq j \\ \ell \neq i}}^m{a_{j\ell} \hat{p}_{j\ell}} \Bigg)
                \\
                & - \hat{\phi}_j a_{ji}^2 \hat{p}_{ji}
			\end{aligned} & i \neq j
		\end{cases}
		\\
		& = \bs{\hat{A}}^T \bs{\Lambda}_{\bs{\hat{\phi}}} \bs{\hat{A}} + \begin{cases}
			\sum_{\substack{\ell=1 \\ \ell \neq i}}^m{\hat{\phi}_\ell a_{\ell i}^2 \hat{p}_{\ell i} (1 - \hat{p}_{\ell i})} + \hat{\phi}_i \sum_{\substack{\ell=1 \\ \ell \neq i}}^m{a_{i\ell}^2 \hat{p}_{i\ell} (1 - \hat{p}_{i\ell})} & i = j,
			\\
			- \hat{\phi}_i a_{ij}^2 \hat{p}_{ij} (1 - \hat{p}_{ij}) - \hat{\phi}_j a_{ji}^2 \hat{p}_{ji} (1 - \hat{p}_{ji}) & i \neq j
		\end{cases}
        \\
        & = \bs{\hat{A}}^T \bs{\Lambda}_{\bs{\hat{\phi}}} \bs{\hat{A}} + \begin{cases}
			(m - 1) \hat{\phi}_i \left( \overline{\hat{\phi} \hat{\phi}_i^{-1} a_{:i}^2 \hat{p}_{:i} (1 - \hat{p}_{:i})} + \overline{a_{i:}^2 \hat{p}_{i:} (1 - \hat{p}_{i:})} \right) & i = j,
			\\
			- \sqrt{\hat{\phi}_i \hat{\phi}_j} \left( \sqrt{\frac{\hat{\phi}_i}{\hat{\phi}_j}} a_{ij}^2 \hat{p}_{ij} (1 - \hat{p}_{ij}) + \sqrt{\frac{\hat{\phi}_j}{\hat{\phi}_i}} a_{ji}^2 \hat{p}_{ji} (1 - \hat{p}_{ji}) \right) & i \neq j
		\end{cases} = \bs{\hat{A}}^T \bs{\Lambda}_{\bs{\hat{\phi}}} \bs{\hat{A}} + \bs{\Lambda}_{\bs{\hat{\phi}}}^{1/2} \bs{\hat{A}}_0 \bs{\Lambda}_{\bs{\hat{\phi}}}^{1/2}.
	\end{aligned}
\end{equation}
Note that the matrix $\bs{\hat{A}}_0$ will be an all-zeroes matrix if $\hat{p}_{ij} = 1$ for all links $(i, j) \in \mathcal{E}$. Plugging the result in Eq.~\eqref{eqn:spod_contract_A0} back into Eq.~\eqref{eqn:spod_contract_A_halfway}, we get
\begin{equation} \label{eqn:spod_contract_A_almost}
	\mathbb{E}\left[ \left\| \bs{\hat{A}}^{(k)} \bs{X}^{(k)} - \bs{1} \bs{\hat{\phi}}^T \bs{X}^{(k)} \right\|_{\bs{\hat{\phi}}}^2 \right] \le \left( \hat{\rho}_A + \rho(\bs{\hat{A}}_0) \right) \mathbb{E}\left[ \left\| \bs{X}^{(k)} - \bs{1} \bs{\hat{\phi}}^T \bs{X}^{(k)} \right\|_{\bs{\hat{\phi}}}^2 \right],
\end{equation}
in which $\rho(\bs{\hat{A}}_0) \le (m - 1) \max( \overline{(\hat{\phi} \hat{\phi}_i^{-1} + \sqrt{\hat{\phi} \hat{\phi}_i^{-1}}) a_{:i}^2 \hat{p}_{:i} (1 - \hat{p}_{:i})} + \overline{(1 + \sqrt{\hat{\phi}_i \hat{\phi}^{-1}}) a_{i:}^2 \hat{p}_{i:} (1 - \hat{p}_{i:})})$. To ensure a contractive property for the upper bound in Eq.~\eqref{eqn:spod_contract_A_almost}, we need to ensure that $\rho(\bs{\hat{A}}_0) < \frac{1 - \hat{\rho}_A}2$. However, we enforce the stronger constraint $\rho(\bs{\hat{A}}_0) < \frac{1 - \hat{\rho}_A}4$ due to what will be needed later to prove Lemma~\ref{lemma:xdispersion}. Noting that $\hat{\rho}_A = 1 - \frac{\min(\hat{\phi}_i) \min(a_{ij}^+ \hat{p}_{ij})^2}{\max(\hat{\phi}_i)^2 D(\mathcal{G}) K(\mathcal{G})}$ from Lemma~\ref{lemma:expected_contraction}-\ref{lemma:expected_contraction:A}, and that we have the bounds $\frac{\min(a_{ij}^+ \hat{p}_{ij})^m}{m} \le \min(\hat{\phi}_i) \le \frac1m \le \max(\hat{\phi}_i) < 1$ for the eigenvector entries (see \cite{nguyen2025distributed}), we need to solve for
\begin{equation}
    \begin{gathered}
        \max(\hat{p}_{ij} (1 - \hat{p}_{ij})) \le \frac{\frac{\min(\hat{\phi}_i) \min(a_{ij}^+ \hat{p}_{ij})^2}{\max(\hat{\phi}_i)^2 D(\mathcal{G}) K(\mathcal{G})}}{16 (m - 1) \max(a_{ij}^+)^2} \frac{\min(\hat{\phi}_i)}{\max(\hat{\phi}_i)}
        \\
        \Rightarrow \begin{cases}
            \min(\hat{p}_{ij})^2 \geq \frac{16 (m - 1) \max(a_{ij}^+)^2 D(\mathcal{G}) K(\mathcal{G})}{m \min(a_{ij}^+)^2} (1 - \min(\hat{p}_{ij})) & \text{if $\bs{\hat{A}}$ admits a doubly-stochastic matrix},
            \\
            \min(\hat{p}_{ij})^{2(m+1)} \geq \frac{16 m^2 (m - 1) \max(a_{ij}^+)^2 D(\mathcal{G}) K(\mathcal{G})}{\min(a_{ij}^+)^{2(m+1)}} (1 - \min(\hat{p}_{ij})) & \text{otherwise}.
        \end{cases}
	\end{gathered}
\end{equation}
Note that we separated the first case for a doubly-stochastic $\bs{\hat{A}}$ to discuss the behavior in edge cases, however, in general we deal with the second case. Solving for a range of unknown values $x = \hat{p}_{ij}$ and $\hat{\tau}_A = \frac{16 m^2 (m - 1) \max(a_{ij}^+)^2 D(\mathcal{G}) K(\mathcal{G})}{\min(a_{ij}^+)^{2(m+1)}}$, we write the equation $f_A(x) = x^{2(m+1)} + \hat{\tau}_A x - \hat{\tau}_A$. Noting that $f_A(0) = - \hat{\tau}_A < 0$ and $f_A(1) = 1 > 0$, there exists a root $\hat{r}_A \in (0, 1)$ such that $f_A(\hat{r}_A) = 0$. Since $f_A(x)$ is an increasing function in $x$ for $x > 0$, which is the interval we are interested in analyzing, the solution to $f_A(x) \geq 0$ will be $x \in [\hat{r}_A, 1]$.

\ref{lemma:sporadic_contraction:B} We start by expanding the norm as
\begin{equation} \label{eqn:spod_contract_B_halfway}
    \begin{aligned}
        \left\| \bs{\Lambda}_{\bs{\hat{\pi}}}^{-1} \bs{\hat{B}}^{(k)} \bs{Y}^{(k)} - m \bs{1} \bs{\bar{y}}^{(k)} \right\|_{\bs{\hat{\pi}}}^2 & = \left\| \bs{\Lambda}_{\bs{\hat{\pi}}}^{1/2} \left( \bs{\Lambda}_{\bs{\hat{\pi}}}^{-1} \bs{\hat{B}}^{(k)} \bs{Y}^{(k)} - \bs{1} \bs{1}^T \bs{Y}^{(k)} \right) \right\|^2 = \left\| \bs{\Lambda}_{\bs{\hat{\pi}}}^{-1/2} \left( \bs{\hat{B}}^{(k)} \bs{Y}^{(k)} - \bs{\hat{\pi}} \bs{1}^T \bs{Y}^{(k)} \right) \right\|^2
        \\
        & = \left\| \bs{\Lambda}_{\bs{\hat{\pi}}}^{-1/2} \left( \bs{\hat{B}}^{(k)} \bs{Y}^{(k)} - \bs{\hat{\pi}} \bs{1}^T \bs{\hat{B}}^{(k)} \bs{Y}^{(k)} \right) \right\|^2 = \left\| \bs{\Lambda}_{\bs{\hat{\pi}}}^{-1/2} \left( \bs{I} - \bs{\hat{\pi}} \bs{1}^T \right) \bs{\hat{B}}^{(k)} \bs{Y}^{(k)} \right\|^2
        \\
        & = \sum_{d=1}^n{\left\| \bs{\Lambda}_{\bs{\hat{\pi}}}^{-1/2} \left( \bs{I} - \bs{\hat{\pi}} \bs{1}^T \right) \bs{\hat{B}}^{(k)} [\bs{Y}^{(k)}]_{[:,d]} \right\|^2}
        \\
        & = \sum_{d=1}^n{[\bs{Y}^{(k)}]_{[:,d]}^T (\bs{\hat{B}}^{(k)})^T \left( \bs{I} - \bs{1} \bs{\hat{\pi}}^T \right) \bs{\Lambda}_{\bs{\hat{\pi}}}^{-1} \left( \bs{I} - \bs{\hat{\pi}} \bs{1}^T \right) \bs{\hat{B}}^{(k)} [\bs{Y}^{(k)}]_{[:,d]}}
        \\
        & = \sum_{d=1}^n{[\bs{Y}^{(k)}]_{[:,d]}^T (\bs{\hat{B}}^{(k)})^T \left( \bs{\Lambda}_{\bs{\hat{\pi}}}^{-1} - \bs{1} \bs{1}^T \right) \bs{\hat{B}}^{(k)} [\bs{Y}^{(k)}]_{[:,d]}}
        \\
        & = \sum_{d=1}^n{[\bs{Y}^{(k)}]_{[:,d]}^T \left( (\bs{\hat{B}}^{(k)})^T \bs{\Lambda}_{\bs{\hat{\pi}}}^{-1} \bs{\hat{B}}^{(k)} - \bs{1} \bs{1}^T \right) [\bs{Y}^{(k)}]_{[:,d]}}.
    \end{aligned}
\end{equation}
Noting that when taking the expected value of the above expression we can push the expectation operation inside to $\mathbb{E}[(\bs{\hat{B}}^{(k)})^T \bs{\Lambda}_{\bs{\hat{\pi}}}^{-1} \bs{\hat{B}}^{(k)} - \bs{1} \bs{1}^T]$, we first determine the values of this product. We have
\begin{equation}
	\begin{aligned}
		& (\bs{\hat{B}}^{(k)})^T \bs{\Lambda}_{\bs{\hat{\pi}}}^{-1} \bs{\hat{B}}^{(k)} = \left[ \sum_{\ell=1}^m{\hat{\pi}_\ell^{-1} \hat{b}_{\ell i}^{(k)} \hat{b}_{\ell j}^{(k)}} \right]_{1 \le i,j \le m}
		\\
		& = \begin{cases}
			\sum_{\substack{\ell=1 \\ \ell \neq i}}^m{\hat{\pi}_\ell^{-1} b_{\ell i}^2 \hat{v}_{\ell i}^{(k)}} + \hat{\pi}_i^{-1} \left( 1 - \sum_{\substack{\ell=1 \\ \ell \neq i}}^m{b_{\ell i} \hat{v}_{\ell i}^{(k)}} \right)^2 & i = j,
			\\
			\sum_{\substack{\ell=1 \\ \ell \neq i \\ \ell \neq j}}^m{\hat{\pi}_\ell^{-1} b_{\ell i} b_{\ell j} \hat{v}_{\ell i}^{(k)} \hat{v}_{\ell j}^{(k)}} + \hat{\pi}_i^{-1} \left( 1 - \sum_{\substack{\ell=1 \\ \ell \neq i}}^m{b_{\ell i} \hat{v}_{\ell i}^{(k)}} \right) b_{ij} \hat{v}_{ij}^{(k)} + \hat{\pi}_j^{-1} b_{ji} \hat{v}_{ji}^{(k)} \left( 1 - \sum_{\substack{\ell'=1 \\ \ell' \neq j}}^m{b_{\ell' j} \hat{v}_{\ell' j}^{(k)}} \right) & i \neq j
		\end{cases}
		\\
		& = \begin{cases}
			\sum_{\substack{\ell=1 \\ \ell \neq i}}^m{\hat{\pi}_\ell^{-1} b_{\ell i}^2 \hat{v}_{\ell i}^{(k)}} + \hat{\pi}_i^{-1} \bigg( 1 - 2 \sum_{\substack{\ell=1 \\ \ell \neq i}}^m{b_{\ell i} \hat{v}_{\ell i}^{(k)}} + \sum_{\substack{\ell=1 \\ \ell \neq i}}^m{\sum_{\substack{\ell'=1 \\ \ell' \neq i \\ \ell' \neq \ell}}^m{b_{\ell i} b_{\ell' i} \hat{v}_{\ell i}^{(k)} \hat{v}_{\ell' i}^{(k)}}} + \sum_{\substack{\ell=1 \\ \ell \neq i}}^m{b_{\ell i}^2 \hat{v}_{\ell i}^{(k)}} \bigg) & i = j,
			\\
			\sum_{\substack{\ell=1 \\ \ell \neq i \\ \ell \neq j}}^m{\hat{\pi}_\ell^{-1} b_{\ell i} b_{\ell j} \hat{v}_{\ell i}^{(k)} \hat{v}_{\ell j}^{(k)}} + \hat{\pi}_i^{-1} \left( 1 - \sum_{\substack{\ell=1 \\ \ell \neq i}}^m{b_{\ell i} \hat{v}_{\ell i}^{(k)}} \right) b_{ij} \hat{v}_{ij}^{(k)} + \hat{\pi}_j^{-1} b_{ji} \hat{v}_{ji}^{(k)} \left( 1 - \sum_{\substack{\ell'=1 \\ \ell' \neq j}}^m{b_{\ell' j} \hat{v}_{\ell' j}^{(k)}} \right) & i \neq j,
		\end{cases}
	\end{aligned}
\end{equation}
where we considered the terms that are statistically dependent on each other together. Now, taking the expected value of this product and using Assumption~\ref{assump:bernoulli}, we get
\begin{equation} \label{eqn:spod_contract_B0}
	\begin{aligned}
		& \mathbb{E}\left[ (\bs{\hat{B}}^{(k)})^T \bs{\Lambda}_{\bs{\hat{\pi}}}^{-1} \bs{\hat{B}}^{(k)} \right]
		\\
		& = \begin{cases}
			\sum_{\substack{\ell=1 \\ \ell \neq i}}^m{\hat{\pi}_\ell^{-1} b_{\ell i}^2 \hat{p}_{\ell i}} + \hat{\pi}_i^{-1} \bigg( 1 - 2 \sum_{\substack{\ell=1 \\ \ell \neq i}}^m{b_{\ell i} \hat{p}_{\ell i}} + \sum_{\substack{\ell=1 \\ \ell \neq i}}^m{\sum_{\substack{\ell'=1 \\ \ell' \neq i \\ \ell' \neq \ell}}^m{b_{\ell i} b_{\ell' i} \hat{p}_{\ell i} \hat{p}_{\ell' i}}} + \sum_{\substack{\ell=1 \\ \ell \neq i}}^m{b_{\ell i}^2 \hat{p}_{\ell i}} \bigg) & i = j,
			\\
			\sum_{\substack{\ell=1 \\ \ell \neq i \\ \ell \neq j}}^m{\hat{\pi}_\ell^{-1} b_{\ell i} b_{\ell j} \hat{p}_{\ell i} \hat{p}_{\ell j}} + \hat{\pi}_i^{-1} \left( 1 - \sum_{\substack{\ell=1 \\ \ell \neq i}}^m{b_{\ell i} \hat{p}_{\ell i}} \right) b_{ij} \hat{p}_{ij} + \hat{\pi}_j^{-1} b_{ji} \hat{p}_{ji} \left( 1 - \sum_{\substack{\ell'=1 \\ \ell' \neq j}}^m{b_{\ell' j} \hat{p}_{\ell' j}} \right) & i \neq j
		\end{cases}
		\\
		& = \bs{\hat{B}}^T \bs{\Lambda}_{\bs{\hat{\pi}}}^{-1} \bs{\hat{B}} + \begin{cases}
			\sum_{\substack{\ell=1 \\ \ell \neq i}}^m{(\hat{\pi}_\ell^{-1} + \hat{\pi}_i^{-1}) b_{\ell i}^2 \hat{p}_{\ell i} (1 - \hat{p}_{\ell i})} & i = j,
			\\
			0 & i \neq j
		\end{cases}
        \\
        & = \bs{\hat{B}}^T \bs{\Lambda}_{\bs{\hat{\pi}}}^{-1} \bs{\hat{B}} + \begin{cases}
			\frac{m - 1}{\hat{\pi}_i} \overline{( \hat{\pi}_i \hat{\pi}^{-1} + 1 ) b_{:i}^2 \hat{p}_{:i} (1 - \hat{p}_{:i})} & i = j,
			\\
			0 & i \neq j
		\end{cases}
        = \bs{\hat{B}}^T \bs{\Lambda}_{\bs{\hat{\pi}}}^{-1} \bs{\hat{B}} + \bs{\Lambda}_{\bs{\hat{\pi}}}^{-1/2} \bs{\Lambda}_{\bs{\hat{b}}_0} \bs{\Lambda}_{\bs{\hat{\pi}}}^{-1/2}.
	\end{aligned}
\end{equation}
Note that the diagonal matrix $\bs{\Lambda}_{\bs{\hat{b}}_0}$ will be an all-zeroes matrix if $\hat{p}_{ij} = 1$ for all links $(i, j) \in \mathcal{E}$. Plugging the result in Eq.~\eqref{eqn:spod_contract_B0} back into Eq.~\eqref{eqn:spod_contract_B_halfway}, we get
\begin{equation} \label{eqn:spod_contract_B_almost}
	\begin{gathered}
		\begin{aligned}
		    \mathbb{E}\left[ \left\| \bs{\Lambda}_{\bs{\hat{\pi}}}^{-1} \bs{\hat{B}}^{(k)} \bs{Y}^{(k)} - m \bs{1} \bs{\bar{y}}^{(k)} \right\|_{\bs{\hat{\pi}}}^2 \right] & \le \hat{\rho}_B \mathbb{E}\left[ \left\| \bs{\Lambda}_{\bs{\hat{\pi}}}^{-1} \bs{Y}^{(k)} - m \bs{1} \bs{\bar{y}}^{(k)} \right\|_{\bs{\hat{\pi}}}^2 \right] + \rho(\bs{\Lambda}_{\bs{\hat{b}}_0}) \mathbb{E}\left[ \left\| \bs{Y}^{(k)} \right\|_{\bs{\hat{\pi}}^{-1}}^2 \right]
            \\
            & \le \left( \hat{\rho}_B + \rho(\bs{\Lambda}_{\bs{\hat{b}}_0}) \right) \mathbb{E}\left[ \left\| \bs{\Lambda}_{\bs{\hat{\pi}}}^{-1} \bs{Y}^{(k)} - m \bs{1} \bs{\bar{y}}^{(k)} \right\|_{\bs{\hat{\pi}}}^2 \right] + m^2 \rho(\bs{\Lambda}_{\bs{\hat{b}}_0}) \mathbb{E}\left[ \left\| \bs{\bar{y}}^{(k)} \right\|^2 \right],
		\end{aligned}
	\end{gathered}
\end{equation}
in which $\rho(\bs{\Lambda}_{\bs{\hat{b}}_0}) \le (m - 1) \max( \overline{( \hat{\pi}_i \hat{\pi}^{-1} + 1 ) b_{:i}^2 \hat{p}_{:i} (1 - \hat{p}_{:i})} )$. To ensure a contractive property for the upper bound in Eq.~\eqref{eqn:spod_contract_B_almost}, we need to ensure that $\rho(\bs{\Lambda}_{\bs{\hat{b}}_0}) \le \frac{1 - \hat{\rho}_B}2$. Noting that $\hat{\rho}_B = 1 - \frac{\min(\hat{\pi}_i)^2 \min((b_{ij} \hat{p}_{ij})^+)^2}{\max(\hat{\pi}_i)^3 D(\mathcal{G}) K(\mathcal{G})}$ from Lemma~\ref{lemma:expected_contraction}-\ref{lemma:expected_contraction:B}, and that we have the bounds $\frac{\min(a_{ij}^+ \hat{p}_{ij})^m}{m} \le \min(\hat{\pi}_i) \le \frac1m \le \max(\hat{\pi}_i) < 1$ for the eigenvector entries (see \cite{nguyen2025distributed}), we need to solve for
\begin{equation}
    \begin{gathered}
        \max(\hat{p}_{ij} (1 - \hat{p}_{ij})) \le \frac{\frac{\min(\hat{\pi}_i)^2 \min((b_{ij} \hat{p}_{ij})^+)^2}{\max(\hat{\pi}_i)^3 D(\mathcal{G}) K(\mathcal{G})}}{4 (m - 1) \max(b_{ij}^+)^2} \frac{\min(\hat{\pi}_i)}{\max(\hat{\pi}_i)}
        \\
        \Rightarrow \begin{cases}
            \min(\hat{p}_{ij})^2 \geq \frac{4 (m - 1) \max(b_{ij}^+)^2 D(\mathcal{G}) K(\mathcal{G})}{m \min(b_{ij}^+)^2} (1 - \min(\hat{p}_{ij})) & \text{if $\bs{\hat{B}}$ admits a doubly-stochastic matrix},
            \\
            \min(\hat{p}_{ij})^{2(m+1)} \geq \frac{4 m^3 (m - 1) \max(b_{ij}^+)^2 D(\mathcal{G}) K(\mathcal{G})}{\min(b_{ij}^+)^{3m+2}} (1 - \min(\hat{p}_{ij})) & \text{otherwise}.
        \end{cases}
    \end{gathered}
\end{equation}
Note that we separated the first case for a doubly-stochastic $\bs{\hat{B}}$ to discuss the behavior in edge cases, however, in general we deal with the second case. Solving for a range of unknown values $x = \hat{p}_{ij}$ and $\hat{\tau}_B = \frac{4 m^3 (m - 1) \max(b_{ij}^+)^2 D(\mathcal{G}) K(\mathcal{G})}{\min(b_{ij}^+)^{3m+2}}$, we write the equation $f_B(x) = x^{3m+2} + \hat{\tau}_B x - \hat{\tau}_B$. Noting that $f_B(0) = - \hat{\tau}_B < 0$ and $f_B(1) = 1 > 0$, there exists a root $\hat{r}_B \in (0, 1)$ such that $f_B(\hat{r}_B) = 0$. Since $f(x)$ is an increasing function in $x$ for $x > 0$, which is the interval we are interested in analyzing, the solution to $f_B(x) \geq 0$ will be $x \in [\hat{r}_B, 1]$.

\subsection{Proof of Lemma~\ref{lemma:trackingeffect}} \label{app:lemma:trackingeffect}
\ref{lemma:trackingeffect:sporadic_modeldiv} We have
\begin{equation}
    \begin{aligned}
        & \left\| \left( \bs{\hat{A}}^{(k)} - \bs{\hat{A}} \right) \bs{X}^{(k)} \right\|_{\bs{\hat{\phi}}}^2 = \left\| \bs{\hat{A}}^{(k)} \bs{X}^{(k)} \mp \bs{1} \bs{\bar{x}}_{\bs{\hat{\phi}}}^{(k)} - \bs{\hat{A}} \bs{X}^{(k)} \right\|_{\bs{\hat{\phi}}}^2
        \\
        & = \left\| \bs{\hat{A}}^{(k)} \bs{X}^{(k)} - \bs{1} \bs{\bar{x}}_{\bs{\hat{\phi}}}^{(k)} \right\|_{\bs{\hat{\phi}}}^2 + \left\| \bs{\hat{A}} \bs{X}^{(k)} - \bs{1} \bs{\bar{x}}_{\bs{\hat{\phi}}}^{(k)} \right\|_{\bs{\hat{\phi}}}^2 - 2 \left\langle \bs{\hat{A}}^{(k)} \bs{X}^{(k)} - \bs{1} \bs{\bar{x}}_{\bs{\hat{\phi}}}^{(k)}, \bs{\hat{A}} \bs{X}^{(k)} - \bs{1} \bs{\bar{x}}_{\bs{\hat{\phi}}}^{(k)} \right\rangle_{\bs{\hat{\phi}}},
    \end{aligned}
\end{equation}
where we separated the inner-product term as its expected value will be zero according to Assumption~\ref{assump:bernoulli} (see Eq.~\eqref{eqn:expected_mat}). Next, we take the expected value of the result and employ Assumption~\ref{assump:bernoulli} to write
\begin{equation}
    \begin{aligned}
        \mathbb{E}\left[ \left\| \left( \bs{\hat{A}}^{(k)} - \bs{\hat{A}} \right) \bs{X}^{(k)} \right\|_{\bs{\hat{\phi}}}^2 \right] & = \mathbb{E}\left[ \left\| \bs{\hat{A}}^{(k)} \bs{X}^{(k)} - \bs{1} \bs{\bar{x}}_{\bs{\hat{\phi}}}^{(k)} \right\|_{\bs{\hat{\phi}}}^2 \right] - \mathbb{E}\left[ \left\| \bs{\hat{A}} \bs{X}^{(k)} - \bs{1} \bs{\bar{x}}_{\bs{\hat{\phi}}}^{(k)} \right\|_{\bs{\hat{\phi}}}^2 \right]
        \\
        & \le \rho(\bs{\hat{A}}_0) \mathbb{E}\left[ \left\| \bs{X}^{(k)} - \bs{1} \bs{\bar{x}}_{\bs{\hat{\phi}}}^{(k)} \right\|_{\bs{\hat{\phi}}}^2 \right].
    \end{aligned}
\end{equation}

\ref{lemma:trackingeffect:modelonestep} We use Eq.~\eqref{eqn:averages} to write
\begin{equation}
    \begin{aligned}
        \left\| \bs{\bar{x}}_{\bs{\hat{\phi}}}^{(k+1)} - \bs{\bar{x}}_{\bs{\hat{\phi}}}^{(k)} \right\|^2 & = \left\| \bs{\hat{\phi}}^T \bs{\hat{A}}^{(k)} \bs{X}^{(k)} - \bs{\hat{\phi}}^T \bs{\Lambda}_{\bs{\eta}}^{(k)} \bs{\hat{B}}^{(k)} \bs{Y}^{(k)} - \bs{\bar{x}}_{\bs{\hat{\phi}}}^{(k)} \right\|^2
        \\
        & \le 2 \left\| \bs{\hat{\phi}}^T \left( \bs{\hat{A}}^{(k)} - \bs{\hat{A}} \right) \bs{X}^{(k)} \right\|^2 + 2 \left\| \bs{\hat{\phi}}^T \bs{\Lambda}_{\bs{\eta}}^{(k)} \bs{\hat{B}}^{(k)} \bs{Y}^{(k)} \right\|^2
        \\
        & = 2 \left\| \sum_{i=1}^m{\hat{\phi}_i \sum_{j=1}^m{(\hat{a}_{ij}^{(k)} - \hat{a}_{ij}) \bs{x}_j^{(k)}}} \right\|^2 + 2 \left\| \sum_{i=1}^m{\hat{\phi}_i \eta_i^{(k)} \sum_{j=1}^m{\hat{b}_{ij}^{(k)} \bs{y}_j^{(k)}}} \right\|^2
        \\
        & \le 2 \sum_{i=1}^m{\hat{\phi}_i \left\| \sum_{j=1}^m{(\hat{a}_{ij}^{(k)} - \hat{a}_{ij}) \bs{x}_j^{(k)}} \right\|^2} + 2 \sum_{i=1}^m{\hat{\phi}_i \left\| \eta_i^{(k)} \sum_{j=1}^m{\hat{b}_{ij}^{(k)} \bs{y}_j^{(k)}} \right\|^2}
    \end{aligned}
\end{equation}
\begin{equation*}
    \begin{aligned}
        & = 2 \left\| \left( \bs{\hat{A}}^{(k)} - \bs{\hat{A}} \right) \bs{X}^{(k)} \right\|_{\bs{\hat{\phi}}}^2 + 2 \left\| \bs{\Lambda}_{\bs{\eta}}^{(k)} \bs{\hat{B}}^{(k)} \bs{Y}^{(k)} \right\|_{\bs{\hat{\phi}}}^2,
    \end{aligned}
\end{equation*}
where we used Young's inequality and Jensen's inequality, respectively. Next, we take the expected value of the result and employ Lemmas~\ref{lemma:descenteffect}-\ref{lemma:descenteffect:yconsphi} and \ref{lemma:trackingeffect}-\ref{lemma:trackingeffect:sporadic_modeldiv} to conclude the proof.

\ref{lemma:trackingeffect:sporadic_sgddiv}
We have
\begin{equation} \label{eqn:graddiff_half}
	\begin{aligned}
		\left\| \bs{\Lambda}_{\bs{v}}^{(k+1)} \bs{G}^{(k+1)} - \bs{\Lambda}_{\bs{v}}^{(k)} \bs{G}^{(k)} \right\|^2 & = \left\| \bs{\Lambda}_{\bs{v}}^{(k+1)} \bs{G}^{(k+1)} \mp \bs{\Lambda}_{\bs{v}}^{(k+1)} \nabla{\bs{F}}(\bs{X}^{(k+1)}) \mp \bs{\Lambda}_{\bs{v}}^{(k)} \nabla{\bs{F}}(\bs{X}^{(k)}) - \bs{\Lambda}_{\bs{v}}^{(k)} \bs{G}^{(k)} \right\|^2
		\\
		& \begin{aligned}
			\,\, = \Big\| \bs{\Lambda}_{\bs{v}}^{(k+1)} & \left( \bs{G}^{(k+1)} - \nabla{\bs{F}}(\bs{X}^{(k+1)}) \right) \Big\|^2 + \left\| \bs{\Lambda}_{\bs{v}}^{(k)} \left( \nabla{\bs{F}}(\bs{X}^{(k)}) - \bs{G}^{(k)} \right) \right\|^2
            \\
            & + \left\| \bs{\Lambda}_{\bs{v}}^{(k+1)} \nabla{\bs{F}}(\bs{X}^{(k+1)}) - \bs{\Lambda}_{\bs{v}}^{(k)} \nabla{\bs{F}}(\bs{X}^{(k)}) \right\|^2
			\\
			& \begin{aligned}
			    \,\, + 2 \bigg\langle & \bs{\Lambda}_{\bs{v}}^{(k+1)} \left( \bs{G}^{(k+1)} - \nabla{\bs{F}}(\bs{X}^{(k+1)}) \right) + \bs{\Lambda}_{\bs{v}}^{(k)} \left( \nabla{\bs{F}}(\bs{X}^{(k)}) - \bs{G}^{(k)} \right),
                \\
                & \bs{\Lambda}_{\bs{v}}^{(k+1)} \nabla{\bs{F}}(\bs{X}^{(k+1)}) - \bs{\Lambda}_{\bs{v}}^{(k)} \nabla{\bs{F}}(\bs{X}^{(k)}) \bigg\rangle.
			\end{aligned}
		\end{aligned}
	\end{aligned}
\end{equation}
where we separated the inner product term because it will be equal to $0$ when taking its expected value, since $\bs{g}_i^{(k)}$ is an unbiased estimate of $\nabla{F}_i(\bs{x}_i^{(k)})$ according to the Assumption~\ref{assump:sgd}. We further expand the bound as
\begin{equation} \label{eqn:graddiff_vv}
	\begin{aligned}
		& \left\| \bs{\Lambda}_{\bs{v}}^{(k+1)} \nabla{\bs{F}}(\bs{X}^{(k+1)}) - \bs{\Lambda}_{\bs{v}}^{(k)} \nabla{\bs{F}}(\bs{X}^{(k)}) \right\|^2
		\\
		& \begin{aligned}
		    \,\, = \bigg\| \bs{\Lambda}_{\bs{v}}^{(k+1)} \nabla{\bs{F}}(\bs{X}^{(k+1)}) & \mp \bs{\Lambda}_{\bs{v}}^{(k+1)} \nabla{\bs{F}}(\bs{1} \bs{\bar{x}}_{\bs{\hat{\phi}}}^{(k+1)}) \mp \bs{\Lambda}_{\bs{v}}^{(k+1)} \bs{\Lambda}_{\bs{v}}^{(k)} \nabla{\bs{F}}(\bs{1} \bs{\bar{x}}_{\bs{\hat{\phi}}}^{(k+1)}) \mp \bs{\Lambda}_{\bs{v}}^{(k+1)} \bs{\Lambda}_{\bs{v}}^{(k)} \nabla{\bs{F}}(\bs{1} \bs{\bar{x}}_{\bs{\hat{\phi}}}^{(k)})
            \\
            & \mp \bs{\Lambda}_{\bs{v}}^{(k)} \nabla{\bs{F}}(\bs{1} \bs{\bar{x}}_{\bs{\hat{\phi}}}^{(k)}) - \bs{\Lambda}_{\bs{v}}^{(k)} \nabla{\bs{F}}(\bs{X}^{(k)}) \bigg\|^2
		\end{aligned}
		\\
		& \begin{aligned}
		    \,\, \le 5 \bigg( & \left\| \bs{\Lambda}_{\bs{v}}^{(k+1)} \left( \nabla{\bs{F}}(\bs{X}^{(k+1)}) - \nabla{\bs{F}}(\bs{1} \bs{\bar{x}}_{\bs{\hat{\phi}}}^{(k+1)}) \right) \right\|^2 + \left\| \bs{\Lambda}_{\bs{v}}^{(k+1)} \left( \bs{I} - \bs{\Lambda}_{\bs{v}}^{(k)} \right) \nabla{\bs{F}}(\bs{1} \bs{\bar{x}}_{\bs{\hat{\phi}}}^{(k+1)}) \right\|^2
            \\
            & + \left\| \bs{\Lambda}_{\bs{v}}^{(k+1)} \bs{\Lambda}_{\bs{v}}^{(k)} \left( \nabla{\bs{F}}(\bs{1} \bs{\bar{x}}_{\bs{\hat{\phi}}}^{(k+1)}) - \nabla{\bs{F}}(\bs{1} \bs{\bar{x}}_{\bs{\hat{\phi}}}^{(k)}) \right) \right\|^2 + \left\| \left( \bs{\Lambda}_{\bs{v}}^{(k+1)} - \bs{I} \right) \bs{\Lambda}_{\bs{v}}^{(k)} \nabla{\bs{F}}(\bs{1} \bs{\bar{x}}_{\bs{\hat{\phi}}}^{(k)}) \right\|^2
            \\
            & + \left\| \bs{\Lambda}_{\bs{v}}^{(k)} \left( \nabla{\bs{F}}(\bs{1} \bs{\bar{x}}_{\bs{\hat{\phi}}}^{(k)}) - \nabla{\bs{F}}(\bs{X}^{(k)}) \right) \right\|^2 \bigg)
		\end{aligned}
	\end{aligned}
\end{equation}
\begin{equation*}
    \begin{aligned}
        & \begin{aligned}
		    \,\, = 5 \sum_{i=1}^m \bigg( & \left\| \nabla{F}_i(\bs{x}_i^{(k+1)}) - \nabla{F}_i(\bs{\bar{x}}_{\bs{\hat{\phi}}}^{(k+1)}) \right\|^2 v_i^{(k+1)} + \left\| \nabla{F}_i(\bs{\bar{x}}_{\bs{\hat{\phi}}}^{(k+1)}) \right\|^2 v_i^{(k+1)} (1 - v_i^{(k)})
            \\
            & + \left\| \nabla{F}_i(\bs{\bar{x}}_{\bs{\hat{\phi}}}^{(k+1)}) - \nabla{F}_i(\bs{\bar{x}}_{\bs{\hat{\phi}}}^{(k)}) \right\|^2 v_i^{(k+1)} v_i^{(k)} + \left\| \nabla{F}_i(\bs{\bar{x}}_{\bs{\hat{\phi}}}^{(k)}) \right\|^2 (1 - v_i^{(k+1)}) v_i^{(k)}
            \\
            & + \left\| \nabla{F}_i(\bs{\bar{x}}_{\bs{\hat{\phi}}}^{(k)}) - \nabla{F}_i(\bs{x}_i^{(k)}) \right\|^2  v_i^{(k)} \bigg),
		\end{aligned}
        \\
        & \begin{aligned}
            \,\, \le 5 \sum_{i=1}^m \bigg( & L_i^2 \left\| \bs{x}_i^{(k+1)} - \bs{\bar{x}}_{\bs{\hat{\phi}}}^{(k+1)} \right\|^2 v_i^{(k+1)} + \left( \delta_{0,i}^2 + \delta_{1,i}^2 \left\| \nabla{F}(\bs{\bar{x}}_{\bs{\hat{\phi}}}^{(k+1)}) \right\|^2 \right) v_i^{(k+1)} (1 - v_i^{(k)})
            \\
            & + L_i^2 \left\| \bs{\bar{x}}_{\bs{\hat{\phi}}}^{(k+1)} - \bs{\bar{x}}_{\bs{\hat{\phi}}}^{(k)} \right\|^2 v_i^{(k+1)} v_i^{(k)} + \left( \delta_{0,i}^2 + \delta_{1,i}^2 \left\| \nabla{F}(\bs{\bar{x}}_{\bs{\hat{\phi}}}^{(k)}) \right\|^2 \right) (1 - v_i^{(k+1)}) v_i^{(k)}
            \\
            & + L_i^2 \left\| \bs{\bar{x}}_{\bs{\hat{\phi}}}^{(k)} - \bs{x}_i^{(k)} \right\|^2 v_i^{(k)} \bigg)
        \end{aligned}
    \end{aligned}
\end{equation*}
in which we first used Young's inequality, and then invoked Assumptions~\ref{assump:lipschitz} and \ref{assump:graddiv}. Plugging Eq.~\eqref{eqn:graddiff_vv} back into Eq.~\eqref{eqn:graddiff_half} and taking the expected value of the result, we get
\begin{equation}
    \begin{aligned}
		& \mathbb{E}\left[ \left\| \bs{\Lambda}_{\bs{v}}^{(k+1)} \bs{G}^{(k+1)} - \bs{\Lambda}_{\bs{v}}^{(k)} \bs{G}^{(k)} \right\|^2 \right]
        \\
        & \begin{aligned}
            \le \sum_{i=1}^m p_i \bigg( & \left\| \bs{g}_i^{(k+1)} - \nabla{F}_i(\bs{x}_i^{(k+1)}) \right\|^2 + \left\| \nabla{F}_i(\bs{x}_i^{(k)}) - \bs{g}_i^{(k)} \right\|^2 + 5 \bigg( L_i^2 \left\| \bs{x}_i^{(k+1)} - \bs{\bar{x}}_{\bs{\hat{\phi}}}^{(k+1)} \right\|^2
            \\
            & + \left( \delta_{0,i}^2 + \delta_{1,i}^2 \left\| \nabla{F}(\bs{\bar{x}}_{\bs{\hat{\phi}}}^{(k+1)}) \right\|^2 \right) (1 - p_i) + L_i^2 \left\| \bs{\bar{x}}_{\bs{\hat{\phi}}}^{(k+1)} - \bs{\bar{x}}_{\bs{\hat{\phi}}}^{(k)} \right\|^2 p_i + \left( \delta_{0,i}^2 + \delta_{1,i}^2 \left\| \nabla{F}(\bs{\bar{x}}_{\bs{\hat{\phi}}}^{(k)}) \right\|^2 \right) (1 - p_i)
            \\
            & + L_i^2 \left\| \bs{\bar{x}}_{\bs{\hat{\phi}}}^{(k)} - \bs{x}_i^{(k)} \right\|^2 \bigg) \bigg)
        \end{aligned}
        \\
        & \begin{aligned}
            \le 2 m \overline{\sigma_0^2 p B^{-1} (1 - B/D)} & + 2\bar{L}^2 \max\left( \frac{\sigma_{1,i}^2 p_i (1 - B_i / D_i)}{B_i \hat{\phi}_i} \right) \left( \mathbb{E}\left[ \left\| \bs{X}^{(k+1)} - \bs{1} \bs{\bar{x}}_{\bs{\hat{\phi}}}^{(k+1)} \right\|_{\bs{\hat{\phi}}}^2 \right] + \mathbb{E}\left[ \left\| \bs{X}^{(k)} - \bs{1} \bs{\bar{x}}_{\bs{\hat{\phi}}}^{(k)} \right\|_{\bs{\hat{\phi}}}^2 \right] \right)
            \\
            & + 2m \overline{\sigma_1^2 p B^{-1} (1 - B/D)} \left( \mathbb{E}\left[ \left\| \nabla{F}(\bs{\bar{x}}_{\bs{\hat{\phi}}}^{(k+1)}) \right\|^2 \right] + \mathbb{E}\left[ \left\| \nabla{F}(\bs{\bar{x}}_{\bs{\hat{\phi}}}^{(k)}) \right\|^2 \right] \right)
            \\
            & + 5 \max(p_i L_i^2) \left( \mathbb{E}\left[ \left\| \bs{X}^{(k+1)} - \bs{1} \bs{\bar{x}}_{\bs{\hat{\phi}}}^{(k+1)} \right\|^2 \right] + \mathbb{E}\left[ \left\| \bs{X}^{(k)} - \bs{1} \bs{\bar{x}}_{\bs{\hat{\phi}}}^{(k)} \right\|^2 \right] \right)
            \\
            & + 10m \overline{p (1 - p) \delta_0^2} + 5m \overline{p (1 - p) \delta_1^2} \left( \mathbb{E}\left[ \left\| \nabla{F}(\bs{\bar{x}}_{\bs{\hat{\phi}}}^{(k+1)}) \right\|^2 \right] + \mathbb{E}\left[ \left\| \nabla{F}(\bs{\bar{x}}_{\bs{\hat{\phi}}}^{(k)}) \right\|^2 \right] \right)
            \\
            & + 5m \overline{p^2 L^2} \mathbb{E}\left[ \left\| \bs{\bar{x}}_{\bs{\hat{\phi}}}^{(k+1)} - \bs{\bar{x}}_{\bs{\hat{\phi}}}^{(k)} \right\|^2 \right],
        \end{aligned}
        \\
        & \begin{aligned}
            \le 2m & \overline{p \left( \sigma_0^2 B^{-1} (1 - B/D) + 5 (1 - p) \delta_0^2 \right)}
            \\
            & + \max\left( \frac{p_i}{\hat{\phi}_i} \left( 2 \frac{(1 - B_i / D_i) \sigma_{1,i}^2}{B_i} \bar{L}^2 + 5 L_i^2 \right) \right) \left( \mathbb{E}\left[ \left\| \bs{X}^{(k+1)} - \bs{1} \bs{\bar{x}}_{\bs{\hat{\phi}}}^{(k+1)} \right\|_{\bs{\hat{\phi}}}^2 \right] + \mathbb{E}\left[ \left\| \bs{X}^{(k)} - \bs{1} \bs{\bar{x}}_{\bs{\hat{\phi}}}^{(k)} \right\|_{\bs{\hat{\phi}}}^2 \right] \right)
            \\
            & + m \overline{p \left( 2 \sigma_1^2 B^{-1} (1 - B/D) + 5 (1 - p) \delta_1^2 \right)} \left( \mathbb{E}\left[ \left\| \nabla{F}(\bs{\bar{x}}_{\bs{\hat{\phi}}}^{(k+1)}) \right\|^2 \right] + \mathbb{E}\left[ \left\| \nabla{F}(\bs{\bar{x}}_{\bs{\hat{\phi}}}^{(k)}) \right\|^2 \right] \right)
            \\
            & + 5m \overline{p^2 L^2} \mathbb{E}\left[ \left\| \bs{\bar{x}}_{\bs{\hat{\phi}}}^{(k+1)} - \bs{\bar{x}}_{\bs{\hat{\phi}}}^{(k)} \right\|^2 \right],
        \end{aligned}
	\end{aligned}
\end{equation}
where we employed Assumption~\ref{assump:bernoulli} and Lemma~\ref{lemma:graderror}-\ref{lemma:graderror:sgd}. Noting that $\| \nabla{F}(\bs{\bar{x}}_{\bs{\hat{\phi}}}^{(k+1)}) \|^2 \le 2 \| \nabla{F}(\bs{\bar{x}}_{\bs{\hat{\phi}}}^{(k+1)}) - \nabla{F}(\bs{\bar{x}}_{\bs{\hat{\phi}}}^{(k)}) \|^2 + 2 \| \nabla{F}(\bs{\bar{x}}_{\bs{\hat{\phi}}}^{(k)}) \|^2 \le 2 \bar{L}^2 \| \bs{\bar{x}}_{\bs{\hat{\phi}}}^{(k+1)} - \bs{\bar{x}}_{\bs{\hat{\phi}}}^{(k)} \|^2 + 2 \| \nabla{F}(\bs{\bar{x}}_{\bs{\hat{\phi}}}^{(k)}) \|^2$  concludes the proof.

\subsection{Proof of Lemma~\ref{lemma:gradconsensus}} \label{app:lemma:gradconsensus}
According to Young's inequality, we have
\begin{equation} \label{eqn:fulldiscrepancy}
    \begin{aligned}
        & \left\| \bs{\Lambda}_{\bs{\hat{\pi}}}^{-1} \bs{\hat{B}}^{(k)} \bs{Y}^{(k)} - m \bs{1} \nabla{F}(\bs{\bar{x}}_{\bs{\hat{\phi}}}^{(k)}) \right\|_{\bs{\hat{\pi}}}^2
        \\
        & = \left\| \bs{\Lambda}_{\bs{\hat{\pi}}}^{-1} \bs{\hat{B}}^{(k)} \bs{Y}^{(k)} + m \bs{1} \left( \mp \bs{\bar{y}}^{(k)} \mp \bs{\bar{g}}_{\bs{v}}^{(k)} \mp \overline{\nabla{\bs{F}(\bs{X}^{(k)})}}_{\bs{v}^{(k)}} \mp \overline{\nabla{\bs{F}(\bs{1} \bs{\bar{x}}_{\bs{\hat{\phi}}}^{(k)})}}_{\bs{v}^{(k)}} -\nabla{F}(\bs{\bar{x}}_{\bs{\hat{\phi}}}^{(k)}) \right) \right\|_{\bs{\hat{\pi}}}^2
        \\
        & \begin{aligned}
            \,\, \le 3 & \left\| \bs{\Lambda}_{\bs{\hat{\pi}}}^{-1} \bs{\hat{B}}^{(k)} \bs{Y}^{(k)} - m \bs{1} \bs{\bar{y}}^{(k)} \right\|_{\bs{\hat{\pi}}}^2 + m^2 \left\| \bs{\bar{g}}_{\bs{v}}^{(k)} - \overline{\nabla{\bs{F}(\bs{X}^{(k)})}}_{\bs{v}^{(k)}} \right\|^2 + 3m^2 \left\| \overline{\nabla{\bs{F}(\bs{X}^{(k)})}}_{\bs{v}^{(k)}} - \overline{\nabla{\bs{F}(\bs{1} \bs{\bar{x}}_{\bs{\hat{\phi}}}^{(k)})}}_{\bs{v}^{(k)}} \right\|^2
            \\
            & + 3m^2 \left\| \overline{\nabla{\bs{F}(\bs{1} \bs{\bar{x}}_{\bs{\hat{\phi}}}^{(k)})}}_{\bs{v}^{(k)}} - \nabla{F}(\bs{\bar{x}}_{\bs{\hat{\phi}}}^{(k)}) \right\|^2
            \\
            & + \sum_{i=1}^m{m \hat{\pi}_i \left\langle \bs{\bar{g}}_{\bs{v}}^{(k)} - \overline{\nabla{\bs{F}(\bs{X}^{(k)})}}_{\bs{v}^{(k)}}, \frac1{\hat{\pi}_i} \sum_{j=1}^m{\hat{b}_{ij}^{(k)} \bs{y}_j^{(k)}} - m \bs{\bar{y}}^{(k)} + m \overline{\nabla{\bs{F}(\bs{X}^{(k)})}}_{\bs{v}^{(k)}} - m \nabla{F}(\bs{\bar{x}}_{\bs{\hat{\phi}}}^{(k)}) \right\rangle},
        \end{aligned}
    \end{aligned}
\end{equation}
where we separated the inner product term from the others because it will be equal to $0$ when taking its expected value, since $\bs{g}_i^{(k)}$ is an unbiased estimate of $\nabla{F}_i(\bs{x}_i^{(k)})$ according to the Assumption~\ref{assump:sgd}.
Also, note that $\bs{\bar{y}}^{(k)} = \bs{\bar{g}}_v^{(k)}$ according to Eq.~\eqref{eqn:ybar}.
Taking the expected value of Eq.~\eqref{eqn:fulldiscrepancy} and invoking Lemmas~\ref{lemma:sporadic_contraction}-\ref{lemma:sporadic_contraction:B}, we get
\begin{equation}
    \begin{aligned}
        \mathbb{E} & \left[ \left\| \bs{\Lambda}_{\bs{\hat{\pi}}}^{-1} \bs{\hat{B}}^{(k)} \bs{Y}^{(k)} - m \bs{1} \nabla{F}(\bs{\bar{x}}_{\bs{\hat{\phi}}}^{(k)}) \right\|_{\bs{\hat{\pi}}}^2 \right] \le 3 (\hat{\rho}_B + \hat{\rho}_{0,B}) \mathbb{E}\left[ \left\| \bs{\Lambda}_{\bs{\hat{\pi}}}^{-1} \bs{Y}^{(k)} - m \bs{1} \bs{\bar{y}}^{(k)} \right\|_{\bs{\hat{\pi}}}^2 \right] + 3 m^2 \hat{\rho}_{0,B} \mathbb{E}\left[ \| \bs{\bar{y}}^{(k)} \|^2 \right]
        \\
        & + m^2 \mathbb{E}\left[ \left\| \bs{\bar{g}}_{\bs{v}}^{(k)} - \overline{\nabla{\bs{F}(\bs{X}^{(k)})}}_{\bs{v}^{(k)}} \right\|^2 \right] + 3m^2 \mathbb{E}\left[ \left\| \overline{\nabla{\bs{F}(\bs{X}^{(k)})}}_{\bs{v}^{(k)}} - \overline{\nabla{\bs{F}(\bs{1} \bs{\bar{x}}_{\bs{\hat{\phi}}}^{(k)})}}_{\bs{v}^{(k)}} \right\|^2 \right]
        \\
        & + 3m^2 \mathbb{E}\left[ \left\| \overline{\nabla{\bs{F}(\bs{1} \bs{\bar{x}}_{\bs{\hat{\phi}}}^{(k)})}}_{\bs{v}^{(k)}} - \nabla{F}(\bs{\bar{x}}_{\bs{\hat{\phi}}}^{(k)}) \right\|^2 \right].
    \end{aligned}
\end{equation}
Applying Lemmas~\ref{lemma:graderror} and \ref{lemma:descenteffect}-\ref{lemma:descenteffect:ybar} to the expression above completes the proof. 

\section{Proofs of Main Lemmas}

\subsection{Proof of Lemma~\ref{lemma:xdispersion}} \label{app:lemma:xdispersion}
According to Young's inequality with $\varepsilon > 0$ to be tuned later, we have
\begin{equation}
    \begin{aligned}
        & \left\| \bs{X}^{(k+1)} - \bs{1} \bs{\bar{x}}_{\bs{\hat{\phi}}}^{(k+1)} \right\|_{\bs{\hat{\phi}}}^2 = \left\| \bs{\hat{A}}^{(k)} \bs{X}^{(k)} - \bs{\Lambda}_{\bs{\eta}}^{(k)} \bs{\hat{B}}^{(k)} \bs{Y}^{(k)} - \bs{1} (\bs{\hat{\phi}}^T \bs{\hat{A}}^{(k)} \bs{X}^{(k)} - \bs{\hat{\phi}}^T \bs{\Lambda}_{\bs{\eta}}^{(k)} \bs{\hat{B}}^{(k)} \bs{Y}^{(k)}) \right\|_{\bs{\hat{\phi}}}^2
        \\
        & = \left\| \left( \bs{\hat{A}}^{(k)} \bs{X}^{(k)} - \bs{1} \bs{\bar{x}}_{\bs{\hat{\phi}}}^{(k)} \right) - \bs{1} \left( \bs{\hat{\phi}}^T \bs{\hat{A}}^{(k)} \bs{X}^{(k)} -\bs{\bar{x}}_{\bs{\hat{\phi}}}^{(k)} \right) - \left( \bs{\Lambda}_{\bs{\eta}}^{(k)} \bs{\hat{B}}^{(k)} \bs{Y}^{(k)} - \bs{1} \bs{\hat{\phi}}^T \bs{\Lambda}_{\bs{\eta}}^{(k)} \bs{\hat{B}}^{(k)} \bs{Y}^{(k)} \right) \right\|_{\bs{\hat{\phi}}}^2
        \\
        & \begin{aligned}
            \,\, \le \left( 1 + \varepsilon \right) & \left\| \bs{\hat{A}}^{(k)} \bs{X}^{(k)} - \bs{1} \bs{\bar{x}}_{\bs{\hat{\phi}}}^{(k)} \right\|_{\bs{\hat{\phi}}}^2 + \left\| \bs{\hat{\phi}}^T \bs{\hat{A}}^{(k)} \bs{X}^{(k)} - \bs{\bar{x}}_{\bs{\hat{\phi}}}^{(k)} \right\|^2 + \left( 1 + \varepsilon^{-1} \right) \left\| \bs{\Lambda}_{\bs{\eta}}^{(k)} \bs{\hat{B}}^{(k)} \bs{Y}^{(k)} - \bs{1} \bs{\hat{\phi}}^T \bs{\Lambda}_{\bs{\eta}}^{(k)} \bs{\hat{B}}^{(k)} \bs{Y}^{(k)} \right\|_{\bs{\hat{\phi}}}^2
            \\
            & - 2 \left\langle \bs{\hat{A}}^{(k)} \bs{X}^{(k)} - \bs{1} \bs{\bar{x}}_{\bs{\hat{\phi}}}^{(k)} - \bs{\Lambda}_{\bs{\eta}}^{(k)} \bs{\hat{B}}^{(k)} \bs{Y}^{(k)} - \bs{1} \bs{\hat{\phi}}^T \bs{\Lambda}_{\bs{\eta}}^{(k)} \bs{\hat{B}}^{(k)} \bs{Y}^{(k)} , \bs{1} \left( \bs{\hat{\phi}}^T \bs{\hat{A}}^{(k)} \bs{X}^{(k)} - \bs{\bar{x}}_{\bs{\hat{\phi}}}^{(k)} \right) \right\rangle
        \end{aligned}
        \\
        & \begin{aligned}
            \,\, \le \left( 1 + \varepsilon \right) & \left\| \bs{\hat{A}}^{(k)} \bs{X}^{(k)} - \bs{1} \bs{\bar{x}}_{\bs{\hat{\phi}}}^{(k)} \right\|_{\bs{\hat{\phi}}}^2 + \left\| \left( \bs{\hat{A}}^{(k)} - \bs{\hat{A}} \right) \bs{X}^{(k)} \right\|_{\bs{\hat{\phi}}}^2 + \left( 1 + \varepsilon^{-1} \right) \left\| \bs{\Lambda}_{\bs{\eta}}^{(k)} \bs{\hat{B}}^{(k)} \bs{Y}^{(k)} - \bs{1} \bs{\hat{\phi}}^T \bs{\Lambda}_{\bs{\eta}}^{(k)} \bs{\hat{B}}^{(k)} \bs{Y}^{(k)} \right\|_{\bs{\hat{\phi}}}^2
            \\
            & - 2 \left\langle \bs{\hat{A}}^{(k)} \bs{X}^{(k)} - \bs{1} \bs{\bar{x}}_{\bs{\hat{\phi}}}^{(k)} - \bs{\Lambda}_{\bs{\eta}}^{(k)} \bs{\hat{B}}^{(k)} \bs{Y}^{(k)} - \bs{1} \bs{\hat{\phi}}^T \bs{\Lambda}_{\bs{\eta}}^{(k)} \bs{\hat{B}}^{(k)} \bs{Y}^{(k)} , \bs{1} \bs{\hat{\phi}}^T \left( \bs{\hat{A}}^{(k)} - \bs{\hat{A}} \right) \bs{X}^{(k)} \right\rangle,
        \end{aligned}
    \end{aligned}
\end{equation}
where we separated the inner-product term since it's expected value will be zero according to Assumption~\ref{assump:bernoulli}. Taking the expected value of this expression and employing Lemmas~\ref{lemma:sporadic_contraction}-\ref{lemma:sporadic_contraction:A}, \ref{lemma:trackingeffect}-\ref{lemma:trackingeffect:sporadic_modeldiv} and \ref{lemma:descenteffect}-\ref{lemma:descenteffect:yconsphi}, we get
\begin{equation}
    \begin{aligned}
        \mathbb{E}\left[ \left\| \bs{X}^{(k+1)} - \bs{1} \bs{\bar{x}}_{\bs{\hat{\phi}}}^{(k+1)} \right\|_{\bs{\hat{\phi}}}^2 \right] \le \left( (1 + \varepsilon) (\hat{\rho}_A + \hat{\rho}_{0,A}) + \hat{\rho}_{0,A} \right) & \mathbb{E}\left[ \left\| \bs{X}^{(k)} - \bs{1} \bs{\bar{x}}_{\bs{\hat{\phi}}}^{(k)} \right\|_{\bs{\hat{\phi}}}^2 \right]
        \\
        & + (1 + \varepsilon^{-1}) \kappa_3 \max(\eta_i^{(k)})^2 \mathbb{E}\left[ \left\| \bs{Y}^{(k)} \right\|_{\bs{\hat{\pi}}^{-1}}^2 \right].
    \end{aligned}
\end{equation}
Next, employing Lemma~\ref{lemma:descenteffect}-\ref{lemma:descenteffect:ypiinv}, we get
\begin{equation}
    \begin{aligned}
        \mathbb{E} & \left[ \left\| \bs{X}^{(k+1)} - \bs{1} \bs{\bar{x}}_{\bs{\hat{\phi}}}^{(k+1)} \right\|_{\bs{\hat{\phi}}}^2 \right] \le \left( \hat{\rho}_A + 2 \hat{\rho}_{0,A} + \varepsilon (\hat{\rho}_A + \hat{\rho}_{0,A}) \right) \mathbb{E}\left[ \left\| \bs{X}^{(k)} - \bs{1} \bs{\bar{x}}_{\bs{\hat{\phi}}}^{(k)} \right\|_{\bs{\hat{\phi}}}^2 \right]
        \\
        & + (1 + \varepsilon^{-1}) \kappa_3 \max(\eta_i^{(k)})^2 \mathbb{E}\left[ \left\| \bs{Y}^{(k)} \right\|_{\bs{\hat{\pi}}^{-1}}^2 \right] + (1 + \varepsilon^{-1}) \kappa_3 m^2 \max(\eta_i^{(k)})^2 \mathbb{E}\left[ \left\| \bs{\bar{y}}^{(k)} \right\|^2 \right].
    \end{aligned}
\end{equation}
Using Lemma~\ref{lemma:descenteffect}-\ref{lemma:descenteffect:ybar} on the last expression and letting $\varepsilon = \frac{1 - (\hat{\rho}_A + 2 \hat{\rho}_{0,A})}{2 (\hat{\rho}_A + \hat{\rho}_{0,A})}$ conclude the proof.

\subsection{Proof of Lemma~\ref{lemma:ydispersion}} \label{app:lemma:ydispersion}
According to Young's inequality, we have
\begin{equation}
    \begin{aligned}
        & \left\| \bs{\Lambda}_{\bs{\hat{\pi}}}^{-1} \bs{Y}^{(k+1)} - m \bs{1} \bs{\bar{y}}^{(k+1)} \right\|_{\bs{\hat{\pi}}}^2
        \\
        & = \left\| \bs{\Lambda}_{\bs{\hat{\pi}}}^{-1} \left( \bs{\hat{B}}^{(k)} \bs{Y}^{(k)} + \bs{\Lambda}_{\bs{v}}^{(k+1)} \bs{G}^{(k+1)} - \bs{\Lambda}_{\bs{v}}^{(k)} \bs{G}^{(k)} \right) - m \bs{1} \left( \bs{\bar{y}}^{(k)} + \bs{\bar{g}}_{\bs{v}}^{(k+1)} - \bs{\bar{g}}_{\bs{v}}^{(k)} \right) \right\|_{\bs{\hat{\pi}}}^2
        \\
        & = \left\| \left( \bs{\Lambda}_{\bs{\hat{\pi}}}^{-1} \bs{\hat{B}}^{(k)} \bs{Y}^{(k)} - m \bs{1} \bs{\bar{y}}^{(k)} \right) + \bs{\Lambda}_{\bs{\hat{\pi}}}^{-1} \left( \bs{\Lambda}_{\bs{v}}^{(k+1)} \bs{G}^{(k+1)} - \bs{\Lambda}_{\bs{v}}^{(k)} \bs{G}^{(k)} \right) - m \bs{1} \left( \bs{\bar{g}}_{\bs{v}}^{(k+1)} - \bs{\bar{g}}_{\bs{v}}^{(k)} \right) \right\|_{\bs{\hat{\pi}}}^2
    \end{aligned}
\end{equation}
\begin{equation*}
    \begin{aligned}
        & \begin{aligned}
            \,\, \le \left( 1 + 2 \frac{1 - (\hat{\rho}_B + \hat{\rho}_{0,B})}{4 (\hat{\rho}_B + \hat{\rho}_{0,B})} \right) & \left\| \bs{\Lambda}_{\bs{\hat{\pi}}}^{-1} \bs{\hat{B}}^{(k)} \bs{Y}^{(k)} - m \bs{1} \bs{\bar{y}}^{(k)} \right\|_{\bs{\hat{\pi}}}^2
            \\
            & + \left( 2 + \frac{4 (\hat{\rho}_B + \hat{\rho}_{0,B})}{1 - (\hat{\rho}_B + \hat{\rho}_{0,B})} \right) \left\| \bs{\Lambda}_{\bs{\hat{\pi}}}^{-1} \left( \bs{\Lambda}_{\bs{v}}^{(k+1)} \bs{G}^{(k+1)} - \bs{\Lambda}_{\bs{v}}^{(k)} \bs{G}^{(k)} \right) \right\|_{\bs{\hat{\pi}}}^2
            \\
            & + \left( 2 + \frac{4 (\hat{\rho}_B + \hat{\rho}_{0,B})}{1 - (\hat{\rho}_B + \hat{\rho}_{0,B})} \right) \left\| m \bs{1} \left( \bs{\bar{g}}_{\bs{v}}^{(k+1)} - \bs{\bar{g}}_{\bs{v}}^{(k)} \right) \right\|_{\bs{\hat{\pi}}}^2
        \end{aligned}
        \\
        & \begin{aligned}
            \,\, \le \frac{1 + \hat{\rho}_B + \hat{\rho}_{0,B}}{2 (\hat{\rho}_B + \hat{\rho}_{0,B})} & \left\| \bs{\Lambda}_{\bs{\hat{\pi}}}^{-1} \bs{\hat{B}}^{(k)} \bs{Y}^{(k)} - m \bs{1} \bs{\bar{y}}^{(k)} \right\|_{\bs{\hat{\pi}}}^2
            \\
            & + \frac{2 (1 + \hat{\rho}_B + \hat{\rho}_{0,B})}{1 - (\hat{\rho}_B + \hat{\rho}_{0,B})} \left( \left\| \bs{\Lambda}_{\bs{v}}^{(k+1)} \bs{G}^{(k+1)} - \bs{\Lambda}_{\bs{v}}^{(k)} \bs{G}^{(k)} \right\|^2 + m^2 \left\| \bs{\bar{g}}_{\bs{v}}^{(k+1)} - \bs{\bar{g}}_{\bs{v}}^{(k)} \right\|^2 \right)
        \end{aligned}
    \end{aligned}
\end{equation*}
where we used Lemma~\ref{lemma:sporadic_contraction}-\ref{lemma:sporadic_contraction:B} in the last inequality. Noting that $m^2 \| \bs{\bar{g}}_{\bs{v}}^{(k+1)} - \bs{\bar{g}}_{\bs{v}}^{(k)} \| \le m \| \bs{\Lambda}_{\bs{v}}^{(k+1)} \bs{G}^{(k+1)} - \bs{\Lambda}_{\bs{v}}^{(k)} \bs{G}^{(k)} \|$, we can take the expected value of the above expression and use Lemma~\ref{lemma:trackingeffect}-\ref{lemma:trackingeffect:sporadic_sgddiv} to get
\begin{equation} \label{eqn:ydispersion_halfway}
    \begin{aligned}
        \mathbb{E} & \left[ \left\| \bs{\Lambda}_{\bs{\hat{\pi}}}^{-1} \bs{Y}^{(k+1)} - m \bs{1} \bs{\bar{y}}^{(k+1)} \right\|_{\bs{\hat{\pi}}}^2 \right] \le \frac{1 + \hat{\rho}_B + \hat{\rho}_{0,B}}2 \mathbb{E}\left[ \left\| \bs{\Lambda}_{\bs{\hat{\pi}}}^{-1} \bs{Y}^{(k)} - m \bs{1} \bs{\bar{y}}^{(k)} \right\|_{\bs{\hat{\pi}}}^2 \right]
        \\
        & \begin{aligned}
            & + \frac{2 (m + 1) (1 + \hat{\rho}_B + \hat{\rho}_{0,B})}{1 - (\hat{\rho}_B + \hat{\rho}_{0,B})} \Bigg[ 2m  \overline{p \left( \sigma_0^2 B^{-1} (1 - B/D) + 5 (1 - p) \delta_0^2 \right)}
            \\
            & + \kappa_4 \left( \mathbb{E}\left[ \left\| \bs{X}^{(k+1)} - \bs{1} \bs{\bar{x}}_{\bs{\hat{\phi}}}^{(k+1)} \right\|_{\bs{\hat{\phi}}}^2 \right] + \mathbb{E}\left[ \left\| \bs{X}^{(k)} - \bs{1} \bs{\bar{x}}_{\bs{\hat{\phi}}}^{(k)} \right\|_{\bs{\hat{\phi}}}^2 \right] \right)
            \\
            & + 3m \overline{p \left( 2 \sigma_1^2 B^{-1} (1 - B/D) + 5 (1 - p) \delta_1^2 \right)} \mathbb{E}\left[ \left\| \nabla{F}(\bs{\bar{x}}_{\bs{\hat{\phi}}}^{(k)}) \right\|^2 \right] + m \kappa_5 \mathbb{E}\left[ \left\| \bs{\bar{x}}_{\bs{\hat{\phi}}}^{(k+1)} - \bs{\bar{x}}_{\bs{\hat{\phi}}}^{(k)} \right\|^2 \right] \Bigg].
        \end{aligned}
    \end{aligned}
\end{equation}
Finally, using Lemma~\ref{lemma:xdispersion} to upper bound $\mathbb{E}[ \| \bs{X}^{(k+1)} - \bs{1} \bs{\bar{x}}_{\bs{\hat{\phi}}}^{(k+1)} \|_{\bs{\hat{\phi}}}^2 ]$ and Lemmas~\ref{lemma:trackingeffect}-\ref{lemma:trackingeffect:modelonestep} and \ref{lemma:descenteffect}-\ref{lemma:descenteffect:ypiinv} to upper bound $\mathbb{E}[ \| \bs{\bar{x}}_{\bs{\hat{\phi}}}^{(k+1)} - \bs{\bar{x}}_{\bs{\hat{\phi}}}^{(k)} \|^2 ]$ concludes the proof.

\subsection{Proof of Lemma~\ref{lemma:loss}} \label{app:lemma:loss}

Let $\bs{x} = \bs{\bar{x}}_{\bs{\hat{\phi}}}^{(k+1)}$ and $\bs{x}' = \bs{\bar{x}}_{\bs{\hat{\phi}}}^{(k)}$ in $F(\bs{x}) \le F(\bs{x}') + \left\langle \nabla{F}(\bs{x}'), \bs{x} - \bs{x}' \right\rangle + \frac12 \bar{L} \| \bs{x} - \bs{x}' \|^2$, which is a consequence of Assumption~\ref{assump:lipschitz}. We get
\begin{equation} \label{eqn:descent_halfway}
    F(\bs{\bar{x}}_{\bs{\hat{\phi}}}^{(k+1)}) \le F(\bs{\bar{x}}_{\bs{\hat{\phi}}}^{(k)}) + \left\langle \nabla{F}(\bs{\bar{x}}_{\bs{\hat{\phi}}}^{(k)}), \bs{\bar{x}}_{\bs{\hat{\phi}}}^{(k+1)} - \bs{\bar{x}}_{\bs{\hat{\phi}}}^{(k)} \right\rangle + \frac12 \bar{L} \left\| \bs{\bar{x}}_{\bs{\hat{\phi}}}^{(k+1)} - \bs{\bar{x}}_{\bs{\hat{\phi}}}^{(k)} \right\|^2.
\end{equation}
First, we use Eq.~\eqref{eqn:averages} (see Definition~\ref{def:avg_vec} for the notation) to write
\begin{equation}
    \bs{\bar{x}}_{\bs{\hat{\phi}}}^{(k+1)} - \bs{\bar{x}}_{\bs{\hat{\phi}}}^{(k)} = \bs{\hat{\phi}}^T \bs{\hat{A}}^{(k)} \bs{X}^{(k)} - \bs{\bar{x}}_{\bs{\hat{\phi}}}^{(k)} - \bs{\hat{\phi}}^T \bs{\Lambda}_{\bs{\eta}}^{(k)} \bs{\hat{B}}^{(k)} \bs{Y}^{(k)}.
\end{equation}
Therefore, we need to obtain upper bounds for (i) $\langle F(\bs{\bar{x}}_{\bs{\hat{\phi}}}^{(k)}), - \bs{\hat{\phi}}^T \bs{\Lambda}_{\bs{\eta}}^{(k)} \bs{\hat{B}}^{(k)} \bs{Y}^{(k)} \rangle$, (ii) $\langle F(\bs{\bar{x}}_{\bs{\hat{\phi}}}^{(k)}), \bs{\hat{\phi}}^T \bs{\hat{A}}^{(k)} \bs{X}^{(k)} - \bs{\bar{x}}_{\bs{\hat{\phi}}}^{(k)} \rangle$ and (iii) $\| \bs{\bar{x}}_{\bs{\hat{\phi}}}^{(k+1)} - \bs{\bar{x}}_{\bs{\hat{\phi}}}^{(k)} \|^2$. We already obtained an upper bound for (iii) in Lemma~\ref{lemma:trackingeffect}-\ref{lemma:trackingeffect:modelonestep}, and the expected value of (ii) will be zero according to Assumption~\ref{assump:bernoulli} (See Eq.~\eqref{eqn:expected_mat}).
We proceed with a bound for (i) as
\begin{equation} \label{eqn:descent_innerprod}
    \begin{aligned}
        & \left\langle \nabla{F}(\bs{\bar{x}}_{\bs{\hat{\phi}}}^{(k)}), - \bs{\hat{\phi}}^T \bs{\Lambda}_{\bs{\eta}}^{(k)} \bs{\hat{B}}^{(k)} \bs{Y}^{(k)} \right\rangle = \sum_{i=1}^m{\hat{\phi}_i \eta_i^{(k)} \left\langle \nabla{F}(\bs{\bar{x}}_{\bs{\hat{\phi}}}^{(k)}), - \sum_{j=1}^m{\hat{b}_{ij}^{(k)} \bs{y}_j^{(k)}} \right\rangle}
        \\
        & = \sum_{i=1}^m{\hat{\phi}_i \eta_i^{(k)} \left\langle \nabla{F}(\bs{\bar{x}}_{\bs{\hat{\phi}}}^{(k)}), \mp m \hat{\pi}_i \nabla{F}(\bs{\bar{x}}_{\bs{\hat{\phi}}}^{(k)}) - \sum_{j=1}^m{\hat{b}_{ij}^{(k)} \bs{y}_j^{(k)}} \right\rangle}
        \\
        & = \sum_{i=1}^m{\hat{\phi}_i \hat{\pi}_i \eta_i^{(k)} \left\langle \nabla{F}(\bs{\bar{x}}_{\bs{\hat{\phi}}}^{(k)}), \mp m \nabla{F}(\bs{\bar{x}}_{\bs{\hat{\phi}}}^{(k)}) - \frac1{\hat{\pi}_i} \sum_{j=1}^m{\hat{b}_{ij}^{(k)} \bs{y}_j^{(k)}} \right\rangle}
        \\
        & = \sum_{i=1}^m \hat{\phi}_i \hat{\pi}_i \eta_i^{(k)} \left( - m \left\| \nabla{F}(\bs{\bar{x}}_{\bs{\hat{\phi}}}^{(k)}) \right\|^2 + \left\langle \nabla{F}(\bs{\bar{x}}_{\bs{\hat{\phi}}}^{(k)}), m \nabla{F}(\bs{\bar{x}}_{\bs{\hat{\phi}}}^{(k)}) - \frac1{\hat{\pi}_i} \sum_{j=1}^m{\hat{b}_{ij}^{(k)} \bs{y}_j^{(k)}} \right\rangle \right)
    \end{aligned}
\end{equation}
\begin{equation*}
    \begin{aligned}
        & \le \sum_{i=1}^m \hat{\phi}_i \hat{\pi}_i \eta_i^{(k)} \left( -\frac{m}2 \left\| \nabla{F}(\bs{\bar{x}}_{\bs{\hat{\phi}}}^{(k)}) \right\|^2 + \frac1{2m} \left\| m \nabla{F}(\bs{\bar{x}}_{\bs{\hat{\phi}}}^{(k)}) - \frac1{\hat{\pi}_i} \sum_{j=1}^m{\hat{b}_{ij}^{(k)} \bs{y}_j^{(k)}} \right\|^2 \right)
        \\
        & \le - \frac{m^2}2 \overline{\hat{\phi} \hat{\pi} \eta^{(k)}} \left\| \nabla{F}(\bs{\bar{x}}_{\bs{\hat{\phi}}}^{(k)}) \right\|^2 + \frac{\max(\hat{\phi}_i \eta_i^{(k)})}{2m} \left\| \bs{\Lambda}_{\bs{\hat{\pi}}}^{-1} \bs{\hat{B}}^{(k)} \bs{Y}^{(k)} - m \bs{1} \nabla{F}(\bs{\bar{x}}_{\bs{\hat{\phi}}}^{(k)}) \right\|_{\bs{\hat{\pi}}}^2.
    \end{aligned}
\end{equation*}
Note that we obtained an upper bound for $\mathbb{E}[ \| \bs{\Lambda}_{\bs{\hat{\pi}}}^{-1} \bs{\hat{B}}^{(k)} \bs{Y}^{(k)} - m \bs{1} \nabla{F}(\bs{\bar{x}}_{\bs{\hat{\phi}}}^{(k)}) \|_{\bs{\hat{\pi}}}^2 ]$ in Lemma~\ref{lemma:gradconsensus}. Substituting Eq.~\eqref{eqn:descent_innerprod} back into Eq.~\eqref{eqn:descent_halfway} and taking the expected value of the result concludes the proof.


\section{Proofs of Propositions}

\subsection{Proof of Proposition~\ref{proposition:lr}} \label{app:proposition:lr}
We have that $\rho(\bs{\Psi}^{(k)}) = \max(|\lambda_1|, |\lambda_2|)$, where $\lambda_1$ and $\lambda_2$ are the eigenvalues of the matrix $\bs{\Psi}^{(k)}$. Since $\tr(\bs{\Psi}^{(k)}) = \lambda_1 + \lambda_2$ and $\det(\bs{\Psi}^{(k)}) = \lambda_1 \lambda_2$, we have that $\max(|\lambda_1|, |\lambda_2|) = \max(\frac12 |\tr(\bs{\Psi}^{(k)}) \pm (\tr(\bs{\Psi}^{(k)})^2 - 4 \det(\bs{\Psi}^{(k)}))^{1/2}|)$. Therefore, by noting that $\tr(\bs{\Psi}^{(k)}) = \psi_{11}^{(k)} + \psi_{22}^{(k)} > 0$ due to both $\psi_{11}^{(k)} > 0$ and $\psi_{22}^{(k)} > 0$, we will have that $\rho(\bs{\Psi}^{(k)}) = \frac12 (\tr(\bs{\Psi}^{(k)}) + (\tr(\bs{\Psi}^{(k)})^2 - 4 \det(\bs{\Psi}^{(k)}))^{1/2})$. Solving this equation, we get $\rho(\bs{\Psi}^{(k)}) = \frac12 (\psi_{11}^{(k)} + \psi_{22}^{(k)} + ((\psi_{11}^{(k)} - \psi_{22}^{(k)})^2 + 4 \psi_{12}^{(k)} \psi_{21}^{(k)})^{1/2})$, and to have $\rho(\bs{\Psi}^{(k)}) < 1$, the following must hold.
\begin{equation} \label{eqn:spectral_tr_det}
    \begin{gathered}
        (\psi_{11}^{(k)} - \psi_{22}^{(k)})^2 + 4 \psi_{12}^{(k)} \psi_{21}^{(k)} < 4 + (\psi_{11}^{(k)})^2 + (\psi_{22}^{(k)})^2 - 4 \psi_{11}^{(k)} - 4 \psi_{22}^{(k)} + 2 \psi_{11}^{(k)} \psi_{22}^{(k)}
        \\
        \Rightarrow \psi_{12}^{(k)} \psi_{21}^{(k)} < 1 - \psi_{11}^{(k)} - \psi_{22}^{(k)} + \psi_{11}^{(k)} \psi_{22}^{(k)} = (1 - \psi_{11}^{(k)}) (1 - \psi_{22}^{(k)})
    \end{gathered}
\end{equation}
Since we have both $\psi_{12}^{(k)} > 0$ and $\psi_{21}^{(k)} > 0$, we need to ensure that both $\psi_{11}^{(k)} < 1$ and $\psi_{22}^{(k)} < 1$ to be able to solve for Eq.~\eqref{eqn:spectral_tr_det}. We will enforce these two constraints by solving for (i) $\psi_{11}^{(k)} < \frac{3 + \tilde{\rho}_A}4$ and (ii) $\psi_{22}^{(k)} < \frac{7 + \tilde{\rho}_B}8$. Furthermore, noting that $\psi_{21}^{(k)}$ and $\psi_{22}^{(k)}$ are comprised of the terms $\kappa_7 \kappa_4 (1 + \psi_{11}^{(k)}) + m \kappa_7 \kappa_5 \hat{\rho}_{0,A}$ and $\kappa_7 \kappa_4 \psi_{12}^{(k)}$ according to Lemma~\ref{lemma:ydispersion}, respectively, we will also enforce $\psi_{21}^{(k)} < 2 \kappa_7 \kappa_4 + m \kappa_7 \kappa_5 \hat{\rho}_{0,A}$ and $\psi_{12}^{(k)} < \frac{1 - \tilde{\rho}_B}{4\kappa_7 \kappa_4}$. However, plugging the constraints (i) and (ii) into Eq.~\eqref{eqn:spectral_tr_det} also requires us to satisfy $\psi_{12}^{(k)} \psi_{21}^{(k)} < \frac{1 - \tilde{\rho}_A}4 \frac{1 - \tilde{\rho}_B}8$. Due to more flexibility satisfying stricter conditions over $\psi_{12}^{(k)}$, we choose the last two constraints as (iii) $\psi_{12}^{(k)} < \frac{1 - \tilde{\rho}_B}{8 (2 \kappa_7 \kappa_4 + m \kappa_7 \kappa_5 \hat{\rho}_{0,A})} \frac{1 - \tilde{\rho}_A}4$ and $\psi_{21}^{(k)} < 2 \kappa_7 \kappa_4 + m \kappa_7 \kappa_5 \hat{\rho}_{0,A}$. Using the values $\psi_{ij}^{(k)}$ given in Lemmas~\ref{lemma:xdispersion} and \ref{lemma:ydispersion}, we have



\begin{itemize}
    \item $\psi_{11}^{(k)} < \frac{3 + \tilde{\rho}_A}4$:
    \begin{equation} \label{eqn:lr_cond1}
        \max(\eta_i^{(k)}) < \frac1{2\sqrt{m}} \frac{1 - \tilde{\rho}_A}{\sqrt{1 + \tilde{\rho}_A}} \frac1{\sqrt{\kappa_1 \kappa_3}}.
    \end{equation}

    \item $\psi_{12}^{(k)} < \frac{1 - \tilde{\rho}_B}{8 \kappa_7 (2 \kappa_4 + m \kappa_5 \hat{\rho}_{0,A})} \frac{1 - \tilde{\rho}_A}4$:
    \begin{equation} \label{eqn:lr_cond2}
        \max(\eta_i^{(k)}) < \frac1{8 \sqrt{m+1}} \frac{1 - \tilde{\rho}_A}{\sqrt{1 + \hat{\rho}_A}} \frac{1 - \tilde{\rho}_B}{\sqrt{1 + \tilde{\rho}_B}} \frac1{\sqrt{\kappa_3 (2 \kappa_4 + m \kappa_5 \hat{\rho}_{0,A})}}.
    \end{equation}

    \item $\psi_{21}^{(k)} < 2 \kappa_7 \kappa_4 + m \kappa_7 \kappa_5 \hat{\rho}_{0,A}$:
    \begin{equation}
        \max(\eta_i^{(k)}) < \frac12 \sqrt{1 - \tilde{\rho}_A} \frac1{ \sqrt{\kappa_3 \kappa_5}},
    \end{equation}
    where we used the fact that $\kappa_1 \le \kappa_4$.

    \item $\psi_{22}^{(k)} < \frac{7 + \tilde{\rho}_B}8$:
    \begin{equation}
        \max(\eta_i^{(k)}) < \frac1{4 \sqrt{m (m + 1)}} \frac{1 - \tilde{\rho}_B}{\sqrt{1 + \tilde{\rho}_B}} \frac1{\sqrt{\kappa_3 \kappa_5}},
    \end{equation}
    where we used the looser upper bound $\psi_{12}^{(k)} < \frac{1 - \tilde{\rho}_B}{4\kappa_7 \kappa_4}$ here to obtain cleaner results, although a stricter upper bound $\psi_{12}^{(k)} < \frac{1 - \tilde{\rho}_B}{8 (2 \kappa_7 \kappa_4 + m \kappa_7 \kappa_5 \hat{\rho}_{0,A})} \frac{1 - \tilde{\rho}_A}4$ was guaranteed.
\end{itemize}
We obtained four constraints on the maximum learning rate $\max(\eta_i^{(k)})$ that need to be satisfied to have $\rho(\bs{\Psi}^{(k)}) < 1$. Noting that the learning rates satisfying Eq.~\eqref{eqn:lr_cond2} are a subset of Eq.~\eqref{eqn:lr_cond1}  (due to $\kappa_1 \le \kappa_4$) concludes the proof.

\subsection{Proof of Proposition~\ref{proposition:loss}} \label{app:proposition:loss}
We must ensure that the coefficients for $\mathbb{E}[\| \nabla{F}(\bs{\bar{x}}_{\bs{\hat{\phi}}}^{(r)}) \|^2]$ in Eq.~\eqref{eqn:loss_expanded} are negative for all $r = 0, ..., k$, i.e., $\gamma_0^{(k)} > 0$ and $\gamma_0^{(r)} - \left( \sum_{s=r+1}^k{(\bs{\psi}_0^{(s)})^T \bs{\Psi}^{(s-2:r)}} \right) \bs{\gamma}^{(r)} > 0$ for $r = 0, ..., k-1$. Using the definitions of $\gamma_0^{(k)}$ and $\bs{\psi}_0^{(k)}$ from Lemma~\ref{lemma:loss} and $\bs{\gamma}^{(k)}$ from Lemmas~\ref{lemma:xdispersion} and \ref{lemma:ydispersion}, we have
\begin{itemize}
    \item $\gamma_0^{(k)} > 0$: We need to obtain the conditions under which we have
    \begin{equation}
        m \hat{\pi}_i - \frac{2 \sigma_{1,i}^2 p_i (1 - B_i / D_i)}{B_i} - 3 (1 - p_i) \delta_{1,i}^2 - 3 \hat{\rho}_{0,B} \kappa_2 > 0.
    \end{equation}
    Let us enforce $\frac{2 \sigma_{1,i}^2 p_i (1 - B_i / D_i)}{B_i} + 3 (1 - p_i) \delta_{1,i}^2 + 3 \hat{\rho}_{0,B} \kappa_2 \le \frac{m \hat{\pi}_i}{\Gamma_1}$ for a scalar $\Gamma_1 > 1$.
    which holds if $3 (1 - p_i) \delta_{1,i}^2 \le \frac{m \hat{\pi}_i}{3 \Gamma_1}$, $\frac{2 \sigma_{1,i}^2 p_i (1 - B_i / D_i)}{B_i} \le \frac{m \hat{\pi}_i}{3 \Gamma_1}$ and $3 \hat{\rho}_{0,B} \kappa_2 \le \frac{m \hat{\pi}_i}{3 \Gamma_1}$ are satisfied. Thus, we need to enforce
    \begin{equation} \label{eqn:p_B_cond}
        \begin{gathered}
            p_i \geq 1 - \frac{m \hat{\pi}_i}{9 \Gamma_1 \delta_{1,i}^2},
            \\
            B_i \geq \frac{D_i}{1 + \frac{m \hat{\pi}_i}{6 \Gamma_1 \sigma_{1,i}^2 p_i} D_i} \qquad \Rightarrow B_i \geq \frac{D_i}{1 + \frac{m \hat{\pi}_i}{6 \Gamma_1 \sigma_{1,i}^2} D_i}
            \\
            \hat{p}_{ij} (1 - \hat{p}_{ij}) \le \frac{m}{18 (m - 1) \kappa_2 \Gamma_1 b_{ij}^2} \frac{\min(\hat{\pi}_i)^2}{\max(\hat{\pi}_i)} \qquad \Rightarrow \hat{p}_{ij} \in (0, 1 -\hat{r}'_B] \cup [\hat{r}'_B, 1]
        \end{gathered}
    \end{equation}
    where we used the fact that $\hat{\rho}_{0,B} = 2 (m - 1) \frac{\max(\hat{\pi}_i)}{\min(\hat{\pi}_i)} \max( b_{ij}^2 \hat{p}_{ij} (1 - \hat{p}_{ij}) )$ from Lemma~\ref{lemma:sporadic_contraction}-\ref{lemma:sporadic_contraction:B}, and $\hat{r}'_B = \frac12 (1 + \sqrt{\max(0, 1 - \hat{\tau}'_B)})$ with $\hat{\tau}'_B = \frac{2m}{9 (m - 1) \kappa_2 \Gamma_1 b_{ij}^2} \frac{\min(\hat{\pi}_i)^2}{\max(\hat{\pi}_i)}$.

    \item $\gamma_0^{(r)} - ( \sum_{s=r+1}^k{(\bs{\psi}_0^{(s)})^T \bs{\Psi}^{(s-2:r)}} ) \bs{\gamma}^{(r)} > 0$:
    To simplify this constraint, we further strengthen it to $\gamma_0^{(r)} - \sum_{s=r+1}^k{(\bs{\psi}_0^{(s)})^T (\prod_{q=r}^{s-2}{\rho(\bs{\Psi}^{(q)})}) \bs{\gamma}^{(r)}} > 0$.
    Since the values of $\bs{\psi}_0^{(k)}$ are proportional to the non-increasing learning rate $\max(\eta_i^{(k)})$ according to the Lemma~\ref{lemma:loss}, the values of $\bs{\psi}_0^{(k)}$ will also be non-increasing, i.e., $\bs{\psi}_0^{(k+1)} \le \bs{\psi}_0^{(k)}$. We can use this fact to ensure that the constraint is valid for all $s = r+1, ..., k$ by strengthening it to $\gamma_0^{(r)} - (\bs{\psi}_0^{(r+1)})^T ( \sum_{s=r+1}^k{(\prod_{q=r}^{s-2}{\rho(\bs{\Psi}^{(q)})})} ) \bs{\gamma}^{(r)} > 0$, or $\gamma_0^{(r)} - (\bs{\psi}_0^{(r)})^T ( \sum_{s=r+1}^k{(\prod_{q=r}^{s-2}{\rho(\bs{\Psi}^{(q)})})} ) \bs{\gamma}^{(r)} > 0$ to make the superscripts the same. Since we are using a constant learning rate $\eta_i^{(k)} = \eta_i^{(0)}$, we will have $\rho(\bs{\Psi}^{(k)}) = \rho_\Psi$ for all $k \geq 0$, and thus the constraint can be simplified to $\gamma_0^{(r)} - (\bs{\psi}_0^{(r)})^T ( \sum_{s=r+1}^k{\rho_\Psi^{s-1-r}} ) \bs{\gamma}^{(r)} > 0$. We strengthen this bound further by letting the sum run to infinity, and by employing the fact that $0 < \rho_\Psi < 1$, we get
    the final constraint to 
    be
    \begin{equation}
        \gamma_0^{(r)} - \frac{(\bs{\psi}_0^{(r)})^T \bs{\gamma}^{(r)}}{1 - \rho_\Psi} > 0.
    \end{equation}
    Let us enforce a stronger constraint $\frac{(\bs{\psi}_0^{(r)})^T \bs{\gamma}^{(r)}}{1 - \rho_\Psi} < \frac{\gamma_0^{(r)}}{\Gamma_2}$ with $\Gamma_2 > 1$. Plugging in the values of these parameters from Lemmas~\ref{lemma:xdispersion}-\ref{lemma:loss} and doing some algebra, we get
    \begin{equation} \label{eqn:lr_pow3}
        \begin{aligned}
            & \left( \left( \frac{m (1 + 3 \hat{\rho}_{0,B})}{3 \kappa_7 \tilde{\rho}_B} + 1 \right) \kappa_4 \kappa_6 + \frac{\kappa_5}{m} \right) \kappa_2 \kappa_3 \max(\eta_i^{(r)})^3
            \\
            & < \frac{m (1 - \rho_\Psi)}{3 \Gamma_2 \kappa_7 \tilde{\rho}_B} \left( 1 - \frac1{\Gamma_1} \right) \frac{\overline{\hat{\phi} \hat{\pi} \eta^{(r)}}}{\max(\hat{\phi}_i)} - \frac{3 \max(\eta_i^{(r)})}{m} \overline{p (2 \sigma_1^2 B^{-1} (1 - B/D) + 5 (1 - p) \delta_1^2)}
        \end{aligned}
    \end{equation}
    where we used $m \hat{\pi}_i - \frac{2 (1 - B_i / D_i) \sigma_{1,i}^2 p_i}{B_i} - 3 (1 - p_i) \delta_{1,i}^2 - 3 \hat{\rho}_{0,B} \kappa_2 \geq m \hat{\pi}_i (1 - \frac1{\Gamma_1})$ as a result of the first part of this proof, and $\kappa_8 \le \max(\hat{\phi}_i) \kappa_1 \le \max(\hat{\phi}_i) \kappa_4$. To be able to satisfy this, we first need to ensure that the right-hand side of the inequality is positive. We have
    \begin{equation} \label{eqn:lr_het}
        \frac{\overline{\eta^{(r)}}}{\max(\eta_i^{(r)})} > \frac9{m^2 \left( 1 - \frac1{\Gamma_1} \right)} \frac{\max(\hat{\phi}_i)}{\min(\hat{\phi}_i \hat{\pi}_i)} \frac{\Gamma_2 \kappa_7 \tilde{\rho}_B}{1 - \rho_\Psi} \overline{p (2 \sigma_1^2 B^{-1} (1 - B/D) + 5 (1 - p) \delta_1^2)}.
    \end{equation}
    Let us enforce $\frac{\bar{\eta}^{(0)}}{\max(\eta_i^{(0)})} > \frac{9 \Gamma_3}{m^2 \left( 1 - \frac1{\Gamma_1} \right)} \frac{\max(\hat{\phi}_i)}{\min(\hat{\phi}_i \hat{\pi}_i)} \frac{\Gamma_2 \kappa_7 \tilde{\rho}_B}{1 - \rho_\Psi} \overline{p (2 \sigma_1^2 B^{-1} (1 - B/D) + 5 (1 - p) \delta_1^2)}$ for $\Gamma_3 > 1$. Since $\frac{\bar{\eta}^{(0)}}{\max(\eta_i^{(0)})} \le 1$ holds by definition, we need to ensure that the right-hand side of this inequality is less that one to be able to satisfy the constraint. Therefore, noting that $p_i \geq 1 - \frac{m \hat{\pi}_i}{9 \Gamma_1 \delta_{1,i}^2}$ and $B_i \geq \frac{D_i}{1 + \frac{m \hat{\pi}_i}{6 \Gamma_1 \sigma_{1,i}^2} D_i}$ from Eq.~\eqref{eqn:p_B_cond}, we have
    \begin{equation}
        \begin{gathered}
            p_i \left( 2 \frac{(1 - B_i / D_i) \sigma_{1,i}^2}{B_i} + 5 (1 - p_i) \delta_{1,i}^2 \right) \le \frac{m^2 \left( 1 - \frac1{\Gamma_1} \right)}{9 \Gamma_3} \frac{\min(\hat{\phi}_i \hat{\pi}_i)}{\max(\hat{\phi}_i)} \frac{1 - \rho_\Psi}{\Gamma_2 \kappa_7 \tilde{\rho}_B}
            \\
            \Rightarrow p_i \le \frac{m \left( \Gamma_1 - 1 \right)}{8 \Gamma_3} \frac{\min(\hat{\phi}_i \hat{\pi}_i)}{\max(\hat{\phi}_i \hat{\pi}_i)} \frac{1 - \rho_\Psi}{\Gamma_2 \kappa_7 \tilde{\rho}_B}.
        \end{gathered}
    \end{equation}
    So far, we have obtained the allowed range $p_i \in (1 - \frac{m \hat{\pi}_i}{9 \Gamma_1 \delta_{1,i}^2}, \frac{m \left( \Gamma_1 - 1 \right)}{8 \Gamma_3} \frac{\min(\hat{\phi}_i \hat{\pi}_i)}{\max(\hat{\phi}_i \hat{\pi}_i)} \frac{1 - \rho_\Psi}{\Gamma_2 \kappa_7 \tilde{\rho}_B}]$ for gradient probabilities. However, by tuning $\Gamma_1$, we can make the maximum value of this range to be one as follows.
    \begin{equation} \label{eqn:Gamma1_val}
        \Gamma_1 = 1 + \frac{8 \Gamma_3}m \frac{\max(\hat{\phi}_i \hat{\pi}_i)}{\min(\hat{\phi}_i \hat{\pi}_i)} \frac{\Gamma_2 \kappa_7 \tilde{\rho}_B}{1 - \rho_\Psi}.
    \end{equation}
    Plugging this back into Eq.~\eqref{eqn:lr_het} gives us
    \begin{equation} \label{eqn:lr_het_Gamma2}
        \frac{\bar{\eta}^{(0)}}{\max(\eta_i^{(0)})} \geq \frac1{\Gamma_3}.
    \end{equation}
    Finally, using the lower bounds obtained in Eqs.~\eqref{eqn:lr_het}, \eqref{eqn:Gamma1_val} and \eqref{eqn:lr_het_Gamma2} in Eq.~\eqref{eqn:lr_pow3}, we get
    \begin{equation}
        \left( \left( \frac{m (1 + 3 \hat{\rho}_{0,B})}{3 \kappa_7 \tilde{\rho}_B} + 1 \right) \kappa_4 \kappa_6 + \frac{\kappa_5}{m} \right) \max(\eta_i^{(0)})^2 < \frac8{3 \kappa_2 \kappa_3} \frac1{\Gamma_1} \left( 1 - \frac1{\Gamma_3} \right) \frac{\max(\hat{\phi}_i \hat{\pi}_i)}{\max(\hat{\phi}_i)},
    \end{equation}
    concluding the proof.
\end{itemize}

\section{Proof of Theorem~\ref{theorem:constant}} \label{app:theorem:constant}
We start from Eq.~\eqref{eqn:loss_expanded} and use Corollary~\ref{corollary:constant} to write
\begin{equation}
    \begin{aligned}
        \mathbb{E}\left[ F(\bs{\bar{x}}_{\bs{\hat{\phi}}}^{(k+1)}) \right] \le F(\bs{\bar{x}}_{\bs{\hat{\phi}}}^{(0)}) & - \gamma_0^{(0)} \mathbb{E}\left[ \left\| \nabla{F}(\bs{\bar{x}}_{\bs{\hat{\phi}}}^{(k)}) \right\|^2 \right]
        \\
        & - \sum_{r=0}^{k-1}{\left( \gamma_0^{(0)} - (\bs{\psi}_0^{(0)})^T \left( \sum_{s=r+1}^k{\bs{\Psi}^{s-1-r}} \right) \bs{\gamma}^{(0)} \right) \mathbb{E}\left[ \left\| \nabla{F}(\bs{\bar{x}}_{\bs{\hat{\phi}}}^{(r)}) \right\|^2 \right]}
        \\
        & + (\bs{\psi}_0^{(0)})^T \left( \sum_{r=0}^k{\bs{\Psi}^r} \right) \bs{\varsigma}^{(0)} + (\bs{\psi}_0^{(0)})^T \sum_{r=0}^{k-1}{\left( \sum_{s=r+1}^k{\bs{\Psi}^{s-1-r}} \right) \bs{\omega}^{(0)}} + \omega_0^{(0)} (k + 1).
    \end{aligned}
\end{equation}
In the proof of Proposition~\ref{proposition:loss} in Appendix~\ref{app:proposition:loss}, we showed that when using a constant learning rate $\eta_i^{(0)}$ and $0 < \rho_\Psi < 1$, we get $\sum_{s=r+1}^k{\bs{\Psi}^{s-1-r}} \le \frac1{1 - \rho_\Psi}$. Using this result, and noting that we also similarly have $\sum_{r=0}^k{\bs{\Psi}^r} \le \frac1{1 - \rho_\Psi}$, we get
\begin{equation}
    \begin{aligned}
        \mathbb{E}\left[ F(\bs{\bar{x}}_{\bs{\hat{\phi}}}^{(k+1)}) \right] \le F(\bs{\bar{x}}_{\bs{\hat{\phi}}}^{(0)}) & - \left( \gamma_0^{(0)} - \frac{(\bs{\psi}_0^{(0)})^T \bs{\gamma}^{(0)}}{1 - \rho_\Psi} \right) \sum_{r=0}^k{\mathbb{E}\left[ \left\| \nabla{F}(\bs{\bar{x}}_{\bs{\hat{\phi}}}^{(r)}) \right\|^2 \right]} + \frac{(\bs{\psi}_0^{(0)})^T \bs{\varsigma}^{(0)}}{1 - \rho_\Psi}
        \\
        & + \frac{(\bs{\psi}_0^{(0)})^T \bs{\omega}^{(0)}}{1 - \rho_\Psi} (k + 1) + \omega_0^{(0)} (k + 1).
    \end{aligned}
\end{equation}
Note that the equation above is a compact reformulation of Eq.~\eqref{eqn:loss_expanded}. Rearranging the above expression and plugging in the values of $\bs{\omega}^{(0)}$, $\bs{\psi}_0^{(0)}$ and $\omega_0^{(0)}$ from Lemmas~\ref{lemma:xdispersion}-\ref{lemma:loss}, we get
\begin{equation}
    \begin{aligned}
        \frac1{k + 1} & \sum_{r=0}^k{\mathbb{E}\left[ \left\| \nabla{F}(\bs{\bar{x}}_{\bs{\hat{\phi}}}^{(r)}) \right\|^2 \right]} \le \frac{F(\bs{\bar{x}}_{\bs{\hat{\phi}}}^{(0)}) - F^\star}{q \max(\eta_i^{(0)}) (k + 1)} + \frac{(1 + 3 \hat{\rho}_{0,B}) \kappa_8}{2q (1 - \rho_\Psi)} \frac{\mathbb{E}\left[ \left\| \bs{X}^{(0)} - \bs{1} \bs{\bar{x}}_{\bs{\hat{\phi}}}^{(0)} \right\|_{\bs{\hat{\phi}}}^2 \right]}{k + 1}
        \\
        & + \frac{3 \max(\hat{\phi}_i) \tilde{\rho}_B}{2m q (1 - \rho_\Psi)} \frac{\mathbb{E}\left[ \left\| \bs{\Lambda}_{\bs{\hat{\pi}}}^{-1} \bs{Y}^{(0)} - m \bs{1} \bs{\bar{y}}^{(0)} \right\|_{\bs{\hat{\pi}}}^2 \right]}{k + 1}
        \\
        & + \frac{\max(\hat{\phi}_i)}{q} \left(\frac{m (1 + 3 \hat{\rho}_{0,B})}{2} + \frac{3 \kappa_7 \tilde{\rho}_B}{1 - \rho_\Psi} \right) \left( \overline{\sigma_0^2 p B^{-1} (1 - B/D)} + 5 \overline{(1 - p) \delta_0^2} \right)
        \\
        & + \frac{3 m \max(\hat{\phi}_i)}{2q} \kappa_9 \kappa_3 \kappa_7 \tilde{\rho}_B \left( \overline{\sigma_0^2 p B^{-1} (1 - B/D)} + 3 \overline{(1 - p) \delta_0^2} \right) \max(\eta_i^{(0)})^2,
    \end{aligned}
\end{equation}
where
$q$ was defined in
Proposition~\ref{proposition:loss}.
    

\section{Additional Experiments} \label{app:additional_exps}

We carry out an additional ablation study of system parameters in our paper on the CIFAR-10 dataset, which is given in Fig.~\ref{fig:cifar10_ablations}. In Fig.~\ref{fig:resnet18_data_dist}, we vary the number of labels each client holds, with $1$ corresponding to the extreme non-IID case, and $10$ corresponding to the completely IID case. In Fig.~\ref{fig:resnet18_graph_conn}, we vary the radius of the RGG, with higher radius indicating a denser connectivity in the network. In Fig.~\ref{fig:resnet18_num_clients}, we vary the number of clients $m$ from $10$ to $50$ to validate that our methodology maintains its superiority for larger network sizes. Finally, in Fig.~\ref{fig:vgg11_learning_rate}, we do a hyperparameter analysis of the learning rate $\eta$ from $10^{-4}$ to $10^{-1}$, which shows that the choice of $\eta = 0.01$ is the optimal choice for both datasets. In all of these studies, we observe that our \texttt{Spod-GT} algorithm is capable of consistently outperforming all baselines at different levels of data heterogeneity, confirming the benefit of integrating the notion of sporadicity in both communications and computations.

\begin{figure*}[h]
    \centering
    \begin{subfigure}[t]{\textwidth}
        \centering
        \includegraphics[width=0.75\textwidth]{images/legend.png}
    \end{subfigure}
    \begin{subfigure}[t]{0.24\textwidth}
        \centering
        \includegraphics[width=\textwidth]{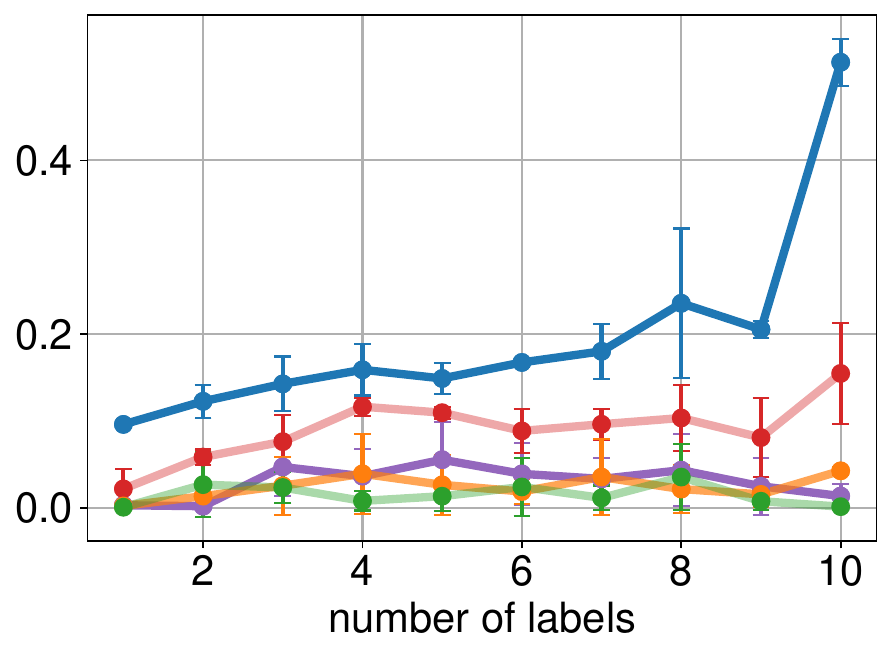}
        \caption{Varying number of labels per client.}
        \label{fig:resnet18_data_dist}
    \end{subfigure}
    \hfill
    \begin{subfigure}[t]{0.24\textwidth}
        \centering
        \includegraphics[width=\textwidth]{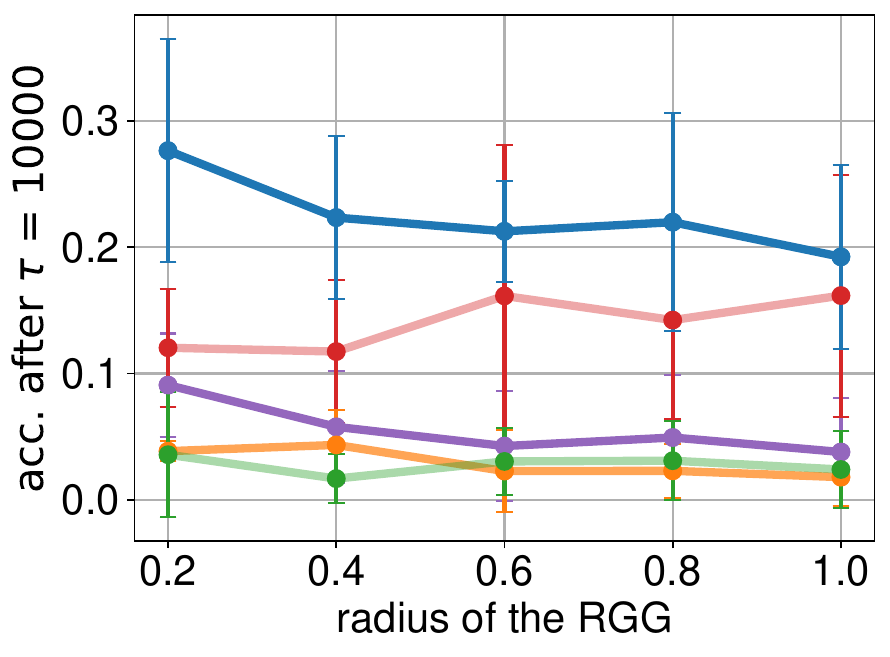}
        \caption{Varying the radius of random geometric  graph.}
        \label{fig:resnet18_graph_conn}
    \end{subfigure}
    \hfill
    \begin{subfigure}[t]{0.24\textwidth}
        \centering
        \includegraphics[width=\textwidth]{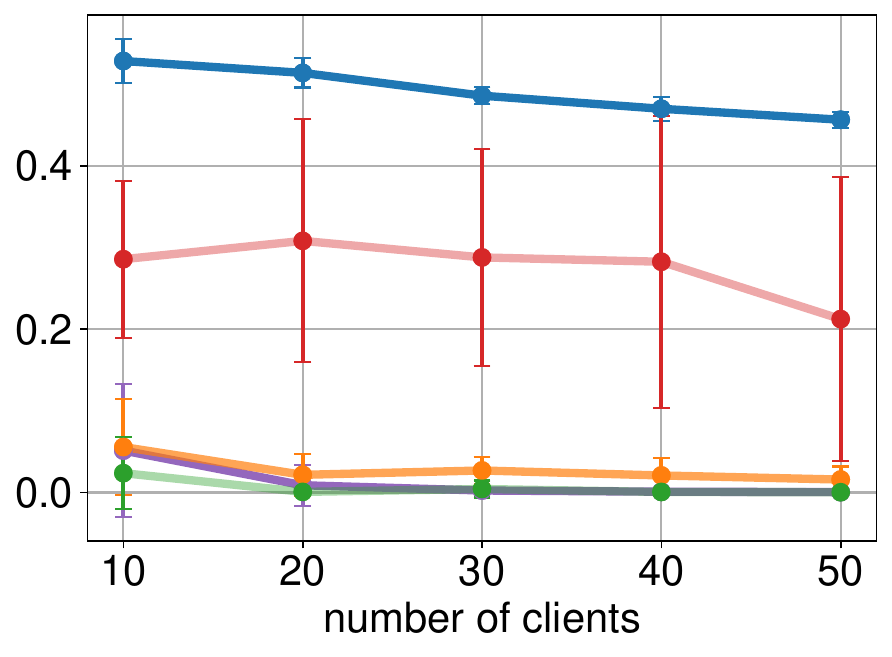}
        \caption{Varying the number of clients $m$ in the network.}
        \label{fig:resnet18_num_clients}
    \end{subfigure}
    \hfill
    \begin{subfigure}[t]{0.24\textwidth}
        \centering
        \includegraphics[width=\textwidth]{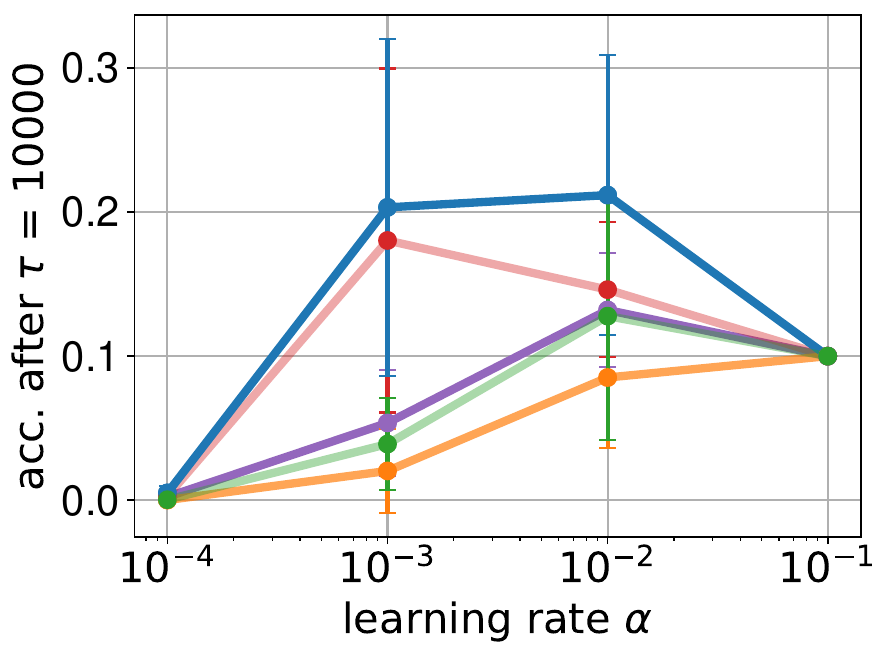}
        \caption{Varying the learning rate $\eta$ for all clients.}
        \label{fig:vgg11_learning_rate}
    \end{subfigure}
    \caption{\small Effects of system parameters on CIFAR-10. The overall results confirm the advantage of \texttt{Spod-GT}.}
    \label{fig:cifar10_ablations}
\end{figure*}

\end{document}